\documentclass{article} 
\usepackage{iclr2025_conference,times}
\usepackage[T1]{fontenc}

\usepackage{amsfonts, amssymb, amsmath, amsthm, bm, accents, dsfont}
\usepackage{fontawesome5}

\usepackage{hyperref, url}

\usepackage{graphicx, tabularx, caption}

\usepackage{blindtext}

\usepackage[english]{babel}

\makeatletter
\protected\def\abx@missing#1{%
  \mbox{\reset@font\color{blue}#1}}
\makeatother

\usepackage{algpseudocode,algorithm}
\algrenewcommand\algorithmicdo{}
\algrenewtext{EndFor}{\algorithmicend}
\algrenewtext{EndProcedure}{\algorithmicend}

\usepackage{xcolor}

\definecolor{stutt.blue}{RGB}{0,81,158}
\definecolor{stutt.lightblue}{RGB}{0,190,255}
\definecolor{stutt.gray}{RGB}{62, 68, 76}
\definecolor{stutt.darkblue}{RGB}{0,50,98}

\definecolor{Set1.red}{rgb}{0.894117647058824,0.101960784313725,0.109803921568627}
\definecolor{Set1.blue}{rgb}{0.215686274509804,0.494117647058824,0.721568627450980}
\definecolor{Set1.green}{rgb}{0.301960784313725,0.686274509803922,0.290196078431373}
\definecolor{Set1.purple}{rgb}{0.596078431372549,0.305882352941177,0.639215686274510}
\definecolor{Set1.orange}{rgb}{1,0.498039215686275,0}
\definecolor{Set1.yellow}{rgb}{1,1,0.2}
\definecolor{Set1.brown}{rgb}{0.650980392156863,0.337254901960784,0.156862745098039}
\definecolor{Set1.pink}{rgb}{0.968627450980392,0.505882352941176,0.749019607843137}
\definecolor{Set1.grey}{rgb}{0.6,0.6,0.6}

\usepackage{microtype}

\theoremstyle{plain}
\newtheorem{theorem}{Theorem}[section]
\newtheorem{proposition}[theorem]{Proposition}
\newtheorem{lemma}[theorem]{Lemma}

\theoremstyle{definition}
\newtheorem{definition}[theorem]{Definition}
\newtheorem{assumption}[theorem]{Assumption}
\theoremstyle{remark}

\usepackage{array}		
\usepackage{tabularx}
\usepackage{longtable}
\usepackage{booktabs}		
\usepackage[inline,shortlabels]{enumitem}
\setenumerate{itemsep=\smallskipamount,topsep=\smallskipamount,left=\parindent}
\setitemize{itemsep=\smallskipamount,topsep=\smallskipamount,left=\parindent}
\usepackage{makecell}

\usepackage{mathtools}

\usepackage[sort&compress,capitalize,nameinlink]{cleveref}
\crefname{assumption}{Assumption}{Assumptions}

\DeclareMathOperator{\E}{\mathbb{E}}

\DeclareMathOperator{\diag}{diag}

\newcommand{\norm}[1]{\ensuremath{\left\| #1 \right\|}}
\newcommand{\abs}[1]{\ensuremath{{\left\vert #1 \right\vert}}}

\DeclareMathOperator{\indicator}{\mathbb{I}}
\newcommand{\ceil}[1]{\left \lceil #1 \right \rceil}

\DeclareMathOperator*{\argmin}{argmin}

\DeclareMathOperator*{\minimize}{minimize}
\DeclareMathOperator*{\maximize}{maximize}
\DeclareMathOperator{\subjectto}{subject\ to}
\DeclareMathOperator{\find}{find}
\DeclareMathOperator{\suchthat}{such\ that}

\newcommand{\calA}{\ensuremath{\mathcal{A}}}
\newcommand{\calB}{\ensuremath{\mathcal{B}}}

\newcommand{\calD}{\ensuremath{\mathcal{D}}}

\newcommand{\calF}{\ensuremath{\mathcal{F}}}
\newcommand{\calG}{\ensuremath{\mathcal{G}}}

\newcommand{\calP}{\ensuremath{\mathcal{P}}}

\newcommand{\calR}{\ensuremath{\mathcal{R}}}
\newcommand{\calS}{\ensuremath{\mathcal{S}}}
\newcommand{\calT}{\ensuremath{\mathcal{T}}}

\newcommand{\calZ}{\ensuremath{\mathcal{Z}}}

\newcommand{\bbN}{\ensuremath{\mathbb{N}}}

\newcommand{\bbR}{\ensuremath{\mathbb{R}}}

\newcommand{\fkm}{\ensuremath{\mathfrak{m}}}

\newcommand{\fkp}{\ensuremath{\mathfrak{p}}}

\newcommand{\setN}{\bbN}

\newcommand{\setR}{\bbR}

\newcommand{\del}{\ensuremath{\partial}}

\makeatletter
\def\nnil{\nil}
\newcounter{prob}
\newenvironment{prob}[1][\nil]{%
	\def\tmp{#1}
	\equation
	\ifx\tmp\nnil
		\refstepcounter{prob}
		\tag{P\Roman{prob}}
	\else
		\tag{\tmp}
	\fi
	\aligned%
}{%
	\endaligned\endequation%
}
\makeatother

\newenvironment{prob*}{%
	\csname equation*\endcsname%
	\aligned%
}{%
	\endaligned%
	\csname endequation*\endcsname%
}

\graphicspath{{figures/}}

\usepackage{caption}
\usepackage{xspace}

\usepackage{multirow}
\usepackage{booktabs}  

\newcommand{\lpde}{\ensuremath{\ell_{\text{pde}}}}
\newcommand{\lbc}{\ensuremath{\ell_{\text{bc}}}}
\newcommand{\ppinn}{P\textsuperscript{2}INN\xspace}
\newcommand{\dvar}[1]{\lambda^\text{#1}}
\newcommand{\data}[1]{x^\text{#1}_n, t^\text{#1}_n}
\newcommand{\datapar}[1]{\data{#1}, \pi^\text{#1}_n}
\newcommand{\loss}[1]{\ell^\text{#1}}

\title{Solving Differential Equations with \\ Constrained Learning}

\author{Viggo~Moro\\
University of Oxford\\
\texttt{viggo.moro@cs.ox.ac.uk}\\
\And
Luiz~F.~O.~Chamon\\
École polytechnique\\
\texttt{luiz.chamon@polytechnique.edu}
}

\newcommand{\clear}[1]{}

\iclrfinalcopy 
\begin{document}

\maketitle

\begin{abstract}
(Partial) differential equations~(PDEs) are fundamental tools for describing natural phenomena, making their solution crucial in science and engineering. While traditional methods, such as the finite element method, provide reliable solutions, their accuracy is often tied to the use of computationally intensive fine meshes. Moreover, they do not naturally account for measurements or prior solutions, and any change in the problem parameters requires results to be fully recomputed. Neural network-based approaches, such as physics-informed neural networks and neural operators, offer a mesh-free alternative by directly fitting those models to the PDE solution. They can also integrate prior knowledge and tackle entire families of PDEs by simply aggregating additional training losses. Nevertheless, they are highly sensitive to hyperparameters such as collocation points and the weights associated with each loss. This paper addresses these challenges by developing a \emph{science-constrained learning}~(SCL) framework. It demonstrates that finding a (weak)~solution of a PDE is equivalent to solving a constrained learning problem with worst-case losses. This explains the limitations of previous methods that minimize the expected value of aggregated losses. SCL also organically integrates structural constraints~(e.g., invariances) and (partial)~measurements or known solutions. The resulting constrained learning problems can be tackled using a practical algorithm that yields accurate solutions across a variety of PDEs, neural network architectures, and prior knowledge levels without extensive hyperparameter tuning and sometimes even at a lower computational~cost.
\end{abstract}

\section{Introduction}

(Partial)~differential equations~(PDEs) are key tools in science and engineering, playing a central role in the solution of inverse problems, systems engineering, and the description of natural phenomena~\citep{Lustig08c, Potter10s, Molesky18i, Evans10p}. As such, a variety of numerical methods have been developed to approximate their solutions, such as the well-known finite element method~(FEM). Despite their celebrated precision and approximation guarantees, these methods provide solutions to a single PDE at a time. Any change to the problem, from boundary condition to mesh size, requires the solution to be recomputed. They are therefore unable to incorporate prior knowledge, such as real-world measurements or known solutions to similar equations~\citep{Brenner07t, LeVeque07f, Katsikadelis16t}.

Methods based on neural networks~(NNs),
such as physics-informed NNs~(PINNs)~\citep{Lagaris98a, Raissi19p, Lu21p}
and neural operators~(NOs)~\citep{Li21f, Lu21l, Rahman23u},
have been developed with these challenges in mind. Rather than discretizing the PDE, they directly fit a NN to its solution. They can therefore be trained to simultaneously solve entire families of PDEs and interpolate known solutions by simply incorporating additional losses to their training objectives~\citep{Li21f, Cho24p, Li24p}. Yet, these methods are highly sensitive to hyperparameters such as the weights used to combine the training losses and the collocation points used to evaluate the PDE residuals, which often leads to low quality or trivial solutions~\citep{Krishnapriyan21c, Wight21s, Wang22w, Wang22i}. This has prompted a variety of heuristics to be proposed based on \emph{ad hoc} weight updates~\citep{Wang21u, Maddu22i} and adaptive or causal sampling~\cite{Nabian21e, Krishnapriyan21c, McClenny23s, Penwarden23a}~(see Appendix~\ref{app:related_work} for further related works).

This paper shows that these limitations are not methodological, but epistemological. It is not an issue of \emph{how} the problem is solved, but \emph{which} problem is being solved. To do so, it
\begin{itemize}[itemsep=2pt,topsep=0pt]
	\item proves that obtaining a~(weak) solution of a PDE requires solving a constrained learning problem with worst-case losses~(Prop.~\ref{thm:weak_formulation_to_statistical_problem}), i.e., it is not enough to use either constrained formulations~\citep{Lu21p, Basir22p} \emph{or} worst-case losses~\citep{Wang22i, Daw23m};

	\item incorporates prior scientific knowledge on the structure~(e.g., invariance) and value~(e.g., measurements, set points) of the solution without resorting to specialized models or data transforms~(Sec.~\ref{sec:science_constrained_learning}). We therefore dub this approach \emph{science-constrained learning}~(SCL);

	\item develops a practical algorithm that foregoes the careful selection of loss weights and collocation points. Contrary to other methods, it explicitly approximates the (weak)~solution of PDEs and yields reliability metrics that capture the difficulty of fitting specific PDE parameters and/or data points~(Sec.~\ref{sec:algorithm});

	\item illustrates the effectiveness~(accuracy and computational cost) of SCL~(Sec.~\ref{sec:experiments}) for a diverse set of PDEs, NN architectures~(e.g., MLPs and NOs), and problem types~(solving a single or a parametric family of PDEs; interpolating known solutions; identifying settings that are difficult to fit).
\end{itemize}

\section{Problem Formulation}

\subsection{Boundary Value Problems}
\label{sec:BVP}

Consider a~(bounded, connected, open) region $\Omega \subset \mathbb{R}^d$ with (smooth)~boundary $\partial \Omega$ and a (partial)~differential operator~$D_\pi$ with coefficients~$\pi \in \Pi \subset \setR^q$ defined on the domain~$\calD = \Omega \times (0, T]$. Here, $\pi$ captures a (finite)~set of parameters of the phenomenon, such as the diffusion rate or viscosity. Given a space~$\calF$ of functions mapping~$\calD$ to~$\mathbb{R}$, we define a boundary value problem~(BVP) as
\begin{prob}[\textup{BVP}]\label{eq:BVP}
    \find& &&u \in \calF
    \\
    \suchthat& &&D_\pi[ u ] (x,t)  = \tau(x, t),
    	&& (x, t) \in \Omega \times (0,T] && \text{(PDE)}
    \\
    &&&u(x, 0) = h(x,0),
    	&& x \in \bar{\Omega} && \text{(IC)}
    \\
    &&&u(x, t) = h(x, t),
    	&& (x, t) \in \partial \Omega \times (0,T] && \text{(BC)}
\end{prob}
where~$\tau: \calD \rightarrow \mathbb{R}$ is a \emph{forcing function} and~$h: \calB \rightarrow \mathbb{R}$ describes the boundary~(BC) and initial~(IC) conditions over~$\calB = (\bar{\Omega} \times \{0\}) \cup (\del\Omega \times (0,T])$ for~$\bar{\Omega} = \Omega \cup \del\Omega$. In what follows, we always consider the initial condition as part of the boundary conditions and refer to them jointly as~BC. Note that the developments in this paper also apply to formulations of~\eqref{eq:BVP} involving other BCs~(e.g., \emph{Neumann}, \emph{periodic}), parameterized BCs, and vector-valued PDEs. We showcase a variety of phenomena that can be described by~\eqref{eq:BVP} in App.~\ref{app:pdes} and refer to, e.g.,~\citep{Evans10p}, for a more detailed treatment.

In general, a (strong)~solution of~\eqref{eq:BVP} need not exist in any function space~$\calF$. This motivates the rise of relaxations such as the \emph{weak formulation} that replaces the pointwise equation~(PDE) in~\eqref{eq:BVP} by the integral equation%
\footnote{Typically, the weak formulation only considers the spatial domain~$x$, handling~$t$ separately using time-stepping methods~\citep{Thomee13g}. Still, \eqref{eq:weak_form} is not uncommon and can be used to obtain weak formulations for parabolic PDEs~(see, e.g., \citep{Knabner03n, Evans10p, Steinbach20c}).}
\begin{equation}\label{eq:weak_form}
	\int_{\calD} D_\pi [ u ](x,t) \varphi(x,t) \, dx dt
		= \int_{\calD} \tau(x, t) \varphi(x,t) \, dx dt, \quad \text{for all } \varphi \in \calT
		\text{.}
\end{equation}
The~$\varphi$ are known as \emph{test functions} and~$\calT$ is typically taken to be a \emph{Sobolev space} due to its natural compatibility with this setting~(see App.~\ref{app:weak_formulation_to_statistical_problem} for further details). The name \emph{weak formulation} comes from the fact that a solution of~\eqref{eq:BVP}~(when it exists) is also a solution of~\eqref{eq:weak_form}, although the converse is not necessarily true. Indeed, \eqref{eq:weak_form} allows for a wider range of solutions with less stringent regularity requirements, particularly with respect to continuity and differentiability~\citep{Evans10p}. The BCs of~\eqref{eq:BVP} can often be homogeneized~(i.e.,~$h \equiv 0$) and imposed implicitly through~$\calT$, thus fully describing its weak solution by~\eqref{eq:weak_form}~\citep{Brenner07t, LeVeque07f, Katsikadelis16t}.

\subsection{Solving boundary value problems}
\label{sec:solving_bvp}

There exists a wide range of numerical methods for solving BVPs, most of which rely on discretizing~\eqref{eq:BVP}~(e.g., finite difference method, FDM) or its weak formulation~(e.g., FEM).
Their well-established approximation guarantees, accuracy, and stable implementations make them ubiquitous in scientific and engineering applications~\citep{Brenner07t, LeVeque07f, Katsikadelis16t}. These classical methods, however, only tackle one BVP at a time and do not naturally incorporate prior knowledge, such as measurements or known (partial)~solutions.
These challenges can be addressed by directly parameterizing the solution~$u$ of BVPs, most notably using NNs. While this approach may not achieve the precision of classical methods~\citep{Krishnapriyan21c, Wight21s, Grossmann24c, McGreivy2024w}, they are able to simultaneously provide solutions for entire families of BVPs and extrapolate new solutions from existing ones.
Generally speaking, these methods can be separated into \emph{unsupervised}, that seek to solve~\eqref{eq:BVP} directly, and \emph{supervised}, that leverage previously computed or measured solutions.

\smallskip
\textbf{Unsupervised methods.}
These approaches train a model~$u_\theta: \calD \times \Pi \to \setR$ with parameters~$\theta \in \Theta \subset \setR^p$~(e.g., a multilayer perceptron, MLP) to fit~\eqref{eq:BVP}. This is the case, for instance, of physics-informed neural networks~(PINNs) that train~$u_\theta$ by solving, for fixed weights~$\mu_\text{pde}, \mu_\text{bc} \geq 0$,
\begin{prob}\label{eq:PINN}
	\minimize_{\theta \in \Theta}\ %
		\mu_\text{pde} \lpde(\theta) + \mu_\text{bc} \lbc(\theta)
		\text{,}
\end{prob}
\vspace{-1.5\baselineskip}
\begin{alignat*}{4}
	\lpde(\theta) &\triangleq \sum_{i=1}^{I} \Bigg[
	\frac{1}{M} \sum_{m=1}^{M}
			\Big( D_{\pi_i} [ u_{\theta}(\pi_i) ](x_m,t_m)  - \tau (x_m, t_m) \Big)^2
	\Bigg]
    	\text{,} \quad&&(x_m, t_m) \in \calD \text{,} \ \ \pi_i \in \Pi
    	\text{,}
    \\
    \lbc(\theta) &\triangleq \frac{1}{N} \sum_{n=1}^{N}
    	\Big( u_{\theta}(x_n,t_n) - h(x_n, t_n) \Big)^2
    	\text{,} &&(x_n, t_n) \in \calB
    	\text{.}
\end{alignat*}
We write~$u_\theta(\pi)(x,t)$ to emphasize that we evaluate~$u_\theta(\pi)$, which approximates the solution of~\eqref{eq:BVP} with coefficients~$\pi$, at~$(x,t) \in \calD$. In practice, the input of the model~$u_\theta$ is simply~$(x,t,\pi) \in \calD \times \Pi$.
The losses~$\lpde$ and~$\lbc$ promote the requirements in~(PDE) and (IC)--(BC) from~\eqref{eq:BVP}, respectively, although the former is computed using automatic differentiation rather than discretization. Though the majority of PINNs target a single BVP, i.e., $I = 1$ in~\eqref{eq:PINN}, their extension to parameterized families of BVPs has been explored~\citep{Cho24p}.

\smallskip
\textbf{Supervised methods.}
Rather than directly solving the BVP, these approaches fit the model~$u_\theta$ to a set of (partial)~solutions~$u^\dagger_n$ of~\eqref{eq:BVP}. In this setting, $u_\theta$ is often a neural operator~(NO) capable of handling infinite-dimensional~(functional) inputs and outputs, such as forcing functions~$\tau$ and ICs~$h(x,0)$. Given, e.g., forcing-solution pairs~$(\tau_j,u^\dagger_j)$, these NOs are trained by solving
\begin{prob}\label{eq:NO}
    \minimize_{\theta \in \Theta}\ \frac{1}{J} \sum_{j=1}^{J}
    	\big\| u_{\theta}(\tau_j) - u^\dagger_j \big\|_{L_2(\calD)}^2
    	\text{.}
\end{prob}
In practice, the functions~$\tau_n$ and~$u^\dagger_n$ are discretized~(in the time or spectral domain) to enable computations~\citep{Li21f, Lu21l, Hao23g, Wei23s}. While it is not uncommon to combine~\eqref{eq:PINN} and~\eqref{eq:NO}, these semi-supervised methods typically rely on MLPs~\citep{Raissi19p, Lu21p}, since it can be challenging to evaluate~$D[u_{\theta}]$ for NOs~\citep{Li24p}.

\smallskip
\textbf{Limitations.}
Though effective in many applications, these methods are very sensitive to their hyperparameters. Indeed, the choice of collocation points~$(x,t)$, PDE coefficients~$\pi_i$, and weights~$\mu_\text{pde},\mu_\text{bc}$ affect both the quality and computational complexity of~\eqref{eq:PINN}~\citep{Krishnapriyan21c, Wight21s, Wang21u}. The same holds for the discretization of the objective in~\eqref{eq:NO}~\citep{Li21f, Hao23g, Wei23s}.
Supervised methods face the additional challenge that acquiring the PDE solutions~$\{u^\dagger_n\}$ used in~\eqref{eq:NO} can be expensive~(relying on, e.g., classical methods) and it is challenging to obtain good performance from small datasets. This issue is aggravated by the heterogeneous difficulty of fitting each solution.
A variety of heuristics for collocation points~\citep{Nabian21e, Daw23m, Penwarden23a}, weights~\citep{Wang21u, Maddu22i, McClenny23s}, and PDE solutions~\citep{Pestourie23p, Musekamp24a} have been put forward to mitigate these challenges. Yet, they generally focus on specific ``failure modes'' or BVPs and seldom address the non-trivial interactions of these yperparameters, limiting their effectiveness.

\section{Science-Constrained Learning}
\label{sec:science_constrained_learning}

In this section, we argue that the challenges faced by previous NN-based BVP solvers arise not because of how~\eqref{eq:PINN} and~\eqref{eq:NO} are solved, but because they are not the appropriate problems to solve in the first place. To do so, we show that obtaining a (weak)~solution of~\eqref{eq:BVP} is equivalent to solving a \emph{constrained learning problem with worst-case losses}. Hence, it is not enough to use (approximations of)~worst-case losses as in, e.g., \citep{Wang22i, Daw23m}, \emph{or} adapting loss weights as in, e.g., \citep{Wang21u, Lu21p, McClenny23s}.
Building on this result, we show how to incorporate other forms of scientific knowledge that are not \emph{mechanistic}~(i.e., PDEs) without resorting to specialized models, including \emph{structural} information~(e.g., invariances) and \emph{observations}~(measurements, simulations) of the solution.
The resulting \emph{science-constrained learning}~(SCL) problem accommodates a variety of knowledge settings, from unsupervised to supervised, and is amenable to a practical algorithm capable of effectively tackling entire families of BVPs and extrapolating solutions from existing ones~(Sec.~\ref{sec:algorithm} and~\ref{sec:experiments}).

In the remainder of this paper, we use~$u_\theta$ to refer to any parameterized model~(MLP, NO, etc.). For clarity, we derive our results for a single BVP instance, omitting the dependence on $\pi$ and/or~$\tau$. We consider these extensions at the end of the section.

\paragraph{Mechanistic~(PDE) knowledge.}
\label{sec:mechanistic}

We begin by showing how weak solutions of~\eqref{eq:BVP} can be obtained using constrained learning. To do so, we relax the BCs of~\eqref{eq:BVP} to relate the weak formulation~\eqref{eq:weak_form} to a distributionally robust constraint. The BCs are then reintroduced using a constrained formulation. We start with the following proposition, where~$W^{k,p}$ refers to the $(k,p)$-th order Sobolev space~(see App.~\ref{app:weak_formulation_to_statistical_problem}) and~$\calP^2(\calS)$ denotes the space of square-integrable probability distributions supported on~$\calS$.

\begin{proposition}\label{thm:weak_formulation_to_statistical_problem}
Let~$u^\dagger \in W^{k^\prime,2}(\calD)$, where $k^\prime \geq 1$ is the degree of the differential operator~$D$, be such that
$\sup_{\psi \in \mathcal{P}^2(\calD)} \E_{(x, t) \sim \psi} \!\Big[
   	\big(D [ u^\dagger ](x,t) - \tau(x, t) \big)^2
\Big]  = 0$.
If the dimension~$d$ of~$\Omega$ satisfies~$d \leq 4 k^\prime - 1$ , then~$u^\dagger$ satisfies~\eqref{eq:weak_form} with~$\calT = W^{k^\prime,2}(\calD)$.
\end{proposition}

A proof is provided in Appendix~\ref{app:weak_formulation_to_statistical_problem}. The equality constraint suggested by Prop.~\ref{thm:weak_formulation_to_statistical_problem} enforces~\eqref{eq:weak_form}, but does not impose the BCs of~\eqref{eq:BVP}. Since they must hold for all~$(x,t) \in \calB$, it is more appropriate to incorporate them using a worst-case loss, namely, $\sup_{(x,t) \in \calB} ( u_\theta(x,t) - h(x, t) )^2$, rather than an average loss as in~\eqref{eq:PINN}. As long as~$(x,t) \mapsto ( u_\theta(x,t) - h(x, t) )^2$ is a function in~$L^2$, this is equivalent to a distributionally robust loss similar to the one from Prop.~\ref{thm:weak_formulation_to_statistical_problem}~(see Appendix~\ref{app:weak_formulation_to_statistical_problem}). We therefore conclude that a weak solution of~\eqref{eq:BVP} is obtained by solving
\begin{prob}\label{eq:PINN_Linf}
	\minimize_{\theta \in \Theta}&
	&&\sup_{\psi \in \mathcal{P}^2(\calB)} \E_{(x, t) \sim \psi} \!\Big[
		\big( u_\theta(x,t) - h(x, t) \big)^2
	\Big]
	\\
	\subjectto& &&\sup_{\psi \in \mathcal{P}^2(\calD)} \E_{(x, t) \sim \psi} \!\Big[
	\big( D [ u_\theta ](x,t) - \tau(x, t) \big)^2
	\Big] \leq \epsilon
	\text{,}
\end{prob}
where~$\epsilon \geq 0$ controls the trade-off between fitting the PDE and the BCs when~$u_\theta$ is not expressive enough to satisfy both.

Prop.~\ref{thm:weak_formulation_to_statistical_problem} elucidates the challenges arising from the choice of collocation points in the unsupervised approach~\eqref{eq:PINN}, most notably PINNs. Indeed, it is not enough to use a fixed distribution~(e.g., uniform): satisfying~\eqref{eq:weak_form} requires training against all distributions~$\psi \in \calP^2$. The use of worst-case losses in a constrained formulation is what makes~\eqref{eq:PINN_Linf} considerably different from previous adaptive sampling methods and loss-weighting schemes. In fact, contrary to previous approaches, by (approximately) solving~\eqref{eq:PINN_Linf} (as detailed in Sec.~\ref{sec:algorithm}), we indeed (approximately) solve~\eqref{eq:BVP}. At the same time, Prop.~\ref{thm:weak_formulation_to_statistical_problem} establishes a limitation of learning-based solvers by restricting the smoothness of their solutions~(essentially, solutions in~$W^{(d+1)/4,2}$). This can be an issue for large-scale dynamical systems, such as those found in smart grid applications, or when transforming higher-order PDEs in higher-dimensional first-order systems. While Prop.~\ref{thm:weak_formulation_to_statistical_problem} describes a sufficient condition for~$u^\dagger$ to satisfy~\eqref{eq:weak_form}, a necessary condition can be obtained by restricting~$\calP^2$ to sufficiently smooth distributions, namely, those belonging to a Sobolev space.

\paragraph{Structural knowledge.}

The constrained form of~\eqref{eq:PINN_Linf} suggests that other information can be incorporated as long as they can be formulated as learning objectives, i.e., statistical losses. This is the case of certain forms of structural knowledge. Indeed, it is often possible to obtain information about the structure of the solution of a BVP, such as invariances or symmetries, without explicitly solving it~\citep{Olver79s, Akhound-Sadegh23l}. While this structural information is already encoded in~\eqref{eq:PINN_Linf}, using it explicitly can reduce training time as well as the number of both collocation points and/or observations~[as in~\eqref{eq:NO}] needed. It also helps ensure the physical validity of outcomes by explicitly avoiding degenerate solutions, a common failure mode of unsupervised methods such as PINNs~\citep{Krishnapriyan21c}~(we do not observe such issues with~\eqref{eq:PINN_Linf}, see Sec.~\ref{sec:pinn_experiments}). Structural knowledge can also be used to remove solution ambiguities, e.g., for the eikonal equation that is invariant to the sign of the solution~(see App.~\ref{app:pdes}).

While structural knowledge can sometimes be incorporated into the model~$u_\theta$, e.g., using equivariant architectures~\citep{Cohen16g, Batzner22e}, it can also be imposed as a worst-case constraint. This is convenient for when such models are intricate to design. Consider, for example, a (finite)~invariance group~$\calG$ whose elements~$\gamma_i$ act on the domain~$(x,t) \in \calD$ such that~$u^\dagger(x,t) = u^\dagger[\gamma_i(x,t)]$, where~$u^\dagger$ is a solution of~\eqref{eq:BVP}. This invariance can be enforced by
\begin{equation}\label{eq:adversarial_invariance_loss}
	\sup_{\psi \in \mathcal{P}^2(\calD)} \E_{(x, t) \sim \psi} \!\Big[
		\big( u_\theta(x, t) - u_\theta[\gamma_i (x,t)] \big)^2
	\Big]
	\leq \epsilon
		\text{,} \quad \gamma_i \in \calG
		\text{.}
\end{equation}
Notice that we use the same distributionally robust formulation as for the BCs in~\eqref{eq:PINN_Linf}. Similar constraints can be constructed for other structures, such as equivariance. In contrast to~\eqref{eq:PINN_Linf}, where we want~$\epsilon \approx 0$, it can be beneficial to use a larger values in~\eqref{eq:adversarial_invariance_loss} to accommodate models~$u_\theta$ that cannot fully capture the solution invariances.

\paragraph{Observational knowledge.}

In addition to mechanistic~(i.e., PDEs) and structural knowledge, we also consider (partial, noisy)~observations of the BVP solution, obtained either via classical methods~(e.g., FEM) or real-world measurements. Although this type of information is commonly associated with NOs, seen as they are typically trained in a supervised manner as in~\eqref{eq:NO}, it is not limited to that architecture. Given observations~$u^\dagger_j$, $j=1,\dots,J$, of solutions of~\eqref{eq:BVP}, we may formulate constraints of the kind
\begin{equation}\label{eq:observational}
\begin{aligned}
	\E_{(x, t) \sim \fkm} \!\Big[
		\big( u_{\theta}(x, t) - u_j^\dagger(x, t) \big)^2
	\Big] &\leq \epsilon
		\text{,} \quad j = 1,\dots,J
		\text{,}
\end{aligned}
\end{equation}
where~$\fkm$ is some distribution~(typically uniform) of points on~$\calD$. Note that~\eqref{eq:observational} is simply a different way of writing the~$L^2$-norm from the objective of~\eqref{eq:NO}. However, rather than averaging~$L^2$ losses, \eqref{eq:observational} constraints the maximum error across data points. By considering each sample individually, it accounts for the heterogeneous difficulty of fitting them and enables the tolerance~$\epsilon$ to be adjusted individually for each observation, e.g., using larger values for noisier samples.

\paragraph{Science-constrained learning.}

Combining~\eqref{eq:PINN_Linf} with~\eqref{eq:adversarial_invariance_loss} and~\eqref{eq:observational}, we are able to formulate a general SCL problem accommodating all knowledge sources considered so far. Explicitly,
\begin{prob}[\textup{SCL}]\label{eq:scl_problem}
	\minimize_{\theta \in \Theta}&
	&&\sup_{\psi \in \mathcal{P}^2(\calB)} \E_{(x, t) \sim \psi} \!\Big[
		\big( u_\theta(x,t) - h(x, t) \big)^2
	\Big]
		&&&\text{(M)}
	\\
	\subjectto& && \sup_{\psi \in \mathcal{P}^2(\calD)} \E_{(x, t) \sim \psi} \!\Big[
		\big( D [ u_\theta ](x,t) - \tau(x, t) \big)^2
	\Big] \leq \epsilon_\text{pde}
		&&&\text{(M)}
	\\
	&&&\sup_{\psi \in \mathcal{P}^2(\calD)} \E_{(x, t) \sim \psi} \!\Big[
		\big( u_\theta(x, t) - u_\theta[\gamma_i (x,t)] \big)^2
	\Big] \leq \epsilon_\text{s}
	\text{,} \quad \gamma_i \in \calG
		&&&\text{(S)}
	\\
	&&&\E_{(x, t) \sim \fkm} \!\Big[
		\big( u_{\theta}(x, t) - u_j^\dagger(x, t) \big)^2
	\Big] \leq \epsilon_\text{o}
		\text{,} \quad j = 1,\dots,J
		\text{.}
		&&&\text{(O)}
\end{prob}
Note that any subset of the constraints in~\eqref{eq:scl_problem} can be used depending on the available information. The objective of~\eqref{eq:scl_problem} is also not restricted to the BCs and can be replaced by any of the other terms. In fact, \eqref{eq:scl_problem} can be formulated without an objective, i.e., as a feasibility problem.

It is straightforward to extend~\eqref{eq:scl_problem} to simultaneously solve a parameterized family of BVPs. However, rather than discretizing the parameter space as in~\eqref{eq:PINN}, we rely on a worst-case formulation that considers all of its possible values rather than only a finite subset. Explicitly, we rewrite~\eqref{eq:scl_problem} as
\begin{equation}\tag{\textup{SCL${}^\prime$}}\label{eq:scl_parametric}
\begin{alignedat}{7}
	\minimize_{\theta \in \Theta}&\ \ %
	&&\sup_{\psi \in \mathcal{P}^2}\ \E_{(x, t, \pi) \sim \psi,\,\tau \sim \fkp} \!\Big[
	\big( u_\theta(\pi,\tau)(x,t) - h(x, t) \big)^2
	\Big]
	&&&\text{(M)}
	\\
	\subjectto& && \sup_{\psi \in \mathcal{P}^2}\ \E_{(x, t, \pi) \sim \psi,\,\tau \sim \fkp} \!\Big[
	\big( D_\pi [ u_\theta(\pi,\tau) ](x,t) - \tau(x, t) \big)^2
	\Big] \leq \epsilon_\text{pde}
	&&&\text{(M)}
	\\
	&&& \sup_{\psi \in \mathcal{P}^2}\ \E_{(x, t, \pi) \sim \psi,\,\tau \sim \fkp} \!\bigg[
	\Big( u_\theta(\pi,\tau)(x, t) - u_\theta(\pi,\tau)\big[\gamma_i(\pi) (x,t)\big] \Big)^2
	\bigg] \leq \epsilon_\text{s}
		\text{, } \gamma_i \in \calG
	\quad\,&&&\text{(S)}
	\\
	&&&\E_{(x, t) \sim \fkm} \!\Big[
	\big( u_{\theta}(\pi_j, \tau_j)(x, t) - u_j^\dagger(x, t) \big)^2
	\Big] \leq \epsilon_\text{o}
	\text{,} \quad j = 1,\dots,J
	\text{,}
	&&&\text{(O)}
\end{alignedat}
\end{equation}
where~$\fkm,\fkp$ are fixed distributions~(e.g., uniform), $u_j^\dagger$ is a solution of~\eqref{eq:BVP} with coefficients~$\pi_j$ and forcing function~$\tau_j$, and the invariance~$\gamma_i$ is now parametrized to account for the fact that its action may depend on~$\pi$~(e.g., translation invariance with different strides). Note that the~$\psi$ are now supported on~$\calB \times \Pi$ or~$\calD \times \Pi$, which we omit in~\eqref{eq:scl_parametric} for clarity. Hence, they target not only BVPs~(parameters $\pi$) that are hard to fit, but also the regions of the domain responsible for this difficulty, enabling performances that would require fine discretizations~(see Sec.~\ref{sec:experiments}). Yet, this approach is not directly applicable to infinite-dimensional parameters~(e.g.,~$\tau$) as it requires sampling from a function space. We leave this extension for future work, considering here a fixed distribution~$\fkp$.

In the next section, we develop a practical algorithm to tackle~\eqref{eq:scl_problem} and~\eqref{eq:scl_parametric} by (i)~leveraging non-convex duality results from constrained learning~\citep{Chamon20p, Chamon23c} and (ii)~deriving explicit approximations of the suprema over~$\psi$ from which we can sample efficiently.

\section{Algorithm}
\label{sec:algorithm}

To develop a practical algorithm for~\eqref{eq:scl_problem}~[and~\eqref{eq:scl_parametric}], we need to overcome the fact that it is (i)~a non-convex constrained optimization problem involving (ii)~worst-case losses. A typical approach to~(i) is to combine the all losses as penalties into a single training objective as in~\eqref{eq:PINN}. Though penalties and constraints are essentially equivalent in convex optimization~(strong duality, \citep{Bertsekas09c}), this is not the case in the non-convex setting of~\eqref{eq:scl_problem}. Hence, regardless of how the weights~$\mu$ in~\eqref{eq:PINN} are adapted~(e.g., \cite{Wang21u, Wight21s, Lu21p, Basir22p}), it need not provide a solution of~\eqref{eq:scl_problem}.

We overcome this issue by first tackling~(ii) using the following proposition:
\begin{proposition}\label{thm:robust_loss}
	Let~$z \mapsto \ell(z) \in L^2$. Then, for all~$\delta > 0$ there exists~$\alpha < \sup_z \ell(z)$ such that
	$\sup_{\psi \in \calP^2}\ \E_{z \sim \psi} \!\big[ \ell(z) \big]
		\leq \E_{z \sim \psi_{\alpha}} \!\big[ \ell(z) \big] + \delta
	$,
	where~$\calP^2 \ni \psi_{\alpha}(z) \propto \big[ \ell(z) - \alpha \big]_+$ for~$[a]_+ = \max(0, a)$.
\end{proposition}

A proof based on~\cite{Robey21a} can be found in App.~\ref{app:weak_formulation_to_statistical_problem}. Prop.~\eqref{thm:robust_loss} shows that the worst-case losses in~\eqref{eq:scl_problem}/\eqref{eq:scl_parametric} can be approximated arbitrarily well by an expectation with respect to~$\psi_\alpha$, a distribution proportional to a truncation of the underlying loss. For clarity, we consider~\eqref{eq:scl_problem} with only constraint~(M), but similar manipulations hold for the constraints~(S) and~(O) as well as~\eqref{eq:scl_parametric}. Explicitly, \eqref{eq:scl_problem}(M) can be written as
\begin{prob}\label{eq:scl_statistical}
	\minimize_{\theta \in \Theta}&
		&&\E_{(x, t) \sim \psi_\alpha^\text{bc}} \!\Big[
			\big( u_\theta(x,t) - h(x, t) \big)^2
		\Big]
	\\
	\subjectto& &&\E_{(x, t) \sim \psi_\alpha^\text{pde}} \!\Big[
		\big( D [ u_\theta ](x,t) - \tau(x, t) \big)^2
	\Big] \leq \epsilon
		\text{,}
\end{prob}
for~$\psi_\alpha^\text{bc}(x,t) \propto \big[
	\big( u_\theta(x,t) - h(x, t) \big)^2 - \alpha
\big]_+$
and $\psi_\alpha^\text{pde}(x,t) \propto\big[
	\big( D [ u_\theta ](x,t) - \tau(x, t) \big)^2 - \alpha
\big]_+$
supported on~$\calB$ and~$\calD$ respectively.

Observe that~\eqref{eq:scl_statistical} now has the form of a constrained learning problem, i.e., a constrained optimization problem with statistical losses. We can therefore use non-convex duality results from~\citep{Chamon20p, Chamon23c, Elenter24n} to show that, under typical conditions from (unconstrained)~learning theory and for rich enough parametrization, its solution can be approximated by solving the \emph{empirical dual problem}
\begin{prob}[\textup{DIV}]\label{eq:scl_dual}
	\maximize_{\lambda \geq 0}\ \min_{\theta \in \Theta}\ \hat{L}(\theta, \lambda)
	\text{,}
\end{prob}
where~$\hat{L}$ is the \emph{empirical Lagrangian} of~\eqref{eq:scl_statistical} based on samples~$(\data{bc}) \sim \psi_\alpha^\text{bc}$ and~$(\data{pde}) \sim \psi_\alpha^\text{pde}$, namely,
\begin{equation}\label{eq:lagrangian}
\begin{aligned}
	\hat{L}(\theta, \lambda) &\triangleq
		\frac{1}{N_\text{bc}} \sum_{n=1}^{N_\text{bc}}
			\big( u_{\theta}(\data{bc}) - h(\data{bc}) \big)^2
	\\
	{}&+ \lambda \Bigg[
		\frac{1}{N_\text{pde}} \sum_{n=1}^{N_\text{pde}} \big(
		D [ u_{\theta} ](\data{pde}) - \tau(\data{pde})
		\big)^2
		- \epsilon
	\Bigg]
		\text{.}
\end{aligned}
\end{equation}
Contrary to previous approaches based on~\eqref{eq:PINN}, such as~\citep{Wang21u, Wight21s, Daw23m}, \eqref{eq:scl_dual} truly approximates the solution of~\eqref{eq:BVP}. Indeed, \citep[Thm.~1]{Chamon23c} and~\citep[Thm.~3.1]{Elenter24n} guarantee that solutions of~\eqref{eq:scl_dual} are near-optimal and near-feasible for~\eqref{eq:scl_statistical} and, in view of~Prop.~\ref{thm:weak_formulation_to_statistical_problem} and~\eqref{thm:robust_loss}, \eqref{eq:BVP}~(see App.~\ref{app:generalization_results} for details). It is worth noting that this is only possible because~\eqref{eq:scl_statistical} is a \emph{statistical} problem. Although similar Lagrangian formulations have been used in~\cite{Lu21p, Basir22p}, they deal with deterministic constrained problem~(fixed collocation points) for which this duality does not hold.

From a practical perspective, \eqref{eq:scl_dual} does not require extensive hyperparameter tuning~[such as~$\mu$ in~\eqref{eq:PINN}], seen as~$\lambda$ is an optimization variable. What is more, despite non-convexity, the duality between~\eqref{eq:scl_statistical} and~\eqref{eq:scl_dual} allows the solution~$\lambda^\star$ of~\eqref{eq:scl_dual} to be interpreted as a sensitivity of the objective~(BC residuals) to small relaxations of~$\epsilon$~\citep{Chamon20p, Chamon23c, Hounie23r}. This information can be used to evaluate the fit of noisy measurements in~\eqref{eq:scl_problem}(O) or the reliability of solutions for different parameters~$\pi$ in~\eqref{eq:scl_parametric}~(see Sec.~\ref{sec:experiments}).
Finally, \eqref{eq:scl_dual} is amenable to practical algorithms such as dual ascent, which updates~$\lambda_0 = 0$ as
\begin{equation}\label{eq:dual_ascent}
	\lambda_{k+1} = \lambda_{k} + \eta \Bigg[
		\frac{1}{N_\text{pde}} \sum_{n=1}^{N_\text{pde}} \big(
			D [ u_{\theta^\dagger_k} ](\data{pde}) - \tau(\data{pde})
		\big)^2 - \epsilon
	\Bigg]_+
	\!\!\text{,}\ \ \text{for}\ %
		\theta^\dagger_k \in \argmin_{\theta \in \Theta}\ \hat{L}(\theta, \lambda_{k})
		\text{.}
\end{equation}
Even if the empirical Lagrangian minimizer~$\theta^\dagger_k$ is only computed approximately, \eqref{eq:dual_ascent} can be shown to converge to a neighborhood of a solution of~\eqref{eq:scl_dual}~\citep{Chamon23c, Elenter24n}.

Still, to turn~\eqref{eq:dual_ascent} into a practical algorithm, we need to obtain samples from~$\psi_\alpha$, which is only known implicitly~[see~\eqref{eq:scl_statistical}]. Nevertheless, since the losses in~\eqref{eq:scl_statistical} are non-negative, the~$\psi_\alpha$ are smooth, fully-supported, square-integrable distributions for~$\alpha = 0$. They are therefore amenable to be sampled using Markov Chain Monte Carlo~(MCMC) techniques~\citep{Robert04m}. We use the Metropolis-Hastings~(MH) algorithm in our experiments since it avoids additional backward passes and higher-order derivatives resulting from differentiating~$D_\pi$ while still providing strong empirical results~(see App.~\ref{app:mh_sampling} and Sec.~\ref{sec:experiments}). Exploring first-order methods~(e.g., Langevin Monte Carlo) and algorithms adapted to discontinuous distributions~(e.g., \citep{Nishimura20d}) to enable faster convergence~(reduce~$N$) and better approximations~(increase~$\alpha$) is left for future work.

It is possible to replace~$\psi_\alpha$ by a fixed distribution~(e.g., uniform) or even fixed points for some of the constraints in~\eqref{eq:scl_problem}/\eqref{eq:scl_parametric} at negligible accuracy costs. We find this is consistently the case for the BC objective, although we emphasize that this is application-dependent. Additionally, certain architectures, such as FNOs~\citep{Li21f}, cannot make predictions at arbitrary points of the domain. It is then more appropriate to take~$\psi_\alpha$ to be a uniform distribution over equispaced points.

The resulting method for solving~\eqref{eq:scl_parametric} is summarized in Alg.~\ref{alg:primal_dual}. Note that rather than obtaining Lagrangian minimizers as in~\eqref{eq:dual_ascent}, Alg.~\ref{alg:primal_dual} alternates between optimizing for~$\theta$~(step~8) and~$\lambda$~(step~9). Such gradient descent-ascent schemes are commonly used in convex optimization~\citep{Arrow58s, Bertsekas15c}, although their convergence is less well understood in non-convex settings~\cite{Lin20n, Fiez21g, Yang22n}. The convergence guarantees for~\eqref{eq:dual_ascent} can be recovered by repeating step~8 ($\theta$-update) multiple times before updating~$\lambda$.
Note that Alg.~\ref{alg:primal_dual} does not rely on training heuristics, such as adaptive or causal sampling~\cite{Krishnapriyan21c, McClenny23s, Daw23m, Penwarden23a, Wang24r}, \emph{ad hoc} weight updates~\citep{Wang21u, Maddu22i}, and conditional updates~\citep{Lu21p, Basir22p}. Indeed, steps 4-7 are empirical estimates of expectations with respect to~$\psi_0$ that themselves approximate the worst-case losses in~\ref{eq:scl_parametric}~(Prop.~\ref{thm:robust_loss}). Steps 8-9 describe a traditional primal-dual (gradient descent-ascent) algorithm for solving problems such as~\eqref{eq:scl_dual}, which itself yields approximate solutions of~\eqref{eq:scl_parametric}~\citep{Chamon20p, Chamon23c, Elenter24n} and, consequently, \eqref{eq:BVP}~(Prop.~\ref{thm:weak_formulation_to_statistical_problem}).

\begin{algorithm}[tb]
\caption{Primal-dual method for \eqref{eq:scl_problem}}
\label{alg:primal_dual}
\begin{algorithmic}[1]

\State \textbf{Inputs}: Differential operator~$D_\pi$, invariant transformations~$\gamma_i \in \calG$, observations set~$(\pi_j,\tau_j,u^\dagger_j)$, parameterized model~$u_{\theta_0}$, and~$\dvar{pde}_0 = \dvar{s${}_i$}_0 = \dvar{o${}_j$}_0 = 0$

\For{$k = 1, \ldots, K$}
	\State $\displaystyle
		\loss{bc}_{k} = \frac{1}{N_\text{BC}} \sum_{n=1}^{N_\text{BC}}
			\Big( u_{\theta_k}(\pi^\text{bc}_n)(\data{bc}) - h(\data{bc}) \Big)^2
	$, \hfill $(\datapar{bc}) \sim \psi_0^\text{BC}$

	\State $\displaystyle
		\loss{pde}_{k} = \frac{1}{N_\text{pde}} \sum_{n=1}^{N_\text{pde}} \Big(
			D_{\pi^\text{pde}_n} \left[ u_{\theta_k}(\pi^\text{pde}_n) \right](\data{pde}) - \tau(\data{pde})
		\Big)^2
	$, \hfill $(\datapar{pde}) \sim \psi_0^\text{PDE}$

	\State $\displaystyle
		\loss{s${}_i$}_{k} = \frac{1}{N_\text{s}} \sum_{n=1}^{N_\text{s}} \Big(
			u_{\theta_k}(\pi^\text{s${}_i$}_n)(\data{s${}_i$})
			- u_{\theta_k}(\pi^\text{s${}_i$}_n)
				\big[ \gamma_i(\pi^\text{s${}_i$}_n) (\data{s${}_i$}) \big]
		\Big)^2
	$, \hfill $(\datapar{s${}_i$}) \sim \psi_0^\text{ST${}_i$}$

	\State $\displaystyle
		\loss{o${}_j$}_{k} = \frac{1}{N_\text{o}} \sum_{n=1}^{N_\text{o}} \Big(
			u_{\theta_k}(\pi_j,\tau_j) (\data{o${}_j$}) - u_j^\dagger(\data{o${}_j$})
			\Big)^2
	$, \hfill $(\data{o${}_j$}) \sim \fkm$

	\State $\displaystyle
	\theta_{k+1} = \theta_{k} - \eta_p \bigg[
		\nabla_\theta \loss{bc}_{k}
		+ \dvar{pde} \nabla_\theta \loss{pde}_{k}
		+ \sum_{i=1}^I \dvar{s${}_i$}_{k} \nabla_\theta \loss{s${}_i$}_{k}
		+ \sum_{j=1}^J \dvar{o${}_j$}_{k} \nabla_\theta \loss{o${}_j$}_{k}
	\bigg]$

    \State $\displaystyle
		\dvar{pde}_{k+1} = \big[
		\dvar{pde}_{k} + \eta_d (\loss{pde}_{k} - \epsilon_\text{pde})
	\big]_+$;
	$\displaystyle
		\dvar{s}_{k+1} = \big[
		\dvar{s}_{k} + \eta_d (\loss{s}_{k} - \epsilon_\text{s})
	\big]_+$;
	$\displaystyle
		\dvar{o${}_j$}_{k+1} = \big[
		\dvar{o${}_j$}_{k} + \eta_d (\loss{o${}_j$}_{k} - \epsilon_\text{o})
	\big]_+$
\EndFor
\end{algorithmic}
\end{algorithm}

\section{Experiments}
\label{sec:experiments}

In this section, we showcase the use of SCL by training MLPs and FNOs~\citep{Li21f} to solve six PDEs~(convection, reaction-diffusion, eikonal, Burgers', diffusion-sorption, and Navier-Stokes). We consider different subsets of constraints from~\eqref{eq:scl_problem}/\eqref{eq:scl_parametric} to illustrate a variety of knowledge settings, but train only the most suitable model in each case since our goal is to illustrate the natural uses of SCL rather than exhaust its potential. Detailed descriptions are provided in the appendices, including BVPs~(App.~\ref{app:pdes}), training procedures~(App.~\ref{app:experimental_details}), and further results~(App.~\ref{app:additional_experiments}).
Code to reproduce these experiments is available at \url{https://github.com/vmoro1/scl}.
In the sequel, we use fixed points~$(x,t)$ for the BC objective rather than~$\psi_0^\text{BC}$ to illustrate how computational complexity can be reduced without significantly affecting the results. We still use~$\psi_0^\text{BC}$ for~$\pi$.

\subsection{Solving a specific BVP}
\label{sec:pinn_experiments}

\begin{table}[b]
	\renewcommand{\arraystretch}{1.35}
	\caption{Relative $L_2$ error for solving specific BVPs (average $\pm$ standard deviation across 10 seeds).}
	\label{tab:results_scl_pinn}
	\vspace{-1mm}
	\begin{center}
	\begin{tabular}{ll|cc|c}
		&& 	PINN~\eqref{eq:PINN} 	& R3~\citep{Daw23m} 		& \eqref{eq:scl_problem}(M)
		\\\hline
		\multirow{2}{*}{\textbf{Convection}}
		& \textbf{$\beta = 30$} &
			$1.17 \pm 0.65 \,\%$	& $0.999 \pm 0.53 \,\%$		& $0.971 \pm 0.30 \,\%$
		\\
		& \textbf{$\beta = 50$} &
			$56.5 \pm 20 \,\%$		& $29.0 \pm 33 \,\%$		& $3.74 \pm 0.87 \,\%$
		\\\hline
		\multirow{2}{*}{\textbf{React.-Diff.}}
		& \textbf{$(\nu, \rho) = (3, 3)$} &
			$0.745 \pm 0.014 \,\%$        & $0.736 \pm 0.068 \,\%$   & $1.82 \pm 0.74 \,\%$
		\\
		& \textbf{$(\nu, \rho) = (3, 5)$} &
			$79.6 \pm 0.27 \,\%$          & $0.665 \pm 0.046 \,\%$   & $0.762 \pm 0.11 \,\%$
		\\\hline
		\multicolumn{2}{l|}{\textbf{Eikonal}} &
			$87.1 \pm 39 \,\%$            & $85.5 \pm 27 \,\%$       & $9.95 \pm 1.9 \,\%$
		\\\hline
	\end{tabular}
	\end{center}
	\label{tab:pinn}
\end{table}

We begin by solving~\eqref{eq:BVP} for fixed parameters~$\pi$, forcing function~$\tau$, and BCs. Though this may not be the best application for NN-based solvers~[see, e.g., \cite{McGreivy2024w}], it remains a valid demonstration. We use~\eqref{eq:scl_problem}(M) to train MLPs to solve convection, reaction-diffusion, and eikonal problems, comparing the results to PINNs~[i.e.,~\eqref{eq:PINN} with weights chosen as in~\citep{Daw23m}]. All experiments use~$1000$ collocation points per epoch obtained by (PINN)~sampling uniformly~(we find this performs better than using fixed points as in, e.g.,~\citep{Raissi19p, Lu21p}), (R3)~using the adaptive heuristic from~\citep{Daw23m}, and (SCL)~using MH after~$4000$ burn-in steps.
Table~\ref{tab:pinn} shows that SCL matches and often outperforms other methods, particularly in challenging scenarios~(e.g., convection with~$\beta = 50$ or non-linear eikonal). This is because it jointly adapts the weight of each loss and the points used to evaluate them during training~(see Alg.~\ref{alg:primal_dual}). Although R3 also targets points with high PDE residuals, it does not approximate the worst-case loss needed to guarantee a (weak)~solution~(Prop.~\ref{thm:weak_formulation_to_statistical_problem}). The improved performance of SCL comes at a higher computational cost~($\approx 5$x), although this is largely offset when solving parametric problems.

\subsection{Solving parametric families of BVPs}
\label{sec:param_sol_experiments}

A key advantages of NN-based approaches is their ability to solve entire families of BVPs at once. Consider fitting an MLP~$u_\theta(x,t,\beta)$ to solve convection problems for~$\beta \in [1,30]$ using~\eqref{eq:scl_parametric}(M) and~\eqref{eq:PINN} as in~\citet{Cho24p}. We do not use their tailored ``\ppinn'' architecture, although it would be compatible with SCL. The~$\pi_j = \beta_j$ are taken to be 4, 7, and 30 equispaced values in~$[1,30]$.
Fig.~\ref{fig:param_sol_convection_helmholtz}a shows that~\eqref{eq:PINN} can only achieve the error of~\eqref{eq:scl_parametric}(M) for the finest discretization, at which point it evaluates the PDE loss 6~times more per epoch. The same pattern holds for reaction-diffusion~(4--7 times) and 2D~Helmholtz~(4--5 times) problems~(App.~\ref{app:additional_experiments}).
This effectiveness is due to the worst-case loss of~\eqref{eq:scl_parametric}(M) jointly selecting~$(x,t,\beta)$ to target challenging coefficients as well as the domain regions responsible for that difficulty. In fact, inspecting~$\psi_0^\text{PDE}$ at different stages of training~(Fig.~\ref{fig:param_sol_convection_helmholtz}b) shows that SCL first fits the solution ``causally,'' focusing on smaller values of~$t$. While this has been proposed in~\citep{Krishnapriyan21c, Penwarden23a, Wang24r}, it arises naturally by solving~\eqref{eq:scl_parametric}(M). As training advances, however, Alg.~\ref{alg:primal_dual} shifts focus to fitting higher convection speeds~$\beta$. This occurs \emph{without} any prior knowledge of the problems or manual tuning.

As we have argued, this is a use case for which SCL is particularly well-suited. Indeed, consider solving a 2D~Helmholtz equation for parameters~$(a_1, a_2) \in [1,2]^2$ using an FEM solver with mesh size chosen to obtain a similar accuracy as SCL. Across~$100$ experiments, the average relative~$L_2$ error for the FEM solver was~$0.036$ for an average runtime of $4.3$~minutes per solution~(see App.~\ref{app:additional_experiments}). Hence, in the time it took to train the SCL model~($31.1$~hours), we could evaluate less than~$440$ parameters combinations using the FEM solver. This is 20 times less than the~$10^4$ combinations used to evaluate the error of~\eqref{eq:scl_parametric}(M)~(average error~$0.013$ for a runtime of $4$~minutes). It is worth noting that these numbers are for a highly-optimized FEM implementation~\citep{Baratta23d}.

\begin{figure}[tb]
	\centering
	\begin{minipage}{0.41\linewidth}
		\centering
		\includegraphics[width=\linewidth]{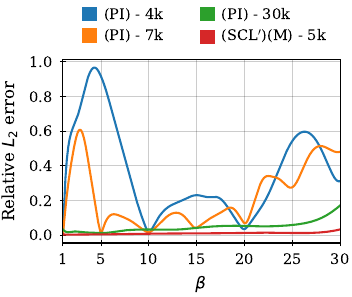}

		\hspace{7mm}{\small(a)}
	\end{minipage}
	\hfill
	\begin{minipage}{0.575\linewidth}
		\centering
		\includegraphics[width=\linewidth]{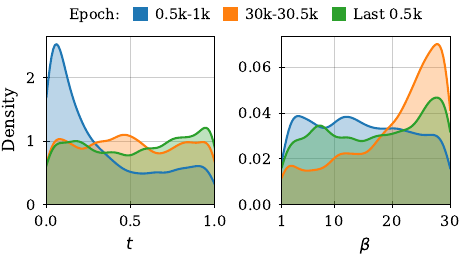}

		\hspace{7mm}{\small(b)}
	\end{minipage}
	\vspace{-1mm}
	\caption{Solving parametric convection BVPs:
		(a)~relative~$L_2$ error vs.~$\beta$~(legend reports number of differential operator evaluations per epoch). (b)~Samples from~$\psi_0^\text{PDE}$ at different training stages.}
	\label{fig:param_sol_convection_helmholtz}
\end{figure}

\subsection{Leveraging invariance when solving BVPs}
\label{sec:invariance_constraint_experiments}

Next, we showcase how SCL can be used to overcome computational limitations or scarce mechanistic knowledge. Consider training an MLP to solve a convection BVP with~$\beta = 30$ and periodic initial condition~$h(0,x) = \sin(x)$. This problem is commonly used to showcase a ``failure mode'' of~\eqref{eq:PINN} since it yields degenerate solutions when using fixed collocation points~(Fig.~\ref{fig:invariance_constraint}a)~\citep{Krishnapriyan21c}.
Note that the constrained~\eqref{eq:scl_problem}(M) also fails when using fixed collocations points~(Fig.~\ref{fig:invariance_constraint}b), even though its stochastic version finds accurate solutions~(Table~\ref{tab:pinn}).
One way to overcome this limitation is by leveraging additional knowledge. In this case, we know from the BVP structure that its solution must be periodic with period~$\pi/15$. By incorporating this invariance using~\eqref{eq:scl_problem}(S), we can obtain accurate solutions despite our use of fixed collocation points for~\eqref{eq:scl_problem}(M)~(Fig.~\ref{fig:invariance_constraint}c). Note that the worst-case loss in~\eqref{eq:scl_problem}(S) is fundamental to avoid degenerate solutions.

\begin{figure}[tb]
	\centering
	\includegraphics[width=\linewidth]{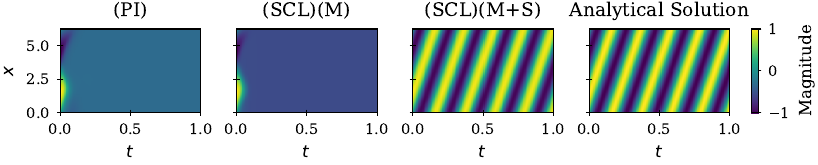}
	\caption{Using invariance in convection BVPs~(BC and PDE losses in~\eqref{eq:PINN} and~\eqref{eq:scl_problem} are evaluated using a \emph{fixed set of collocation points}).
	}
	\label{fig:invariance_constraint}
\end{figure}

\subsection{Supervised solution of BVPs}
\label{sec:pointwise_constrained_fno_experiments}

We conclude by exploring supervised methods that rely only on pairs of initial conditions~$h_j(x,0)$ and corresponding (weak)~solutions~$u^\dagger_j$. To showcase the versatility of SCL, we train FNOs rather than MLPs, fixing~$\fkm$ to a uniform distribution over a regular grid to accommodate their predictions. We also cast the SCL problem using only the observational constraints~\eqref{eq:scl_problem}(O), i.e., without any objective. In contrast to~\eqref{eq:NO}, which minimizes the average error, this formulation enforces a maximum error of~$\epsilon_\text{o}$ across samples. Though apparently minor, this leads to better prediction quality, as shown in Table~\ref{tab:results_NO}.
This occurs because the difficulty of fitting PDE solutions varies across ICs. While the average tends to emphasize the majority of ``easy-to-fit samples,'' enforcing a maximum error gives more weight to challenging data points. This becomes clear in Fig.~\ref{fig:dual_variables_burgers}, which compares the magnitude of each IC in the training set with its final dual variables~$\dvar{OB${}_j$}$ for the Burgers' equation. Immediately, we notice a trend where ICs with either small or large magnitudes appear harder to fit. This information can be leveraged to guide the collection of additional data points or improve the NO architecture. Indeed, any model improvement can immediately take advantage of SCL since it is (virtually)~independent of the choice of~$u_\theta$. The benefits of having~$\dvar{OB}$ are clear when predicting which IC is hard to fit is intricate, as is the case for the Navier-Stokes equation~(see App.~\ref{app:additional_experiments}).

\begin{figure}[tb]
		\begin{minipage}[c]{0.5\linewidth}
		\centering
		\captionof{table}{Relative $L_2$ error on test set (average across 10 seeds, see App.~\ref{app:additional_experiments} for standard deviation).}
		\label{tab:results_NO}
		\renewcommand{\arraystretch}{1.2}
		\begin{tabular}{lccc}
			& $\nu$ & \eqref{eq:NO} & \eqref{eq:scl_problem}(O)
			\\\hline
			\textbf{Burgers'}                       & $10^{-3}$ &$0.0540 \,\%$ & $0.0444 \,\%$
			\\\hline
			\multirow{3}{*}{\textbf{Navier-Stokes}} & $10^{-3}$ & $4.29 \,\%$ & $3.31 \,\%$  \\
													& $10^{-4}$ & $32.2 \,\%$ & $29.9 \,\%$  \\
													& $10^{-5}$ & $27.6 \,\%$ & $26.0 \,\%$
			\\\hline
			\multicolumn{2}{l}{\textbf{Diffusion-Sorption}} &$0.274 \,\%$ & $0.218 \,\%$
			\\\hline
		\end{tabular}
	\end{minipage}
	\hfill
	\begin{minipage}[c]{0.455\columnwidth}
		\centering
		\includegraphics[width=\columnwidth]{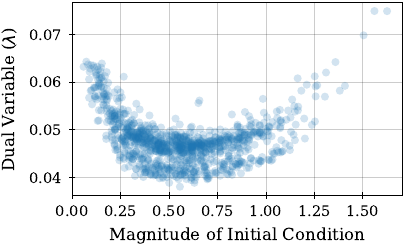}

		\caption{Final value of dual variables~($\dvar{OB${}_j$}$) vs.\ magnitude of IC for Burgers' equation.}
	\label{fig:dual_variables_burgers}
	\end{minipage}
\end{figure}

\section{Conclusion}

This paper developed SCL, a technique for solving BVPs based on constrained learning. It demonstrated that finding (weak)~solutions of PDEs is equivalent to solving constrained learning problems with worst-case losses, which also allows prior knowledge to be naturally incorporated to the solution of the BVP, e.g., structural constraints~(e.g., invariances), real-world measurements, and previously known solutions. It developed a practical algorithm to tackle SCL problems and showcased its performance across a variety of PDEs and NN~architectures. SCL not only yields accurate solutions, but also tackles many challenges faced by previous methods, such as extensive hyperparameter tuning, degenerate solutions, and fine discretizations of domain and/or coefficients. Future work includes exploring worst-case losses for functional parameters such as~$\tau$ and the use of SCL for active~learning.

\newpage
\subsubsection*{Acknowledgments}
This work was partly funded by the Deutsche Forschungsgemeinschaft (DFG, German Research Foundation) under Germany's Excellence Strategy (EXC 2075-390740016). It was performed in part on the HoreKa supercomputer funded by the Ministry of Science, Research and the Arts Baden-Württemberg and by the Federal Ministry of Education and Research. Viggo~Moro thanks the International Max Planck Research School for Intelligent Systems~(IMPRS-IS) for their support and G-Research for supporting the travel costs.

\bibliographystyle{iclr2025_conference}
\bibliography{aux_files/IEEEabrv.bib,aux_files/af.bib,aux_files/bayes.bib,aux_files/control.bib,aux_files/gsp.bib,aux_files/math,aux_files/ml.bib,aux_files/rkhs.bib,aux_files/rl.bib,aux_files/sp.bib,aux_files/stat.bib,aux_files/vm.bib}

\begin{thebibliography}{99}
\providecommand{\natexlab}[1]{#1}
\providecommand{\url}[1]{\texttt{#1}}
\expandafter\ifx\csname urlstyle\endcsname\relax
  \providecommand{\doi}[1]{doi: #1}\else
  \providecommand{\doi}{doi: \begingroup \urlstyle{rm}\Url}\fi

\bibitem[Akhound-Sadegh et~al.(2023)Akhound-Sadegh, Perreault-Levasseur, Brandstetter, Welling, and Ravanbakhsh]{Akhound-Sadegh23l}
Tara Akhound-Sadegh, Laurence Perreault-Levasseur, Johannes Brandstetter, Max Welling, and Siamak Ravanbakhsh.
\newblock Lie point symmetry and physics informed networks, 2023.

\bibitem[Arrow et~al.(1958)Arrow, Hurwicz, and Uzawa]{Arrow58s}
K.~J. Arrow, L.~Hurwicz, and H.~Uzawa.
\newblock \emph{Studies in linear and non-linear programming}.
\newblock Stanford University Press, 1958.

\bibitem[Baratta et~al.(2023)Baratta, Dean, Dokken, Habera, Hale, Richardson, Rognes, Scroggs, Sime, and Wells]{Baratta23d}
Igor~A. Baratta, Joseph~P. Dean, J{\o}rgen~S. Dokken, Michal Habera, Jack~S. Hale, Chris~N. Richardson, Marie~E. Rognes, Matthew~W. Scroggs, Nathan Sime, and Garth~N. Wells.
\newblock {DOLFINx}: the next generation {FEniCS} problem solving environment, 2023.

\bibitem[Basir \& Senocak(2022)Basir and Senocak]{Basir22p}
Shamsulhaq Basir and Inanc Senocak.
\newblock Physics and equality constrained artificial neural networks: Application to forward and inverse problems with multi-fidelity data fusion.
\newblock \emph{Journal of Computational Physics}, 463:\penalty0 111301, 2022.

\bibitem[Batzner et~al.(2022)Batzner, Musaelian, Sun, Geiger, Mailoa, Kornbluth, Molinari, Smidt, and Kozinsky]{Batzner22e}
Simon Batzner, Albert Musaelian, Lixin Sun, Mario Geiger, Jonathan~P Mailoa, Mordechai Kornbluth, Nicola Molinari, Tess~E Smidt, and Boris Kozinsky.
\newblock {E}(3)-equivariant graph neural networks for data-efficient and accurate interatomic potentials.
\newblock \emph{Nature communications}, 13\penalty0 (1):\penalty0 2453, 2022.

\bibitem[Bertsekas(2009)]{Bertsekas09c}
D.~P. Bertsekas.
\newblock \emph{Convex optimization theory}.
\newblock Athena Scientific, 2009.

\bibitem[Bertsekas(2015)]{Bertsekas15c}
D.~P. Bertsekas.
\newblock \emph{Convex optimization algorithms}.
\newblock Athena Scientific, 2015.

\bibitem[Boroun et~al.(2023)Boroun, Alizadeh, and Jalilzadeh]{Boroun23a}
Morteza Boroun, Zeinab Alizadeh, and Afrooz Jalilzadeh.
\newblock Accelerated primal-dual scheme for a class of stochastic nonconvex-concave saddle point problems.
\newblock In \emph{American Control Conference}, pp.\  204--209, 2023.

\bibitem[Brenner \& Scott(2007)Brenner and Scott]{Brenner07t}
S.~Brenner and R.~Scott.
\newblock \emph{The Mathematical Theory of Finite Element Methods}.
\newblock Texts in Applied Mathematics. Springer New York, 2007.

\bibitem[Chakraborty(2021)]{Chakraborty21t}
Souvik Chakraborty.
\newblock Transfer learning based multi-fidelity physics informed deep neural network.
\newblock \emph{Journal of Computational Physics}, 426:\penalty0 109942, 2021.

\bibitem[Chalapathi et~al.(2024)Chalapathi, Du, and Krishnapriyan]{Chalapathi24s}
Nithin Chalapathi, Yiheng Du, and Aditi~S. Krishnapriyan.
\newblock Scaling physics-informed hard constraints with mixture-of-experts.
\newblock In \emph{International Conference on Learning Representations}, 2024.

\bibitem[Chamon \& Ribeiro(2020)Chamon and Ribeiro]{Chamon20p}
L.~F.~O. Chamon and A.~Ribeiro.
\newblock Probably approximately correct constrained learning.
\newblock In \emph{Advances in Neural Information Processing}, 2020.

\bibitem[Chamon et~al.(2023)Chamon, Paternain, {Calvo-Fullana}, and Ribeiro]{Chamon23c}
L.~F.~O. Chamon, S.~Paternain, M.~{Calvo-Fullana}, and A.~Ribeiro.
\newblock Constrained learning with non-convex losses.
\newblock \emph{IEEE Trans. on Inf. Theory}, 69[3]:\penalty0 1739--1760, 2023.

\bibitem[Chen et~al.(2020)Chen, Lu, Karniadakis, and Negro]{Chen20p}
Yuyao Chen, Lu~Lu, George~Em Karniadakis, and Luca~Dal Negro.
\newblock Physics-informed neural networks for inverse problems in nano-optics and metamaterials.
\newblock \emph{Opt. Express}, 28\penalty0 (8):\penalty0 11618--11633, 2020.

\bibitem[Cheng et~al.(2022)Cheng, Lei, Chen, Dhillon, and Hsieh]{Cheng22c}
Minhao Cheng, Qi~Lei, Pin-Yu Chen, Inderjit Dhillon, and Cho-Jui Hsieh.
\newblock {CAT}: {C}ustomized adversarial training for improved robustness.
\newblock In \emph{International Joint Conference on Artificial Intelligence}, pp.\  673--679, 2022.

\bibitem[Chiu et~al.(2022)Chiu, Wong, Ooi, Dao, and Ong]{Chiu22c}
Pao-Hsiung Chiu, Jian~Cheng Wong, Chinchun Ooi, My~Ha Dao, and Yew-Soon Ong.
\newblock Can-pinn: A fast physics-informed neural network based on coupled-automatic–numerical differentiation method.
\newblock \emph{Computer Methods in Applied Mechanics and Engineering}, 395:\penalty0 114909, 2022.

\bibitem[Cho et~al.(2024)Cho, Jo, Lim, Lee, Lee, Hong, and Park]{Cho24p}
Woojin Cho, Minju Jo, Haksoo Lim, Kookjin Lee, Dongeun Lee, Sanghyun Hong, and Noseong Park.
\newblock Parameterized physics-informed neural networks for parameterized {PDE}s.
\newblock In \emph{International Conference on Machine Learning}, 2024.

\bibitem[Cohen \& Welling(2016)Cohen and Welling]{Cohen16g}
Taco Cohen and Max Welling.
\newblock Group equivariant convolutional networks.
\newblock In \emph{International Conference on Machine Learning}, pp.\  2990--2999, 2016.

\bibitem[Cotter et~al.(2019)Cotter, Jiang, Gupta, Wang, Narayan, You, and Sridharan]{Cotter19o}
Andrew Cotter, Heinrich Jiang, Maya Gupta, Serena Wang, Taman Narayan, Seungil You, and Karthik Sridharan.
\newblock Optimization with non-differentiable constraints with applications to fairness, recall, churn, and other goals.
\newblock \emph{Journal of Machine Learning Research}, 20\penalty0 (172):\penalty0 1--59, 2019.

\bibitem[Daw et~al.(2023)Daw, Bu, Wang, Perdikaris, and Karpatne]{Daw23m}
Arka Daw, Jie Bu, Sifan Wang, Paris Perdikaris, and Anuj Karpatne.
\newblock Mitigating propagation failures in physics-informed neural networks using retain-resample-release ({R}3) sampling.
\newblock In \emph{International Conference on Machine Learning}, pp.\  7264--7302, 2023.

\bibitem[Desai et~al.(2022)Desai, Mattheakis, Joy, Protopapas, and Roberts]{Desai22o}
Shaan Desai, Marios Mattheakis, Hayden Joy, Pavlos Protopapas, and Stephen~J. Roberts.
\newblock One-shot transfer learning of physics-informed neural networks.
\newblock In \emph{AI for Science Workshop (ICML)}, 2022.

\bibitem[Dhillon et~al.(2018)Dhillon, Azizzadenesheli, Bernstein, Kossaifi, Khanna, Lipton, and Anandkumar]{Dhillon18s}
Guneet~S. Dhillon, Kamyar Azizzadenesheli, Jeremy~D. Bernstein, Jean Kossaifi, Aran Khanna, Zachary~C. Lipton, and Animashree Anandkumar.
\newblock Stochastic activation pruning for robust adversarial defense.
\newblock In \emph{International Conference on Learning Representations}, 2018.

\bibitem[Dissanayake \& Phan-Thien(1994)Dissanayake and Phan-Thien]{Dissanayake94n}
Mahesh Dissanayake and Nhan Phan-Thien.
\newblock Neural-network-based approximations for solving partial differential equations.
\newblock \emph{Communications in Numerical Methods in Engineering}, 10:\penalty0 195--201, 1994.

\bibitem[Elenter et~al.(2024)Elenter, Chamon, and Ribeiro]{Elenter24n}
J.~Elenter, L.~F.~O. Chamon, and A.~Ribeiro.
\newblock Near-optimal solutions of constrained learning problems.
\newblock In \emph{International Conference on Learning Representations}, 2024.

\bibitem[Evans(2010)]{Evans10p}
L.C. Evans.
\newblock \emph{Partial Differential Equations}.
\newblock Graduate studies in mathematics. American Mathematical Society, 2010.

\bibitem[Fathony et~al.(2021)Fathony, Sahu, Willmott, and Kolter]{Fathony21m}
Rizal Fathony, Anit~Kumar Sahu, Devin Willmott, and J~Zico Kolter.
\newblock Multiplicative filter networks.
\newblock In \emph{International Conference on Learning Representations}, 2021.

\bibitem[Fiez et~al.(2021)Fiez, Ratliff, Mazumdar, Faulkner, and Narang]{Fiez21g}
Tanner Fiez, Lillian Ratliff, Eric Mazumdar, Evan Faulkner, and Adhyyan Narang.
\newblock Global convergence to local minmax equilibrium in classes of nonconvex zero-sum games.
\newblock In \emph{Advances in Neural Information Processing Systems}, pp.\  29049--29063, 2021.

\bibitem[Fiorenza et~al.(2021)Fiorenza, Formica, Roskovec, and Soudsk{\`y}]{Fiorenza21d}
Alberto Fiorenza, Maria~Rosaria Formica, Tom{\'a}{\v{s}}~G. Roskovec, and Filip Soudsk{\`y}.
\newblock Detailed proof of classical {G}agliardo-{N}irenberg interpolation inequality with historical remarks.
\newblock \emph{Zeitschrift f{\"u}r Analysis und ihre Anwendungen}, 40\penalty0 (2):\penalty0 217--236, 2021.

\bibitem[Gao et~al.(2021)Gao, Sun, and Wang]{Gao21p}
Han Gao, Luning Sun, and Jian-Xun Wang.
\newblock Phygeonet: Physics-informed geometry-adaptive convolutional neural networks for solving parameterized steady-state pdes on irregular domain.
\newblock \emph{Journal of Computational Physics}, 428:\penalty0 110079, 2021.

\bibitem[Goodfellow et~al.(2015)Goodfellow, Shlens, and Szegedy]{Goodfellow15e}
Ian~J. Goodfellow, Jonathon Shlens, and Christian Szegedy.
\newblock Explaining and harnessing adversarial examples.
\newblock In \emph{International Conference on Learning Representations}, 2015.

\bibitem[Goswami et~al.(2020)Goswami, Anitescu, Chakraborty, and Rabczuk]{Goswami20t}
Somdatta Goswami, Cosmin Anitescu, Souvik Chakraborty, and Timon Rabczuk.
\newblock Transfer learning enhanced physics informed neural network for phase-field modeling of fracture.
\newblock \emph{Theoretical and Applied Fracture Mechanics}, 106:\penalty0 102447, 2020.

\bibitem[Grossmann et~al.(2024)Grossmann, Komorowska, Latz, and Schönlieb]{Grossmann24c}
Tamara~G Grossmann, Urszula~Julia Komorowska, Jonas Latz, and Carola-Bibiane Schönlieb.
\newblock Can physics-informed neural networks beat the finite element method?
\newblock \emph{IMA Journal of Applied Mathematics}, 89\penalty0 (1):\penalty0 143--174, 2024.

\bibitem[Gupta et~al.(2021)Gupta, Xiao, and Bogdan]{Gupta21m}
Gaurav Gupta, Xiongye Xiao, and Paul Bogdan.
\newblock Multiwavelet-based operator learning for differential equations.
\newblock \emph{Advances in neural information processing systems}, 34:\penalty0 24048--24062, 2021.

\bibitem[Gupta \& Brandstetter(2023)Gupta and Brandstetter]{Gupta23t}
Jayesh~K Gupta and Johannes Brandstetter.
\newblock Towards multi-spatiotemporal-scale generalized {PDE} modeling.
\newblock \emph{Transactions on Machine Learning Research}, 2023.

\bibitem[Hao et~al.(2023)Hao, Wang, Su, Ying, Dong, Liu, Cheng, Song, and Zhu]{Hao23g}
Zhongkai Hao, Zhengyi Wang, Hang Su, Chengyang Ying, Yinpeng Dong, Songming Liu, Ze~Cheng, Jian Song, and Jun Zhu.
\newblock {GNOT}: {A} general neural operator transformer for operator learning.
\newblock In \emph{International Conference on Machine Learning}, pp.\  12556--12569, 2023.

\bibitem[Helwig et~al.(2023)Helwig, Zhang, Fu, Kurtin, Wojtowytsch, and Ji]{Helwig23g}
Jacob Helwig, Xuan Zhang, Cong Fu, Jerry Kurtin, Stephan Wojtowytsch, and Shuiwang Ji.
\newblock Group equivariant {F}ourier neural operators for partial differential equations.
\newblock \emph{arXiv preprint arXiv:2306.05697}, 2023.

\bibitem[Hornik(1991)]{Hornik91a}
Kurt Hornik.
\newblock Approximation capabilities of multilayer feedforward networks.
\newblock \emph{Neural networks}, 4\penalty0 (2):\penalty0 251--257, 1991.

\bibitem[Hounie et~al.(2023{\natexlab{a}})Hounie, Chamon, and Ribeiro]{Hounie23a}
I.~Hounie, L.~F.~O. Chamon, and A.~Ribeiro.
\newblock Automatic data augmentation via invariance-constrained learning.
\newblock In \emph{International Conference on Machine Learning}, 2023{\natexlab{a}}.

\bibitem[Hounie et~al.(2023{\natexlab{b}})Hounie, Ribeiro, and Chamon]{Hounie23r}
I.~Hounie, A.~Ribeiro, and L.~F.~O. Chamon.
\newblock Resilient constrained learning.
\newblock In \emph{Advances in Neural Information Processing}, 2023{\natexlab{b}}.

\bibitem[Jagtap et~al.(2020)Jagtap, Kawaguchi, and Karniadakis]{Jagtap20a}
Ameya~D. Jagtap, Kenji Kawaguchi, and George~Em Karniadakis.
\newblock Adaptive activation functions accelerate convergence in deep and physics-informed neural networks.
\newblock \emph{Journal of Computational Physics}, 404:\penalty0 109136, 2020.

\bibitem[Jarner \& Hansen(2000)Jarner and Hansen]{Jarner00g}
S{\o}ren~Fiig Jarner and Ernst Hansen.
\newblock Geometric ergodicity of metropolis algorithms.
\newblock \emph{Stochastic Processes and their Applications}, 85\penalty0 (2):\penalty0 341--361, 2000.

\bibitem[Kang et~al.(2023)Kang, Lee, Hong, Yun, and Park]{Kang23p}
Namgyu Kang, Byeonghyeon Lee, Youngjoon Hong, Seok-Bae Yun, and Eunbyung Park.
\newblock Pixel: Physics-informed cell representations for fast and accurate pde solvers.
\newblock In \emph{AAAI Conference on Artificial Intelligence}, pp.\  8186--8194, 2023.

\bibitem[Katsikadelis(2016)]{Katsikadelis16t}
J.T. Katsikadelis.
\newblock \emph{The Boundary Element Method for Engineers and Scientists: Theory and Applications}.
\newblock Elsevier Science, 2016.

\bibitem[Kearns et~al.(2018)Kearns, Neel, Roth, and Wu]{Kearns18p}
Michael Kearns, Seth Neel, Aaron Roth, and Zhiwei~Steven Wu.
\newblock Preventing fairness {G}errymandering: {A}uditing and learning for subgroup fairness.
\newblock In \emph{International Conference on Machine Learning}, pp.\  2564--2572, 2018.

\bibitem[Kharazmi et~al.(2021)Kharazmi, Zhang, and Karniadakis]{Kharazmi21h}
Ehsan Kharazmi, Zhongqiang Zhang, and George~E.M. Karniadakis.
\newblock hp-vpinns: Variational physics-informed neural networks with domain decomposition.
\newblock \emph{Computer Methods in Applied Mechanics and Engineering}, 374:\penalty0 113547, 2021.

\bibitem[Kingma \& Ba(2017)Kingma and Ba]{Kingma17a}
Diederik~P. Kingma and Jimmy Ba.
\newblock Adam: A method for stochastic optimization, 2017.
\newblock arXiv:1412.6980v9.

\bibitem[Knabner \& Angerman(2003)Knabner and Angerman]{Knabner03n}
P.~Knabner and L.~Angerman.
\newblock \emph{Numerical Methods for Elliptic and Parabolic Partial Differential Equations}.
\newblock Texts in Applied Mathematics. Springer New York, 2003.

\bibitem[Kovachki et~al.(2023)Kovachki, Li, Liu, Azizzadenesheli, Bhattacharya, Stuart, and Anandkumar]{Kovachki23n}
Nikola Kovachki, Zongyi Li, Burigede Liu, Kamyar Azizzadenesheli, Kaushik Bhattacharya, Andrew Stuart, and Anima Anandkumar.
\newblock Neural operator: Learning maps between function spaces with applications to {PDE}s.
\newblock \emph{Journal of Machine Learning Research}, 24\penalty0 (89):\penalty0 1--97, 2023.

\bibitem[Krishnapriyan et~al.(2021)Krishnapriyan, Gholami, Zhe, Kirby, and Mahoney]{Krishnapriyan21c}
Aditi Krishnapriyan, Amir Gholami, Shandian Zhe, Robert Kirby, and Michael~W Mahoney.
\newblock Characterizing possible failure modes in physics-informed neural networks.
\newblock In \emph{Advances in Neural Information Processing}, volume~34, pp.\  26548--26560, 2021.

\bibitem[Lagaris et~al.(1998)Lagaris, Likas, and Fotiadis]{Lagaris98a}
I.E. Lagaris, A.~Likas, and D.I. Fotiadis.
\newblock Artificial neural networks for solving ordinary and partial differential equations.
\newblock \emph{{IEEE} Trans. Neural Netw.}, 9\penalty0 (5):\penalty0 987--1000, 1998.

\bibitem[LeVeque(2007)]{LeVeque07f}
Randall~J. LeVeque.
\newblock \emph{Finite Difference Methods for Ordinary and Partial Differential Equations}.
\newblock Society for Industrial and Applied Mathematics, 2007.

\bibitem[Li et~al.(2020)Li, Kovachki, Azizzadenesheli, Liu, Stuart, Bhattacharya, and Anandkumar]{Li20m}
Zongyi Li, Nikola Kovachki, Kamyar Azizzadenesheli, Burigede Liu, Andrew Stuart, Kaushik Bhattacharya, and Anima Anandkumar.
\newblock Multipole graph neural operator for parametric partial differential equations.
\newblock In \emph{Advances in Neural Information Processing Systems}, pp.\  6755--6766, 2020.

\bibitem[Li et~al.(2021)Li, Kovachki, Azizzadenesheli, liu, Bhattacharya, Stuart, and Anandkumar]{Li21f}
Zongyi Li, Nikola~Borislavov Kovachki, Kamyar Azizzadenesheli, Burigede liu, Kaushik Bhattacharya, Andrew Stuart, and Anima Anandkumar.
\newblock Fourier neural operator for parametric partial differential equations.
\newblock In \emph{International Conference on Learning Representations}, 2021.

\bibitem[Li et~al.(2023)Li, Huang, Liu, and Anandkumar]{Li23f}
Zongyi Li, Daniel~Zhengyu Huang, Burigede Liu, and Anima Anandkumar.
\newblock Fourier neural operator with learned deformations for pdes on general geometries.
\newblock \emph{Journal of Machine Learning Research}, 24\penalty0 (388):\penalty0 1--26, 2023.

\bibitem[Li et~al.(2024)Li, Zheng, Kovachki, Jin, Chen, Liu, Azizzadenesheli, and Anandkumar]{Li24p}
Zongyi Li, Hongkai Zheng, Nikola Kovachki, David Jin, Haoxuan Chen, Burigede Liu, Kamyar Azizzadenesheli, and Anima Anandkumar.
\newblock Physics-informed neural operator for learning partial differential equations.
\newblock \emph{ACM / IMS J. Data Sci.}, 1\penalty0 (3), 2024.

\bibitem[Lin et~al.(2020)Lin, Jin, and Jordan]{Lin20n}
Tianyi Lin, Chi Jin, and Michael~I. Jordan.
\newblock Near-optimal algorithms for minimax optimization.
\newblock In \emph{Conference on Learning Theory}, pp.\  2738--2779, 2020.

\bibitem[Lu et~al.(2021{\natexlab{a}})Lu, Jin, Pang, Zhang, and Karniadakis]{Lu21l}
Lu~Lu, Pengzhan Jin, Guofei Pang, Zhongqiang Zhang, and George~Em Karniadakis.
\newblock Learning nonlinear operators via {DeepONet} based on the universal approximation theorem of operators.
\newblock \emph{Nature Machine Intelligence}, 3\penalty0 (3):\penalty0 218--229, 2021{\natexlab{a}}.

\bibitem[Lu et~al.(2021{\natexlab{b}})Lu, Pestourie, Yao, Wang, Verdugo, and Johnson]{Lu21p}
Lu~Lu, Rapha{\"e}l Pestourie, Wenjie Yao, Zhicheng Wang, Francesc Verdugo, and Steven~G. Johnson.
\newblock Physics-informed neural networks with hard constraints for inverse design.
\newblock \emph{SIAM Journal on Scientific Computing}, 43\penalty0 (6):\penalty0 B1105--B1132, 2021{\natexlab{b}}.

\bibitem[Lustig et~al.(2008)Lustig, Donoho, Santos, and Pauly]{Lustig08c}
Michael Lustig, David~L Donoho, Juan~M Santos, and John~M Pauly.
\newblock Compressed sensing {MRI}.
\newblock \emph{IEEE signal processing magazine}, 25\penalty0 (2):\penalty0 72--82, 2008.

\bibitem[Maddu et~al.(2022)Maddu, Sturm, M{\"u}ller, and Sbalzarini]{Maddu22i}
Suryanarayana Maddu, Dominik Sturm, Christian~L M{\"u}ller, and Ivo~F Sbalzarini.
\newblock Inverse {D}irichlet weighting enables reliable training of physics informed neural networks.
\newblock \emph{Machine Learning: Science and Technology}, 3\penalty0 (1):\penalty0 015026, 2022.

\bibitem[Markidis(2021)]{Markidis21t}
Stefano Markidis.
\newblock The old and the new: Can physics-informed deep-learning replace traditional linear solvers?
\newblock \emph{Frontiers in Big Data}, 4, 2021.

\bibitem[McClenny \& Braga-Neto(2023)McClenny and Braga-Neto]{McClenny23s}
Levi~D. McClenny and Ulisses~M. Braga-Neto.
\newblock Self-adaptive physics-informed neural networks.
\newblock \emph{Journal of Computational Physics}, 474:\penalty0 111722, 2023.

\bibitem[McGreivy \& Hakim(2024)McGreivy and Hakim]{McGreivy2024w}
Nick McGreivy and Ammar Hakim.
\newblock Weak baselines and reporting biases lead to overoptimism in machine learning for fluid-related partial differential equations.
\newblock \emph{Nature Machine Intelligence}, 6\penalty0 (10):\penalty0 1256--1269, 2024.

\bibitem[M\k{a}dry et~al.(2018)M\k{a}dry, Makelov, Schmidt, Tsipras, and Vladu]{Madry18t}
Aleksander M\k{a}dry, Aleksandar Makelov, Ludwig Schmidt, Dimitris Tsipras, and Adrian Vladu.
\newblock Towards deep learning models resistant to adversarial attacks.
\newblock In \emph{International Conference on Learning Representations}, 2018.

\bibitem[Molesky et~al.(2018)Molesky, Lin, Piggott, Jin, Vuckovi{\'{c}}, and Rodriguez]{Molesky18i}
Sean Molesky, Zin Lin, Alexander~Y. Piggott, Weiliang Jin, Jelena Vuckovi{\'{c}}, and Alejandro~W. Rodriguez.
\newblock Inverse design in nanophotonics.
\newblock \emph{Nature Photonics}, 12\penalty0 (11):\penalty0 659--670, 2018.

\bibitem[Moseley et~al.(2023)Moseley, Markham, and Nissen-Meyer]{Moseley23f}
Ben Moseley, Andrew Markham, and Tarje Nissen-Meyer.
\newblock Finite basis physics-informed neural networks ({FBPINNs}): {A} scalable domain decomposition approach for solving differential equations.
\newblock \emph{Advances in Computational Mathematics}, 49\penalty0 (4):\penalty0 62, 2023.

\bibitem[Musekamp et~al.(2024)Musekamp, Kalimuthu, Holzm{\"{u}}ller, Takamoto, and Niepert]{Musekamp24a}
Daniel Musekamp, Marimuthu Kalimuthu, David Holzm{\"{u}}ller, Makoto Takamoto, and Mathias Niepert.
\newblock Active learning for neural {PDE} solvers.
\newblock In \emph{Advances in Neural Information Processing. Workshop on Data-driven and Differentiable Simulations, Surrogates, and Solvers.}, 2024.

\bibitem[Nabian et~al.(2021)Nabian, Gladstone, and Meidani]{Nabian21e}
Mohammad~Amin Nabian, Rini~Jasmine Gladstone, and Hadi Meidani.
\newblock Efficient training of physics-informed neural networks via importance sampling.
\newblock \emph{Computer-Aided Civil and Infrastructure Engineering}, 36\penalty0 (8):\penalty0 962--977, 2021.

\bibitem[Nirenberg(1959)]{Nirenberg59o}
Louis Nirenberg.
\newblock On elliptic partial differential equations.
\newblock \emph{Annali della Scuola Normale Superiore di Pisa-Scienze Fisiche e Matematiche}, 13\penalty0 (2):\penalty0 115--162, 1959.

\bibitem[Nishimura et~al.(2020)Nishimura, Dunson, and Lu]{Nishimura20d}
Akihiko Nishimura, David~B Dunson, and Jianfeng Lu.
\newblock Discontinuous {H}amiltonian {M}onte {C}arlo for discrete parameters and discontinuous likelihoods.
\newblock \emph{Biometrika}, 107\penalty0 (2):\penalty0 365--380, 2020.

\bibitem[Olver(1979)]{Olver79s}
Peter~J. Olver.
\newblock Symmetry groups and group invariant solutions of partial differential equations.
\newblock \emph{J. Differential Geometry}, 14:\penalty0 497--542, 1979.

\bibitem[Patel et~al.(2022)Patel, Manickam, Trask, Wood, Lee, Tomas, and Cyr]{Patel22t}
Ravi~G. Patel, Indu Manickam, Nathaniel~A. Trask, Mitchell~A. Wood, Myoungkyu Lee, Ignacio Tomas, and Eric~C. Cyr.
\newblock Thermodynamically consistent physics-informed neural networks for hyperbolic systems.
\newblock \emph{Journal of Computational Physics}, 449:\penalty0 110754, 2022.

\bibitem[Penwarden et~al.(2023)Penwarden, Jagtap, Zhe, Karniadakis, and Kirby]{Penwarden23a}
Michael Penwarden, Ameya~D. Jagtap, Shandian Zhe, George~Em Karniadakis, and Robert~M. Kirby.
\newblock A unified scalable framework for causal sweeping strategies for {P}hysics-{I}nformed {N}eural {N}etworks ({PINNs}) and their temporal decompositions.
\newblock \emph{Journal of Computational Physics}, 493, 2023.

\bibitem[Pestourie et~al.(2023)Pestourie, Mroueh, Rackauckas, Das, and Johnson]{Pestourie23p}
Rapha{\"e}l Pestourie, Youssef Mroueh, Chris Rackauckas, Payel Das, and Steven~G. Johnson.
\newblock Physics-enhanced deep surrogates for partial differential equations.
\newblock \emph{Nature Machine Intelligence}, 5[12]:\penalty0 1458--1465, 2023.

\bibitem[Potter et~al.(2010)Potter, Ertin, Parker, and Cetin]{Potter10s}
Lee~C. Potter, Emre Ertin, Jason~T. Parker, and Müjdat Cetin.
\newblock Sparsity and compressed sensing in radar imaging.
\newblock \emph{Proceedings of the IEEE}, 98\penalty0 (6):\penalty0 1006--1020, 2010.

\bibitem[Psichogios \& Ungar(1992)Psichogios and Ungar]{Psichogios92a}
Dimitris~C. Psichogios and Lyle~H. Ungar.
\newblock A hybrid neural network‐first principles approach to process modeling.
\newblock \emph{Aiche Journal}, 38:\penalty0 1499--1511, 1992.

\bibitem[Rahman et~al.(2023)Rahman, Ross, and Azizzadenesheli]{Rahman23u}
Md~Ashiqur Rahman, Zachary~E Ross, and Kamyar Azizzadenesheli.
\newblock U-{NO}: U-shaped neural operators.
\newblock \emph{Transactions on Machine Learning Research}, 2023.

\bibitem[Raissi et~al.(2019)Raissi, Perdikaris, and Karniadakis]{Raissi19p}
M.~Raissi, P.~Perdikaris, and G.E. Karniadakis.
\newblock Physics-informed neural networks: A deep learning framework for solving forward and inverse problems involving nonlinear partial differential equations.
\newblock \emph{Journal of Computational Physics}, 378:\penalty0 686--707, 2019.

\bibitem[Robert \& Casella(2004)Robert and Casella]{Robert04m}
C.P. Robert and G.~Casella.
\newblock \emph{Monte {Carlo} statistical methods}.
\newblock Springer Verlag, 2004.

\bibitem[Robey* et~al.(2021)Robey*, Chamon*, Pappas, Hassani, and Ribeiro]{Robey21a}
A.~Robey*, L.~F.~O. Chamon*, G.~J. Pappas, H.~Hassani, and A.~Ribeiro.
\newblock Adversarial robustness with semi-infinite constrained learning.
\newblock In \emph{Advances in Neural Information Processing}, 2021.

\bibitem[Rockafellar \& Wets(2004)Rockafellar and Wets]{Rockafellar04v}
R.~T. Rockafellar and R.~J-B Wets.
\newblock \emph{Variational Analysis}, volume 317.
\newblock Springer Science \& Business Media, 2004.

\bibitem[Son et~al.(2021)Son, Jang, Han, and Hwang]{Son21s}
Hwijae Son, Jin~Woo Jang, Woo~Jin Han, and Hyung~Ju Hwang.
\newblock Sobolev training for physics informed neural networks.
\newblock \emph{arXiv:2101.08932}, 2021.

\bibitem[Steinbach \& Zank(2020)Steinbach and Zank]{Steinbach20c}
Olaf Steinbach and Marco Zank.
\newblock Coercive space-time finite element methods for initial boundary value problems.
\newblock \emph{Electronic Transactions on Numerical Analysis}, 52:\penalty0 154–194, 2020.

\bibitem[Takamoto et~al.(2022)Takamoto, Praditia, Leiteritz, MacKinlay, Alesiani, Pflüger, and Niepert]{Takamoto22p}
Makoto Takamoto, Timothy Praditia, Raphael Leiteritz, Dan MacKinlay, Francesco Alesiani, Dirk Pflüger, and Mathias Niepert.
\newblock {PDEBench}: {A}n extensive benchmark for scientific machine learning.
\newblock In \emph{Advances in Neural Information Processing Track on Datasets and Benchmarks}, 2022.

\bibitem[Thomee(2013)]{Thomee13g}
V.~Thomee.
\newblock \emph{Galerkin Finite Element Methods for Parabolic Problems}.
\newblock Springer Series in Computational Mathematics. Springer Berlin Heidelberg, 2013.

\bibitem[Tran et~al.(2023)Tran, Mathews, Xie, and Ong]{Tran23f}
Alasdair Tran, Alexander Mathews, Lexing Xie, and Cheng~Soon Ong.
\newblock Factorized {F}ourier neural operators.
\newblock In \emph{International Conference on Learning Representations}, 2023.

\bibitem[Wang et~al.(2022{\natexlab{a}})Wang, Li, He, and Wang]{Wang22i}
Chuwei Wang, Shanda Li, Di~He, and Liwei Wang.
\newblock Is {$L^2$} physics informed loss always suitable for training physics informed neural network?
\newblock In \emph{Advances in Neural Information Processing}, 2022{\natexlab{a}}.

\bibitem[Wang et~al.(2021{\natexlab{a}})Wang, Teng, and Perdikaris]{Wang21u}
Sifan Wang, Yujun Teng, and Paris Perdikaris.
\newblock Understanding and mitigating gradient flow pathologies in physics-informed neural networks.
\newblock \emph{SIAM Journal on Scientific Computing}, 43\penalty0 (5):\penalty0 A3055--A3081, 2021{\natexlab{a}}.

\bibitem[Wang et~al.(2021{\natexlab{b}})Wang, Wang, and Perdikaris]{Wang21o}
Sifan Wang, Hanwen Wang, and Paris Perdikaris.
\newblock On the eigenvector bias of {F}ourier feature networks: {F}rom regression to solving multi-scale {PDEs} with physics-informed neural networks.
\newblock \emph{Computer Methods in Applied Mechanics and Engineering}, 384:\penalty0 113938, 2021{\natexlab{b}}.

\bibitem[Wang et~al.(2022{\natexlab{b}})Wang, Yu, and Perdikaris]{Wang22w}
Sifan Wang, Xinling Yu, and Paris Perdikaris.
\newblock When and why {PINNs} fail to train: {A} neural tangent kernel perspective.
\newblock \emph{Journal of Computational Physics}, 449:\penalty0 110768, 2022{\natexlab{b}}.

\bibitem[Wang et~al.(2024)Wang, Sankaran, and Perdikaris]{Wang24r}
Sifan Wang, Shyam Sankaran, and Paris Perdikaris.
\newblock Respecting causality for training physics-informed neural networks.
\newblock \emph{Computer Methods in Applied Mechanics and Engineering}, 421:\penalty0 116813, 2024.

\bibitem[Wei \& Zhang(2023)Wei and Zhang]{Wei23s}
Min Wei and Xuesong Zhang.
\newblock Super-resolution neural operator.
\newblock In \emph{Proceedings of the IEEE/CVF Conference on Computer Vision and Pattern Recognition (CVPR)}, pp.\  18247--18256, 2023.

\bibitem[Wight \& Zhao(2021)Wight and Zhao]{Wight21s}
Colby~L. Wight and Jia Zhao.
\newblock Solving {Allen-Cahn} and {Cahn-Hilliard} equations using the adaptive physics informed neural networks.
\newblock \emph{Communications in Computational Physics}, 29\penalty0 (3):\penalty0 930--954, 2021.

\bibitem[Wu et~al.(2023)Wu, Zhu, Tan, Kartha, and Lu]{Wu23a}
Chenxi Wu, Min Zhu, Qinyang Tan, Yadhu Kartha, and Lu~Lu.
\newblock A comprehensive study of non-adaptive and residual-based adaptive sampling for physics-informed neural networks.
\newblock \emph{Computer Methods in Applied Mechanics and Engineering}, 403:\penalty0 115671, 2023.

\bibitem[Wu et~al.(2020)Wu, Xia, and Wang]{Wu20a}
Dongxian Wu, Shu-Tao Xia, and Yisen Wang.
\newblock Adversarial weight perturbation helps robust generalization.
\newblock In \emph{Advances in Neural Information Processing}, pp.\  2958--2969, 2020.

\bibitem[Xu et~al.(2023)Xu, Cao, Yuan, and Meschke]{Xu23t}
Chen Xu, Ba~Trung Cao, Yong Yuan, and G{\"u}nther Meschke.
\newblock Transfer learning based physics-informed neural networks for solving inverse problems in engineering structures under different loading scenarios.
\newblock \emph{Computer Methods in Applied Mechanics and Engineering}, 405:\penalty0 115852, 2023.

\bibitem[Yang et~al.(2020)Yang, Kiyavash, and He]{Yang20g}
Junchi Yang, Negar Kiyavash, and Niao He.
\newblock Global convergence and variance reduction for a class of nonconvex-nonconcave minimax problems.
\newblock In \emph{Advances in Neural Information Processing Systems}, pp.\  1153--1165, 2020.

\bibitem[Yang et~al.(2022)Yang, Li, and He]{Yang22n}
Junchi Yang, Xiang Li, and Niao He.
\newblock Nest your adaptive algorithm for parameter-agnostic nonconvex minimax optimization.
\newblock In \emph{Advances in Neural Information Processing Systems}, 2022.

\bibitem[Yu et~al.(2022)Yu, Lu, Meng, and Karniadakis]{Yu22g}
Jeremy Yu, Lu~Lu, Xuhui Meng, and George~Em Karniadakis.
\newblock Gradient-enhanced physics-informed neural networks for forward and inverse {PDE} problems.
\newblock \emph{Computer Methods in Applied Mechanics and Engineering}, 393, 2022.

\end{thebibliography}

\newpage
\appendix

\newpage
\section{Applications of~\eqref{eq:BVP}}
\label{app:pdes}

We showcase here a variety of phenomena that can be described using~\eqref{eq:BVP}. The PDEs detailed in this section are used throughout our experiments to illustrate the performance of SCL.

\subsection{Convection equation}

The one-dimensional convection equation models the transport of a scalar quantity~$u(x,t)$, such as temperature or concentration, along the spatial dimension~$x$. We consider the convection problem
\begin{subequations}\label{eq:convection}
\begin{align}
    \frac{\partial u(x,t)}{\partial t} + \beta \frac{\partial u(x,t)}{\partial x} &= 0,
    	& (x,t) &\in (0, 2\pi) \times (0, 1]
    \\
    u(x,0) &= \sin (x), & x &\in [0, 2\pi]
    \\
    u(0,t) &= u(2\pi,t), & t &\in (0,1]
    	\label{eq:convection_bc}
\end{align}
\end{subequations}
where the coefficient~$\beta$ denotes the \emph{convection speed}. Despite its use of periodic BCs, namely, \eqref{eq:convection_bc}, this problem~(and its variations) can be cast as a straightforward extension of~\eqref{eq:BVP}.

\subsection{Reaction-diffusion equation}

The one-dimensional reaction-diffusion equation describes a variety of phenomena, including chemical reactions, population dynamics, and heat transfer, depending on the form of its reaction term. In this paper, we consider the reaction-diffusion problem with periodic BC and Gaussian IC. Explicitly,
\begin{subequations}\label{eq:reaction}
\begin{align}
    \frac{\partial u(x,t)}{\partial t} - \nu \frac{\partial^2 u(x,t)}{\partial x^2}
    	&= \rho u(x,t) \big(1 - u(x,t) \big),
    	& (x,t) &\in (0, 2\pi) \times (0, 1]
    \\
    u(x,0) &= \exp \bigg(-\frac{1}{2} \Big( \frac{x - \pi}{\pi / 4} \Big)^2 \bigg),
    	& x &\in [0, 2\pi]
    \\
    u(0,t) &= u(2\pi,t), & t &\in (0,1]
\end{align}
\end{subequations}
where~$\nu > 0$ and~$\rho$ are the diffusion and reaction coefficients, respectively.

\subsection{Eikonal equation}

The eikonal equation is encountered in many applications involving wave propagation, e.g., electromagnetism. It also describes the (signed)~distance between any point~$x \in \Omega$ and some fixed boundary~$\del\calS$, hence its usage in vision applications. In this case, we consider the BVP
\begin{subequations}\label{eq:eikonal}
\begin{alignat}{4}
    \norm{\nabla u(x,y)} &= 1, \quad &(x,y) &\in \Omega
    \\
    u(x,y) &= 0, &(x,y) &\in \del\calS
    	\label{eq:eikonal_ic}
\end{alignat}
\end{subequations}
where~$\Omega = (-1,1)^2$ and~$\del\calS$ is a complex shape, in our case, the gears figure from~\citep{Daw23m}. In order to ensure that negative~(positive) distances are assigned to the interior~(exterior) of the shape, we add a structural constraint to the solution which enforces that~$u$ is non-negative on the boundary of~$\Omega$. Explicitly, we enforce that
\begin{equation}\label{eq:eikonal_bc}
	u(x, y) \geq 0 \text{,} \quad (x,y) \in \del\Omega
		\text{.}
\end{equation}
This can be done by adding a structural constraint to~\eqref{eq:scl_problem}, i.e., by replacing~\eqref{eq:scl_problem}(S) with a loss that induces~\eqref{eq:eikonal_bc}, explicitly
\begin{equation}\label{eq:eikonal_structural}
	\sup_{\psi \in \calP^2(\del\Omega)} \E_{(x, y) \sim \psi} \!\Big[
		\big[ -u_\theta(x, y) \big]_{+}
	\Big] \leq \epsilon_\text{s}
		\text{.}
\end{equation}
In our experiments, we find that this constraint is not particularly difficult to enforce and that we can obtain good results by replacing the worst-case~$\psi_\alpha$ by fixed, equispaced points on~$\del\Omega$.

\subsection{Helmholtz equation}

The two-dimensional Helmholtz equation models wave propagation and vibration phenomena in various physical contexts. Here, we consider the problem
\begin{subequations}\label{eq:helmholtz}
\begin{align}
    \nabla^2 u(x, y) + k^2 u(x, y) &= \tau(x, y,\pi),
    	& (x, y) &\in \text{interior}(\calS)
    \\
    u(x, y) &= \sin(\pi a_1 x) \sin(\pi a_2 y),
    	& (x, y) &\in \del\calS
\end{align}
\end{subequations}
where~$\calS = [0,1]^2$; $k > 0$ is the wave number; coefficients $\pi = (a_1,a_2)$ represent the spatial frequencies in the~$x$ and~$y$ directions, respectively; and the forcing function is~$\tau(x,y,\pi) = \left( k^2 - \pi^2 a_1^2 - \pi^2 a_2^2 \right) \left( \sin(\pi a_1 x) \sin (\pi a_2 y) \right)$.

\subsection{Burgers' equation}

The one-dimensional Burgers' equation models the behavior of a scalar field $u(x,t)$ under the combined effects of nonlinear convection and diffusion. It is commonly used in fluid dynamics and traffic flow to describe shock waves and turbulence. In particular, we consider the following BVP with periodic BCs
\begin{subequations}\label{eq:burger}
\begin{align}
    \frac{\partial u(x,t)}{\partial t} + \dfrac{1}{2} \frac{\partial u^2(x,t)}{\partial x}
    	&= \nu \frac{\partial^2 u(x,t)}{\partial x^2},
    	&(x,t) &\in (0, 1) \times (0, 1]
    \\
    u(x,0) &= h_0(x), &x &\in [0, 1]
    \\
    u(0,t) &= u(1,t), &t &\in (0,1]
\end{align}
\end{subequations}
where~$\nu > 0$ is the viscosity coefficient, which governs the strength of the diffusion term. We consider ICs~$h_0$ from the same distribution as~\citet{Li21f}.

\subsection{Diffusion-sorption equation}

The diffusion-sorption equation models the transport of a scalar field~$u(x,t)$ (e.g., a contaminant concentration) in a porous medium. It is commonly used in environmental science and chemical engineering. In terms of this PDE, we consider the following BVP
\begin{subequations}\label{eq:sorption}
\begin{align}
    \frac{\partial u(x,t)}{\partial t} &= \frac{\nu}{R\big(u(x,t)\big)}
    	\frac{\partial^2 u(x,t)}{\partial x^2},
    	&(x,t) &\in (0, 1) \times (0, 500]
    \\
    u(x,0) &= h_0(x), &x &\in [0, 1]
    \\
    u(0,t) &= 1 \text{,} \quad
    	u(1,t) = \nu \frac{\partial u(1,t)}{\partial x}
    	\text{,}
    	&t &\in (0, 500]
\end{align}
\end{subequations}
where~$\nu = 5 \times 10^{-4}$ is the diffusion coefficient and $R$ is the retardation factor defined as
\begin{equation*}
	R(u) = 1 + \frac{(1 - \phi)}{\phi} \rho_s k n_f u^{n_f - 1}
		\text{,}
\end{equation*}
where~$\rho_s = 2880$ is the bulk density, $\phi = 0.29$ is the medium porosity, $k = 3.5 \times 10^{-4}$ is Freundlich’s sorption parameter, and~$n_f = 0.874$ is Freundlich’s exponent. We consider the distribution of ICs~$h_0$ from~\citet{Takamoto22p}.

A notable aspect of~\eqref{eq:sorption} is its nonlinearity due to the dependence on~$u(x,t)$ of the effective diffusion coefficient~$\frac{\nu}{R(u)}$. What is more, it can become singular when~$u = 0$, making this equation particularly challenging to solve.

\subsection{Navier-Stokes equation}

The two-dimensional, incompressible Navier-Stokes equation in vorticity form removes the pressure term and focuses on the dynamics of rotational flow. It is used to describe the local rotation of a fluid, i.e., the vorticity~$\omega(x, t) = \nabla \times u(x,t)$, where~$u$ is the two-dimensional velocity field and~$\nabla \times f$ denotes the \emph{curl} of~$f$. Although vorticity is more challenging to model than velocity, it offers a deeper understanding of the flow dynamics. The velocity field can be recovered from the vorticity using Poisson's equation.

Explicitly, for~$x = [x_1,x_2]^\top$, we consider the BVP
\begin{subequations}\label{eq:navier}
\begin{align}
	\frac{\partial \omega(x,t)}{\partial t} + u(x,t)^\top \nabla \omega (x,t) &= \nu \nabla^2 \omega + \tau(x), \quad &(x,t) &\in (0,1)^2 \times (0, T]
	\\
	\omega(x,0) &= \omega_0, \quad &x &\in [0,1]^2,
	\\
	\nabla \cdot u(x,t) &= 0 , \quad &(x,t) &\in (0,1)^2 \times (0, T]
\end{align}
\end{subequations}
where~$\nu > 0$ is the viscosity coefficient and~$\nabla \cdot f$ denotes the divergence of~$f$. The forcing function is taken to be
\begin{equation*}
    \tau(x) = 0.1 \big( \sin\big( 2\pi (x_1 + x_2) \big) + \cos\big( 2\pi(x_1 + x_2) \big) \big)
\end{equation*}
and the ICs~$\omega_0$ are taken from the same distribution as in~\citet{Li21f}. We consider three settings, namely, $\nu = 10^{-3}$ with~$T=50$, $\nu = 10^{-4}$ with~$T=30$, and~$\nu = 10^{-5}$ and~$T=20$.

\newpage
\section{Weak solutions and robust learning}
\label{app:weak_formulation_to_statistical_problem}

The space~$\calT$ of test functions plays a fundamental role when defining the weak formulations~\eqref{eq:weak_form}. Typically, it is chosen to be a Sobolev space due to its natural compatibility with BVPs and the fact that it leads to less stringent differentiability requirements on~$u$. It therefore overcomes the main issues with the strong formulation~\eqref{eq:BVP}. A Sobolev space consists of functions in some Lebesgue space~$L^p$ whose \emph{weak derivatives} are also in~$L^p$. Recall that~$L^p$ is the space of~$p$-integrable functions. To define a Sobolev space, we therefore need to start by defining a weak derivative.

A locally integrable function~$f$ defined on an open set~$\calZ \subset \setR^d$ is weakly differentiable with respect to~$z_i$ if there exists~$Df$ also locally integrable~(i.e., $f,Df \in L^1_{\text{loc}}(\calZ)$) such that
\begin{equation}\label{eq:weak_derivative}
	\int_{\calZ} Df(z) \xi(z) dz =
		-\int_{\calD} \varphi(z) \dfrac{\del \xi(z)}{\del z_i} dz,
		\quad \text{for all } \xi \in C_c^\infty(\calZ)
		\text{,}
\end{equation}
where~$C_c^\infty(\calZ)$ is the space of infinitely differentiable, compactly supported functions. We say~$Df$ is the weak derivative of $f$. To generalize~\eqref{eq:weak_derivative} to higher-order derivatives, consider the multi-index~$\alpha = (\alpha_1, \ldots, \alpha_{d})$ to be a~$d$-tuple of non-negative integers and let~$\abs{\alpha} = \sum_{i=1}^{d} \alpha_i$. We then define the $\alpha$-weak derivatives of~$f$, denoted $D^\alpha f$, as
\begin{equation}\label{eq:alpha_weak}
	\int_{\calZ} D^\alpha f(z) \xi(z) dz =
		-\int_{\calD} \varphi(z) \frac{
		   	\partial^{\abs{\alpha}} \xi(z)
		}{
		   	\partial z_1^{\alpha_1} \cdots \partial z_d^{\alpha_d}
		} dz,
		\quad \text{for all } \xi \in C_c^\infty(\calZ)
		\text{.}
\end{equation}

We can now define what we mean by Sobolev space~\citep{Evans10p}.

\begin{definition}[Sobolev space]
	For an integer~$k \geq 0$ and~$1 \leq p \leq \infty$, we define the Sobolev space~$W^{k,p}(\calZ) = \{ f \in L^p(\calZ) \mid D^\alpha f \in L^p(\calZ) \text{ for all multi-indices } \alpha \text{ with } \abs{\alpha} \leq k \} $.
\end{definition}

We write~$W^{k,p}$ whenever the set~$\calZ$ is clear from the context. Note that since~$L^p \subset L^1_{\text{loc}}$ for~$p \geq 1$, Sobolev spaces impose more restrictions than weak differentiability. Also, while~$W^{k,p}$ is in general a Banach space, $W^{k,2}$ is a Hilbert space~\citep{Evans10p}.

Having set the groundwork, we can now proceed with the proof of Prop.~\ref{thm:weak_formulation_to_statistical_problem}.

\subsection{Proof of Prop.~\ref{thm:weak_formulation_to_statistical_problem}}

The proof follows by constructing a measure of the deviation from the weak formulation~\eqref{eq:weak_form} and showing that it is dominated by the proposed worst-case statistical loss. This immediately implies Prop.~\eqref{thm:weak_formulation_to_statistical_problem}. Explicitly, note that the weak formulation~\eqref{eq:weak_form} can equivalently be expressed as
\begin{equation}\label{eq:weak_form_eq}
    \left[
    	\int_{\calD} \Big( D [ u ](x,t) - \tau(x, t) \Big) \varphi(x, t) dx dt
    \right]^2 = 0
    	\text{,} \quad \text{for all } \varphi \in W^{k^\prime,2}
\end{equation}
where we omitted the dependence on the coefficients~$\pi$ for conciseness.
Using Jensen's inequality, we can upper bound~\eqref{eq:weak_form_eq} for any~$\varphi$ as in
\begin{equation}\label{eq:weak_form_jensen}
    \left[
       	\int_{\calD} \Big( D [ u ](x,t) - \tau(x, t) \Big) \varphi(x, t) dx dt
    \right]^2
    \leq
    \int_{\calD} \Big( D [ u ](x,t) - \tau(x, t) \Big)^2
   		\varphi^2(x, t) dx dt
    \text{.}
\end{equation}
Since~$\varphi \in W^{k^\prime,2} \subset L^2$, the partition function~$Z_\psi = \int_{\calD} \varphi(x, t)^2 dx dt = \norm{\varphi}_{L^2}^2 < \infty$ is well-defined. We can therefore consider the normalized~$\psi = \varphi^2/Z_\psi$ in~\eqref{eq:weak_form_jensen} to get
\begin{equation}\label{eq:weak_form_normalized}
    \left[
       	\int_{\calD} \Big( D [ u ](x,t) - \tau(x, t) \Big) \varphi(x, t) dx dt
    \right]^2
    \leq
    Z_\psi \int_{\calD} \Big( D [ u ](x,t) - \tau(x, t) \Big)^2 \psi(x, t) dx dt
    \text{.}
\end{equation}
Since~$\psi$ is a non-negative, normalized function, it is the density of a probability measure, i.e., $\psi \in \calP$. The following technical lemma shows that it is in fact in~$\calP^2$, i.e., it is a square-integrable probability density.

\begin{lemma}\label{thm:wasserstein}
	Let~$\varphi \in W^{k^\prime,2}$ and~$\psi \propto \varphi^2$ be a probability distribution. Then, $\psi \in \calP^2$.
\end{lemma}

Before proving Lemma~\ref{thm:wasserstein}, let us conclude the proof. The hypothesis on~$u^\dagger$ implies that
\begin{equation*}
    \sup_{\psi \in \calP^2}\ \int_{\calD} \Big( D [ u^\dagger ](x,t) - \tau(x, t) \Big)^2
    	\psi(x, t) dx dt = 0
\end{equation*}
and since~$Z_\psi$ is bounded, we obtain from~\eqref{eq:weak_form_normalized} that
\begin{equation*}
    \left[
       	\int_{\calD} \Big( D [ u^\dagger ](x,t) - \tau(x, t) \Big) \varphi(x, t) dx dt
    \right]^2 \leq 0
    \text{.}
\end{equation*}
Noticing that this holds for all~$\psi \in W^{k^\prime,2}$, we recover~\eqref{eq:weak_form_eq}, which concludes the proof.\hfill$\blacksquare$

\begin{proof}[Proof of Lemma~\ref{thm:wasserstein}]
To show~$\psi \in \calP^2$, we must show that~$\phi \in L^4$. Indeed,
\begin{equation}\label{eq:psi_p2}
	\int_{\calD} \psi(x, t)^2 dx dt
		= \frac{1}{Z^2} \int_{\calD} \varphi(x, t)^4 dx dt
		= \left( \frac{\norm{\varphi}_{L^4}}{\norm{\varphi}_{L^2}} \right)^4
		\text{.}
\end{equation}
Since~$\varphi \in W^{k^\prime,2} \subset L^2$, suffices it to show that the numerator is finite. To do so, we can use the Gagliardo-Nirenberg interpolation inequality~\citep{Nirenberg59o, Fiorenza21d} to write
\begin{equation}\label{eq:gagliardo}
	\norm{\varphi}_{L^4} \leq C \big(
		\norm{ D^{m} \varphi}_{L^2}^{\alpha} \norm{\varphi}_{L^2}^{1-\alpha}
		+ \norm{\varphi}_{L^2}
	\big)
		\text{,} \quad \text{for } \alpha \in [0,1]
		\text{,}
\end{equation}
as long as (i)~$k^\prime \geq m \in \setN$ and (ii)~$4\alpha m = d+1$. An additional condition must hold in particular cases:
\begin{equation*}
    \text{(iii) if } m - \dfrac{d+1}{2} \text{ is a non-negative integer, then } \alpha < 1
    	\text{.}
\end{equation*}
Note that~$d+1$ is the dimension of the space-time domain~$\calD$. Since~$\varphi \in W^{k^\prime,2}$, the right-hand of~\eqref{eq:gagliardo} is finite. We therefore only need to show that~(i)--(iii) are satisfied under the hypothesis of the theorem. To do so, we consider three cases:

\begin{enumerate}[(a)]
	\item $\bm{d = 0}$: in this case, $d+1$ is odd so that condition~(iii) does not apply. Immediately, \eqref{eq:gagliardo} holds for~$m = 1$ and~$\alpha = 0.25$.

	\item $\bm{d = 1}$: since~(ii) requires~$m \geq 1$, we now have that $m - \frac{d+1}{2}$ is a non-negative integer. Hence, condition~(iii) applies. Nevertheless, we can still take~$m = 1$ and~$\alpha = 0.5$ in~\eqref{eq:gagliardo}.

	\item $\bm{d \geq 2}$: we can now consider all other cases by taking
	\begin{equation*}
	    m = \ceil{\dfrac{d+1}{4}} \quad \text{and} \quad \alpha = \frac{d+1}{4m}
	    	\text{.}
	\end{equation*}
	Indeed, (ii) holds by construction. Additionally, from the hypothesis of the theorem, we have~$d+1 \leq 4k^\prime$, which by the monotonicity of the ceiling operation implies that~(i) holds. It suffices to show that~(iii) never applies. For~$2\leq d \leq 3$, we have~$m = 1$ and~$(d+1)/2 > 1$. For $d > 3$, we have
	\begin{equation}
	    \ceil{\dfrac{d+1}{4}} - \dfrac{d+1}{2} \leq
	    	\left( \dfrac{d+1}{4} + 1 \right) - \dfrac{d+1}{2} < 0
	    	\text{.}
	\end{equation}

\end{enumerate}

Thus, under the hypothesis of the theorem, \eqref{eq:gagliardo} implies that~$\norm{\varphi}_{L^4} < \infty$. From~\eqref{eq:psi_p2}, this in turn implies that~$\psi \in \calP^2$.
\end{proof}

\subsection{Prop.~\ref{thm:robust_loss}}

For the sake of completeness, we provide a short discussion of the preliminary material needed to prove this proposition.

\subsubsection{Preliminaries}\label{app:prelims_robust_loss}

We begin by showing that the supremum can be written as a distributionally robust optimization problem. Indeed, consider the worst-case loss~$\bar{\ell} (\theta) = \sup_{(x,t) \in \calD}\ \ell( u_\theta(x, t) )$ and assume that~$(x,t) \mapsto \ell( u_\theta(x, t) )$ is a function in~$L^2$. This loss can be written in epigraph form as
\begin{prob}\label{eq:epigraph}
	\bar{\ell}(\theta) = \inf_{t \in \setR}\ t \ \subjectto\ \ell( u_\theta(x, t) ) \leq t
		\text{,} \quad \text{for all }(x,t) \in \calD
		\text{.}
\end{prob}
Writing~\eqref{eq:epigraph} in Lagragian form~\cite[Ch. 4]{Bertsekas09c}, we obtain
\begin{prob}\label{P:primal_lagrangian}
	\bar{\ell}(\theta) = \inf_{t \in \setR}\ \sup_{\psi \in L^2_+}\ L_\textup{\ref{eq:epigraph}}(\theta, t, \psi)
	\text{,}
\end{prob}
where~$L^2_+$ denotes the subspace of almost everywhere non-negative functions of~$L^2$. Here, the Lagrangian $L_\textup{\ref{eq:epigraph}}(\theta, t, \psi)$ is defined as
\begin{equation}\label{E:lagrangian_sip}
	\begin{aligned}
		L_\textup{\ref{eq:epigraph}}(\theta, t, \psi) &=
			t + \int_{\calD} \psi(x,t) \big[ \ell( u_\theta(x, t) ) - t\big] dx dt
		\\
		{}&= t  \bigg[ 1 - \int_{\calD} \psi(x,t) dx dt \bigg]
			+ \int_{\calD} \psi(x,t) \ell( u_\theta(x, t) ) dx dt
		\text{,}
	\end{aligned}
\end{equation}
Since~\eqref{E:lagrangian_sip} is a linear function of~$t$, $\bar{\ell}(\theta)$ is the optimal value of a linear program parametrized by~$\theta$. Hence, strong duality holds~\cite[Ch. 4]{Bertsekas09c} and we obtain that
\begin{equation}\label{E:primal_function2}
	\bar{\ell}(\theta) = \sup_{\psi \in L^2_+}\ d_\textup{\ref{eq:epigraph}}(\psi)
	\quad\text{where}\quad
	d_\textup{\ref{eq:epigraph}}(\psi) \triangleq \min_{t \in \setR}\ L_\textup{\ref{eq:epigraph}}(\theta, t, \psi)
		\text{.}
\end{equation}
Since~$t$ is unconstrained and~$L_\textup{\ref{eq:epigraph}}$ is linear in~$t$, the dual function diverges to~$-\infty$ unless~$\int_{\calD} \psi(x,t) dx dt = 1$. From~\eqref{E:lagrangian_sip} and~\eqref{E:primal_function2}, we thus obtain that
\begin{equation}\label{eq:primal_function3}
		\bar{\ell}(\theta) = \sup_{\psi \in \calP^2} \int_{\calD} \psi(x,t) \ell( u_\theta(x, t) ) dx dt
\end{equation}

We also quickly review the necessary variational results for normal integrands. The majority of this exposition is adapted from~\cite{Rockafellar04v}. Throughout, we let the tuple $(T, \calA)$ denote a measurable space, where $T$ is a nonempty set and~$\calA$ is a $\sigma$-algebra of measurable sets belonging to $T$.

\begin{definition}[Carath\'eodory integrand]
	A function $f:T \times \setR^n\to\setR$ is called a \textbf{Carath\'eodory integrand} if it is measurable in $t$ for each $x$ and continuous in $x$ for each $t$.
\end{definition}

\begin{definition}[Decomposable space]~\label{def:decomposable}
	A space $\calF$ of measurable functions $g:T\to\setR^n$ is \textbf{decomposable} in association with a measure $\mu$ on $\calA$ if for every function $g_0\in\calF$, for every set $A\in\calA$ with $\mu(A)<\infty$, and for every bounded, measurable function $g_1:A\to\setR^n$, the space $\calF$ contains the function $g:T\to\setR^n$ defined by
	\begin{align}
		g(t) = \begin{cases}
			g_0(t) &\quad\text{for } t\in T\backslash A, \\
			g_1(t) &\quad\text{for } t\in A.
		\end{cases}
	\end{align}
\end{definition}

\noindent Note that Lebesgue spaces~$L^p$ are decomposable for all $p\in[1,\infty]$~(see, e.g.,~\cite[Ch.~14]{Rockafellar04v}). We can now state a crucial result concerning the interchangability of maximization and integration.

\begin{theorem}[Thm.\ 14.60 in~\cite{Rockafellar04v}] \label{thm:interchange}
	Let $\calF$ be a decomposable space of measurable functions and~$F:T\times\setR^n$ be a Carath\'eodory integrand. Then, as long as~$\int_T f(\tau,\phi(\tau))d\tau \neq 0$ for all~$\phi \in \calF$,
	\begin{align}
		\inf_{\phi\in\calF} \int_T f(\tau,\phi(\tau))d\tau = \int_T \left[ \inf_{x \in \setR^n} f(\tau,x) \right]d\tau.
	\end{align}
	Moreover, as long as this common value is not $-\infty$, one has that
	\begin{align}
		\bar\phi \in \argmin_{\phi\in\calF} \int_T f(\tau,\phi(\tau))d\tau
			\iff \bar\phi(\tau) \in \argmin_{x\in\setR^n} f(\tau,x) \quad\text{for almost every } \tau \in T.
	\end{align}
\end{theorem}

\subsection{Proof of Proposition 3.2}\label{app:final-proof-of-3.2}

\begin{proof}
	Consider a sequence~$\psi_n^\star \in \calP^2$ converging to~$\bar{\ell}(\theta)$ in~$L^2$, i.e., a solution of~\eqref{eq:primal_function3}. In other words, for every~$\delta > 0$, there exists~$N_\delta < \infty$ such that
	\begin{equation}\label{eq:sup_approximation}
		\bar{\ell}_\delta(\theta) \triangleq
			\int_{\calD} \psi^\star_n(x,t) \ell( u_\theta(x, t) ) dx dt \geq \bar{\ell}(\theta) - \delta
			\text{,} \quad \text{for all $n \geq N_\delta$.}
	\end{equation}
	Consider now~$c_\delta = \min_{n \leq N_\delta} \norm{\psi_n^\star}^2_{L^2}$. Since~$\psi_n^\star \in \calP^2 \subset L^2$, $\norm{\psi_n^\star}^2_{L^2}$ is finite for all~$n$. And since~$N_\delta$ is finite, so is~$c_\delta$. We can therefore rewrite~\eqref{eq:sup_approximation} as
	\begin{prob}\label{P:primal_function_mod}
		\bar{\ell}_\delta(\theta) = \sup_{\psi \in L^2_+}&
			&&\int_{\calD} \psi(x,t) \ell( u_\theta(x, t) ) dx dt
		\\
		\subjectto& &&\int_{\calD} \psi(x,t)dx dt= 1, \quad \norm{\psi}^2_{L^2} \leq c_\delta,
	\end{prob}
	where we rewrote~$\calP^2$ as $\{\psi \in L^2_+ \mid \int_{\calD} \psi(x,t)dx dt= 1\}$. Notice that~\eqref{P:primal_function_mod} is a convex quadratic program in~$\psi$. Furthermore, note that a zero-mean normal distribution with variance~$c^2/2$ is strictly feasible for~\eqref{P:primal_function_mod}~(it belongs to~$\calP^2$ and strictly satisfies the $L^2$-norm constraint). Hence, Slater's condition holds and we find that~\eqref{P:primal_function_mod} is strongly dual~\cite[Ch. 4]{Bertsekas09c}. We therefore conclude that for every~$c_\delta > 0$, there exists~$\mu_\delta \in \setR$ and~$0 \leq \gamma_\delta < \infty$ such that
	\begin{align*}
		\bar{\ell}_\delta(\theta) &= \sup_{\psi \in L_+^2}
			\int_{\calD} \psi(x,t) \ell( u_\theta(x, t) ) dx dt
				+ \alpha_\delta \bigg[ \int_{\calD} \psi(x,t)dx dt - 1 \bigg]
				+ \gamma_\delta \big[ \norm{\psi}^2_{L^2} - c_\delta \big]
		\\
		&= \sup_{\psi \in L_+^2}
			\int_{\calD} \big[ \psi(x,t) \ell( u_\theta(x, t) )
				+ \gamma_\delta \psi(x,t)^2
				+ \alpha_\delta \psi(x,t) \big] dx dt
			- \gamma_\delta c_\delta - \alpha_\delta
		\text{.}
	\end{align*}

	To conclude, we use the fact that~$L^2_+$ is decomposable and since that the integrand is Carath\'eodory to exchange the supremum and the integral to obtain that
	\begin{equation*}
		\bar{\ell}_\delta(\theta) = \int_{\calD} \bigg[ \sup_{\psi \in \setR}\ %
			\psi \ell( u_\theta(x, t) ) + \gamma_\delta \psi^2 + \alpha_\delta \psi \bigg] dx dt
		- \gamma_\delta c_\delta - \alpha_\delta
	\end{equation*}
	A straightforward calculation of the inner maximization problem shown above yields that the solution to~\eqref{P:primal_function_mod} is given by
	\begin{equation}\label{eq:optimal_lambda}
		\psi_\alpha(x,t) = \frac{\big[\ell( u_\theta(x, t) ) - \alpha \big]_+}{2\gamma}
		\text{,}
	\end{equation}
	where~$[z]_+ = \max(0,z)$ denotes the projection onto the non-negative orthant and~$\alpha,\gamma$ are chosen so that
	\begin{equation*}
		\int_\calD \psi_\alpha^\star(x,t) dx dt = 1 \quad \text{and} \quad \norm{\psi_\alpha}^2_{L^2} \leq c_\delta
			\text{.}
	\end{equation*}
	Hence, for any~$\delta > 0$ in~\eqref{eq:sup_approximation}, we can find~$\alpha$ such that~\eqref{eq:optimal_lambda} approximates~$\bar{\ell}$.
\end{proof}

\newpage
\section{Sampling with the Metropolis-Hastings algorithm}
\label{app:mh_sampling}

Algorithm~\ref{alg:primal_dual} uses samples from one or more~$\psi_0$ in order to compute the losses in steps~4--7. This is, however, not straightforward unless those distributions are fixed to, e.g., a uniform~(as we typically do for the BCs). That is because we only know~$\psi_0$ up to a normalization factor. We overcome this issue using MCMC techniques, more specifically, the Metropolis-Hastings algorithm~\citep{Robert04m}. In Alg.~\ref{alg:MH_sampling}, we consider the general case of sampling from~$\psi_0 \propto \ell$, where~$\ell$ is a non-negative, scalar-valued loss. We denote by~$z_n$ the desired samples and~$\mathcal{R}$ their support. In Alg.~\ref{alg:primal_dual}, for instance, we would take~$z_n = (\datapar{pde})$, $\calR = \calD \times \Pi$, and
\begin{equation}\label{eq:mh_pde_loss}
	\ell(z_n) = \Big(
		D_{\pi^\text{pde}_n} \left[ u_{\theta_k}(\pi^\text{pde}_n) \right](\data{pde})
			- \tau(\data{pde})
	\Big)^2
\end{equation}
in step~5 and~$z_n = (\data{o${_j}$})$, $\calR = \calD$, and
\begin{equation*}
	\ell(z_n) =  \Big(
		u_{\theta_k}(\pi_j,\tau_j) (\data{o${}_j$}) - u_j^\dagger(\data{o${}_j$})
	\Big)^2
\end{equation*}
in step~7.

Typically, the covariance~$\Sigma$ is taken to be diagonal (independent proposals), e.g., $\sigma^2 I$. The choice of parameter~$\sigma^2$ of the proposal~(step~3) affects the \emph{mixing rate}, i.e., how fast the samples converge to the desired distribution. Smaller values of~$\sigma^2$ will lead to slower mixing chains since the algorithm will not explore the space efficiently. On the other hand, large values will cause the acceptance probability in step~4 to be too small, so that the algorithm will remain stuck~(step~5). Oftentimes, the parameter~$\sigma^2$ is adapted during a burn-in phase to hit a specific acceptance rate, around~$30\%$ as a rule-of-thumb~\citep{Robert04m}. In our experiments, we find a reasonable value for~$\sigma^2$ and keep it fixed throughout training. We also consider a burn-in period by using only the last~$N_0$ samples generated by Alg.~\ref{alg:MH_sampling}. For the single BVP experiments in Sec.~\ref{sec:pinn_experiments} we use~$N_0 = 1000$ and for the parameterized BVPs in Sec.~\ref{sec:param_sol_experiments} we use~$N_0 = 2500$.

Given that we sample only from bounded domains~(i.e., some subset of~$\calD \times \Pi$), the target distribution~$\psi_\alpha$ has finite tails for any $\alpha$, satisfying sufficient conditions for uniform ergodicity~[see, e.g.,~\citep{Jarner00g}]. The law of the samples obtained from Alg.~\ref{alg:MH_sampling} therefore converge~(in Kullback-Leibler divergence) to~$\psi_\alpha$. Additionally, since we prove in Prop.~\ref{thm:weak_formulation_to_statistical_problem} and~\ref{thm:robust_loss} that~$\psi_0$~(and more generally, $\psi_\alpha$) are square-integrable, alternative sampling technique with faster mixing rates can be used. That is the case, for instance, of Langevin Monte Carlo~(LMC)~\citep{Robert04m}. Yet, the LMC algorithm uses first-order information of~$\ell$. For the PDE loss in~\eqref{eq:mh_pde_loss}, this means higher-order space-time derivatives of~$u_\theta$ and thus, additional backward passes. It is not clear that the benefits of faster mixing outweigh the increase in computational complexity, especially given that good results can be obtained using Alg.~\ref{alg:MH_sampling}.

\begin{algorithm}[h]
\caption{Metropolis-Hastings algorithm with Gaussian proposal}
\label{alg:MH_sampling}
\begin{algorithmic}[1]
\State $z_0 \sim \text{Uniform}(\calR)$ \Comment{Sample initial state}
\For{$n = 0, \ldots, N-1$}
    \State $\hat{z} \sim \text{Gaussian}\left(z_{n}, \Sigma \right)$
    	\Comment{Draw proposal}
    \State $\displaystyle
    	p_{n} = \min \left( 1, \frac{\ell(\hat{z})}{\ell(z_{n})} \right) \indicator(\hat{z} \in \calR)$
    	\Comment{Evaluate acceptance probability}
    \State $\displaystyle
    \begin{cases}
    	z_{n+1} = \hat{z} \text{,} & \text{with probability $p_n$}
    	\\
    	z_{n+1} = z_n \text{,} & \text{with probability $1-p_n$}
    \end{cases}$
    	\Comment{Update state}
\EndFor
\State \textbf{return} $\{ z_1, \dots, z_{N} \}$
\end{algorithmic}
\end{algorithm}

\newpage
\section{Generalization Results}
\label{app:generalization_results}

Here, we formalize the generalization guarantees for when solutions of the empirical dual problem are (probably approximately) near-optimal and near-feasible for the statistical primal problem. These were first detailed in~\citet{Chamon20p, Chamon23c}. As done in Sec.~\ref{sec:algorithm}, for simplicity and clarity, we consider~\eqref{eq:scl_problem}(M). Specifically, we are interested when solutions of~\eqref{eq:scl_dual} are (probably approximately) near-optimal and near-feasible for~\eqref{eq:scl_statistical}. Generalization guarantees for the for the complete~\eqref{eq:scl_problem} as well as its parametric extension~\eqref{eq:scl_parametric} follow in the same way.

We begin with the essential (non-convex)~duality assumptions. In particular, we assume that the hypothesis space \( \mathcal{F}_\theta = \{ u_\theta : \theta \in \Theta \} \) is sufficiently expressive~(Assumption~\ref{ass:rich_parametrization}) and that there exists a function \( u_\theta \in \mathcal{F}_\theta \) that is strictly feasible~(Assumption~\ref{ass:strict_feas}). For NNs in particular, universal approximation theorems indicate that these assumptions are satisfied for large enough models~(see, e.g., \citep{Hornik91a}). Finally, we impose a learning theoretic limit on the complexity of the hypothesis space~$\mathcal{F}_\theta$ in order to ensure that our empirical approximations are well-posed~(Assumption~\ref{ass:uniform_convergence}). The main theorem our results are based on, namely~\citep[Thm.~1]{Chamon23c}, also require the losses to be convex, $M$-Lipschitz continuous, and $[0, B]$ bounded. Since we only consider quadratic losses on the bounded domain~$\calD$, these assumptions hold immediately.

\begin{assumption}\label{ass:rich_parametrization}
The parametrization \(u_\theta\) is rich enough that for each \(\theta_1, \theta_2 \in \Theta\) and \(\beta \in [0,1]\), there exists \(\theta \in \Theta\) such that \(\sup_{(x,t) \in \calD} \left| \beta u_{\theta_1}(x,t) + (1-\beta) u_{\theta_2}(x,t) - u_\theta(x,t) \right| \leq \nu\).
\end{assumption}

\begin{assumption}\label{ass:strict_feas}
There exist \(\theta'\) such that \(u_{\theta'}\) is strictly feasible for~\ref{eq:scl_statistical}, i.e., such that
\begin{equation*}
	\E_{(x, t) \sim \psi_\alpha^\text{pde}} \!\Big[ \big( D [ u_\theta ](x,t) - \tau(x, t) \big)^2 \Big] \leq \epsilon - M\nu.
\end{equation*}
\end{assumption}

\begin{assumption}\label{ass:uniform_convergence}
There exist $\zeta(N, \delta)$ monotonically decreasing with \(N\) such that with probability $1 - \delta$ over samples~$(\data{bc}) \sim \psi_\alpha^\text{bc}$ and~$(\data{pde}) \sim \psi_\alpha^\text{pde}$, it holds for all~$\theta \in \Theta$ that
\begin{align*}
	\left| \E_{(x, t) \sim \psi_\alpha^\text{bc}} \!\Big[\big( u_\theta(x,t) - h(x, t) \big)^2
		- \frac{1}{N} \sum_{n=1}^{N} \big( u_{\theta}(\data{bc}) - h(\data{bc}) \big)^2 \right|
		&\leq \zeta(N, \delta)
	\\
	\left| \E_{(x, t) \sim \psi_\alpha^\text{pde}} \!\Big[ \big( D [ u_\theta ](x,t) - \tau(x, t) \big)^2 \Big]
		- \frac{1}{N} \sum_{n=1}^{N} \big( D [ u_\theta ](\data{pde}) - \tau(\data{pde}) \big)^2 \right|
		&\leq \zeta(N, \delta).
\end{align*}
\end{assumption}

Under these assumptions, we can bound the empirical duality gap between~\eqref{eq:scl_statistical} and~\eqref{eq:scl_dual}, i.e., $\Delta = \abs{P^\star - D^\star}$, where

\begin{prob*}
	P^\star = \minimize_{\theta \in \Theta}&
	&&\E_{(x, t) \sim \psi_\alpha^\text{bc}} \!\Big[
	\big( u_\theta(x,t) - h(x, t) \big)^2
	\Big]
	\\
	\subjectto& &&\E_{(x, t) \sim \psi_\alpha^\text{pde}} \!\Big[
	\big( D [ u_\theta ](x,t) - \tau(x, t) \big)^2
	\Big] \leq \epsilon
\end{prob*}
and
\begin{prob*}
	D^\star = \max_{\lambda \geq 0}\ \min_{\theta \in \Theta}\ \hat{L}(\theta, \lambda)
	\text{.}
\end{prob*}

\begin{proposition}\label{thm:generalization}
	Under Assumptions~\ref{ass:rich_parametrization}--\ref{ass:uniform_convergence}, it holds with probability $1 - (3m + 2)\delta$ that
	\begin{equation*}
		\Delta \leq O\big( \lambda^\star (M \nu + \zeta) \big)
			\text{,}
	\end{equation*}
	where~$\lambda^\star$ is a solution of~\eqref{eq:scl_dual}.
\end{proposition}

Prop.~\ref{thm:generalization} is obtained directly from~\citep[Thm.~1]{Chamon23c}. This duality gap bound is enough to guarantee that the dual ascent algorithm in~\eqref{eq:dual_ascent} provides a near-optimal and near-feasible randomized solution of~\eqref{eq:scl_statistical}. Since all our losses are strongly convex~(quadratic), we can further show that randomization is not necessary using the last iterate guarantees from~\citep[Prop. 4.1]{Elenter24n}. This is in spite of the fact that~\eqref{eq:scl_statistical} is a non-convex optimization problem.

On the other hand, the convergence of primal-dual methods such as Alg.~\ref{alg:primal_dual} in non-convex settings is the subject of active research, see, e.g., \citep{Yang20g, Lin20n, Fiez21g, Boroun23a}. Transferring the guarantees from~\eqref{eq:dual_ascent} to Alg.~\ref{alg:primal_dual} requires additional conditions, e.g., step size separation as in~\citep{Yang20g}. Such convergence guarantees are, however, beyond the scope of this paper and left for future work.

\newpage
\section{Experimental details}
\label{app:experimental_details}

\subsection{Hyperparameters and implementation details}

Throughout our experiments, we use the relative $L_2$ error as a performance metric, which we define as
\begin{equation}\label{eq:rel_l2_error}
	e_\text{rel}(\pi,h) =\sqrt{
		\dfrac{
			\sum_{n=1}^N \big[ u_\theta(\pi,h)(x_n,t_n) - u^\dagger(\pi,h)(x_n,t_n) \big]^2
		}{
			\sum_{n=1}^N \big[ u^\dagger(\pi,h)(x_n,t_n) \big]^2
		}}
		\text{,}
\end{equation}
where~$u^\dagger$ is the solution of~\eqref{eq:BVP} obtained either analytically or by using classical numerical methods. For MLPs, the collocation points~$\{(x_n,t_n)\}$ are taken from a dense regular grid of points~(see exact numbers below), and for FNOs, they are determined by the test sets from~\citep{Li21f, Takamoto22p}. For parametrized problems, we report the average error
\begin{equation*}
	\bar{e}_\text{rel} = \dfrac{1}{J} \sum_{j=1}^J e_\text{rel}(\pi_j,h_j)
		\text{,}
\end{equation*}
evaluated either on a dense regular grid of points~(for coefficients~$\pi$, see exact numbers below) or based on the test sets from~\citep{Li21f, Takamoto22p}.

To provide sensitivity measures, we run all experiments for 10 different seeds and report average and standard deviations of the results. We find that for certain difficult problem~(e.g., diffusion with $\beta=50$ or reaction-diffusion with~$(\nu,\rho) = (3,5)$) the hyperparameters of SCL and R3 sometimes need to be adjusted for certain seeds. This occurs rarely, but shows that there may not be one-size-fits-all hyperparameter settings. For PINNs in~\eqref{eq:PINN}, we were unable to find any hyperparameters that solved those problems.

\subsubsection{Solving a specific BVP~(Sec.~\ref{sec:pinn_experiments})}

In this section, we formulated SCL problems of the form~\eqref{eq:scl_problem}(M) in order to use the same information that PINNs traditionally rely on. Recall that we do use the worst-case distribution~$\psi_0$~(or even random points) for the BCs, but instead consider fixed, regularly distributed points. This reduces the overall computational complexity of the problem at essentially no performance cost. Explicitly, we consider the following problem
\begin{prob*}
	\minimize_{\theta \in \Theta}&& &\frac{1}{N} \sum_{n=1}^{N}
		\Big( u_\theta(\data{bc}) - h(\data{bc}) \Big)^2
	\\
	\subjectto&& &\E_{(x, t) \sim \psi_0^\text{PDE}} \!\Big[
		\big( D [ u_\theta ](x,t) - \tau(x, t) \big)^2
	\Big] \leq \epsilon_\text{pde}\text{.}
\end{prob*}
To compute the objective for the convection and reaction-diffusion PDEs, we use~$256$ points~$(x^\text{bc},0)$, $x^\text{bc} \in [0,2\pi]$, for the IC and~$100$ points equally spaced in~$t \in (0,1]$ to evaluate the period BC. For the eikonal PDE, recall from~\eqref{eq:eikonal_structural} we use an additional structural constraint. In this case, we therefore formulate the SCL problem
\begin{prob*}
	\minimize_{\theta \in \Theta}&& &\frac{1}{M} \sum_{m=1}^{M}
		\big[ u_\theta(x_m,y_m) \big]^2
	\\
	\subjectto&& &\E_{(x, t) \sim \psi_0^\text{PDE}} \!\Big[
		\big( D [ u_\theta ](x,t) - \tau(x, t) \big)^2
	\Big] \leq \epsilon_\text{pde}
	\\
	&&&\frac{1}{N} \sum_{n=1}^{N} \big[ -u_\theta(x_n,y_n) \big]_+
		\leq \epsilon_\text{s}
	\text{,}
\end{prob*}
where we use fixed collocation points for the BCs and structural constraint, namely, $M = 2234$ points on~$\del\calS$~(the gears figure from~\citep{Daw23m}) and~$N=40$ points on~$\del\Omega$. We use the exact same points for~\eqref{eq:PINN}.

\paragraph{Problem hyperparameters.}
For SCL, the tolerance~$\epsilon_\text{pde}$ was selected by starting with a small value~(e.g., $10^{-4}$) and increasing it when the dual variables became too large during training to accommodate difficult problems. After a coarse hyperparameter search, we kept the weights~$\mu$ in~\eqref{eq:PINN} used in~\citep{Daw23m}. Note that we used different weights for the BC and IC to solve the eikonal~PDE, since in this case the BCs play a less critical role. All values are displayed in Table~\ref{tab:hyper_prob_pinn}. When solving the Eikonal equation with SCL, we used $\epsilon_\text{s} = 10^{-3}$ as the tolerance for the structural constraint. 

\begin{table}[b]
	\centering
	\renewcommand{\arraystretch}{1.2}
	\caption{Problem hyperparameters for solving a specific BVP}
	\label{tab:hyper_prob_pinn}
	\begin{tabular}{l|ccc|c}
		& $\mu_{\text{D}}$ & $\mu_{\text{BC}}$ & $\mu_{\text{IC}}$ & $\epsilon_\text{pde}$
		\\\hline
		Convection: $\beta = 30$                  & 1 & 100 & 100 & $10^{-3}$ \\
		Convection: $\beta = 50$                  & 1 & 100 & 100 & $5 \times 10^{-3}$ \\
		Reaction-diffusion: $(\nu, \rho) = (3,3)$ & 1 & 100 & 100 & $10^{-2}$ \\
		Reaction-diffusion: $(\nu, \rho) = (3,5)$ & 1 & 100 & 100 & $5 \times 10^{-3}$ \\
		Eikonal                                   & 1 & 10  & 500 & $5 \times 10^{-1}$ \\
	\end{tabular}
\end{table}

\paragraph{Model.}
We used MLPs with 4 hidden layers for~$u_\theta$ each with 50 neurons for the convection and reaction-diffusion equations or 128 neurons for the eikonal equation and hyperbolic tangent activation function.

\paragraph{Training.}
To evaluate the PDE loss, all methods used~$1000$ collocation points sampled uniformly at random at the beginning of each epoch~(PINN), obtained using the R3 from~\citep{Daw23m}~(R3), or using Alg.~\ref{alg:MH_sampling}~(SCL). For R3, we use the hyperparameters from~\citep{Daw23m}. For Alg.~\ref{alg:MH_sampling}, we use~$\Sigma = \diag(0.25,0.01)$ for drawing proposals for~$x$ and~$t$ respectively for both convection and reaction-diffusion. For the eikonal PDE, we use~$\Sigma = 0.04 \times I$. In both cases, we run the algorithm for~$N = 5000$ and use only the last~$1000$ samples. All methods were trained using Adam with the default parameters from~\citep{Kingma17a} and learning rates described in Table~\ref{tab:hyper_opt_pinn}. Note that the baselines only use learning rate~$\eta_p$, since they do not use dual methods.

\begin{table}[t]
	\centering
	\renewcommand{\arraystretch}{1.2}
	\caption{Training hyperparameters for solving a specific BVP}
	\label{tab:hyper_opt_pinn}
	\begin{tabular}{lcccc}
		& $\eta_p$ & \parbox{1.6cm}{\centering$\eta_d$\linebreak(only SCL)} & Learning rate decay & Iterations
		\\\hline
		Convection: $\beta = 30$                  & $10^{-3}$ & $10^{-4}$ & $0.9 \eta$ every $5\,000$ iter. & $175\,000$ \\
		Convection: $\beta = 50$                  & $10^{-3}$ & $10^{-4}$ & $0.9 \eta$ every $5\,000$ iter. & $200\,000$ \\
		Reaction-diffusion: $(\nu, \rho) = (3,3)$ & $10^{-3}$ & $10^{-4}$ & ---                             & $200\,000$ \\
		Reaction-diffusion: $(\nu, \rho) = (3,5)$ & $10^{-3}$ & $10^{-4}$ & ---                             & $200\,000$ \\
		Eikonal                                   & $10^{-3}$ & $10^{-4}$ & $0.9 \eta$ every $5\,000$ iter. & $60\,000$ \\
	\end{tabular}
\end{table}

\paragraph{Testing.}
The solution of the convection and reaction-diffusion PDEs were tested on a dense regular grid of~$256 \times 100$ points~$(x,t) \in \calD$ against their analytical solutions. The solution of the eikonal PDE was tested on a dense regular grid of~$384 \times 384$ points~$(x,y) \in \Omega$ against the ground truth predictions from~\citep{Daw23m}.

\subsubsection{Solving parametric families of BVPs~(Sec.~\ref{sec:param_sol_experiments})}

The SCL problem we formulate here is similar to the previous section, although we used the parameterized version~\eqref{eq:scl_parametric}(M). Once again, we replace the $(x,t)$ marginals of the worst-case distribution~$\psi_0^\text{BC}$ by a fixed, uniform distribution. Note, however, that we keep the worst-case formulation for the coefficients~$\pi$. Explicitly, we consider the SCL problem
\begin{prob*}
	\minimize_{\theta \in \Theta}& &&\E_{\pi \sim \psi_0^\text{BC}} \!\bigg[
		\frac{1}{N_\text{bc}} \sum_{n=1}^{N_\text{bc}}
		\Big( u_\theta(\pi)(\data{bc}) - h(\pi)(\data{bc}) \Big)^2
	\bigg]
	\\
	\subjectto& &&\E_{(x, t, \pi) \sim \psi_0^\text{PDE}} \!\Big[
		\big( D_\pi [ u_\theta(\pi) ](x,t) - \tau(\pi)(x, t) \big)^2
	\Big] \leq \epsilon_\text{pde}
\end{prob*}
Once again, we compute the objective for the convection and reaction-diffusion PDEs using~$256$ points~$(x,0)$, $x \in [0,2\pi]$, for the IC and~$100$ points equally spaced in~$t \in (0,1]$ to evaluate the period BC. For the Helmholtz PDE, we use~$4\times256$ points equally space around~$\del\Omega$ to evaluate the BC. We use the exact same points for~\eqref{eq:PINN}. Note that we include the coefficients~$\pi$ in the forcing function to account for the Helmholtz BVP~(see Sec.~\ref{app:pdes}).

\paragraph{Problem hyperparameters.}
Once again, the tolerance~$\epsilon_\text{pde}$ were selected by starting with a small value~(e.g., $10^{-4}$) and increasing when the dual variables achieved too large a value during training to accommodate difficult problems. The weights~$\mu$ in~\eqref{eq:PINN} for the baselines were taken from~\citep{Daw23m}. Exact values are displayed in Table~\ref{tab:hyper_prob_param}.

\begin{table}[b]
	\centering
	\renewcommand{\arraystretch}{1.2}
	\caption{Problem hyperparameters for solving a parametric family of BVPs}
	\label{tab:hyper_prob_param}
	\begin{tabular}{l|cc|c}
		& $\mu_{\text{D}}$ & $\mu_{\text{BC}}$ & $\epsilon_\text{pde}$
		\\\hline
		Convection                                     & 1 & 100 & $10^{-3}$ \\
		Reaction-diffusion: $(\nu,\rho) \in [0,5]^2$   & 1 & 100 & $5 \times 10^{-3}$ \\
		Reaction-diffusion: $(\nu,\rho) \in [0,10]^2$  & 1 & 100 & $10^{-2}$ \\
		Reaction-diffusion: $(\nu,\rho) \in [0,20]^2$  & 1 & 100 & $10^{-1}$ \\
		Helmholtz: $(a_1, a_2) \in [1,2]^2$            & 1 & 100 & $5 \times 10^{-1}$ \\
		Helmholtz: $(a_1, a_2) \in [1,3]^2$            & 1 & 100 & $5$ \\
	\end{tabular}
\end{table}

\paragraph{Model.}
We used MLPs with 4 hidden layers of 50 neurons.

\paragraph{Training.}
When training using~\eqref{eq:PINN}, we used~$1000$ collocation points per coefficient value, sampled uniformly at random at the beginning of each epoch. For \eqref{eq:scl_parametric}(M), we used Alg.~\ref{alg:MH_sampling}. For the PDE loss, we used~$\Sigma = \diag(0.25, 0.01, \sigma_\pi^2)$ to sample from~$(x,t,\pi)$ for both the convection~($\sigma_pi^2 = 9$ for coefficient~$\beta$) and reaction-diffusion~[$\sigma_pi^2 = (1,1)$ for coefficients~$(\nu,\rho)$] equations. For the Helmholtz PDE, we used~$\Sigma = 0.04 \times I$ to sample from~$(x,t,\pi)$ for coefficients~$\pi = (a_1,a_2)$. The same variances~$\sigma_\pi^2$ were used to sample worst-case coefficients for the BC~(recall that the distribution over collocation points is fixed). In all cases, SCL uses the last~$2500$ out of~$5000$ samples generated by the MH algorithm, except for the Helmholtz PDE with~$(a_1, a_2) \in [1,3]^2$, where we use all $5,000$ samples to account for the additional difficulty of the problem.

All models were trained for~$200,000$ using Adam with the default parameters from~\citep{Kingma17a} and learning rate of $10^{-3}$ for~\eqref{eq:PINN}. For SCL, the dual learning rate was~$\eta_d = 10^{-4}$. In all cases, we decayed the learning rates by a factor of~$0.9$ every 5000 epochs, with the exception of the reaction-diffusion PDE where we found it better to keep the learning rate constant.

\paragraph{Testing.}
The solution of the convection and reaction-diffusion PDEs were tested on a dense regular grid of~$256 \times 100 \times 1000$ points~$(x,t,\pi) \in \calD \times \Pi$ and~$256 \times 100 \times 100 \times 100$ points~$(x,t,\nu,\rho) \in \calD \times \Pi$, respectively, against their analytical solutions. The solution of the Helmholtz PDE was tested on a dense regular grid of~$256 \times 256 \times 100 \times 100$ points~$(x,y,a_1,a_2) \in \Omega \times \Pi$ against its analytical solution.

\subsubsection{Leveraging invariance when solving BVPs~(Sec.~\ref{sec:invariance_constraint_experiments})}

The SCL problems formulated in this section are of the form~\eqref{eq:scl_problem}(M+I). To showcase the advantages of integrating additional knowledge, such as the structure of the BVP solution, we consider fixed collocation points for the constraints~\eqref{eq:scl_problem}(M). This is in fact not uncommon for PINNs, see, e.g., \citep{Raissi19p, Lu21p}. These points are sampled uniformly at random once and then kept constant throughout training. Recall that for our convection BVP~(Sec.~\ref{app:pdes}), the solution is periodic with period~$2\pi/\beta$. We therefore use the problem
\begin{prob*}
	\minimize_{\theta \in \Theta}& &&\frac{1}{N} \sum_{n=1}^{N} \Big( u_\theta(x_n,t_n) - h(x_n, t_n) \Big)^2
	\\
	\subjectto& &&\frac{1}{M} \sum_{m=1}^{M} \Big( D [ u_{\theta} ](x_m,t_m) - \tau (x_m, t_m)\Big)^2  \leq \epsilon_\text{pde}
	\\
	&&&\E_{(x, t) \sim \psi_0^\text{ST}} \!\bigg[
		\Big( u_\theta(x, t) - u_\theta \Big[x, t + \frac{2 \pi}{\beta} \Big] \Big)^2
	\bigg] \leq \epsilon_\text{s}
\end{prob*}
For both~\eqref{eq:PINN} and~\eqref{eq:scl_parametric}(M) we use a total of~$N = 456$ collocation points, namely~$256$ points~$(x,0)$, $x \in [0,2\pi]$, for the IC and~$100$ points equally spaced in~$t \in (0,1]$ to evaluate the period BC. We use~$M = 100$ collocation points sampled uniformly at random in the beginning of training and kept fixed throughout for the PDE loss.

\paragraph{Problem hyperparameters.}
For SCL, we take~$\epsilon_\text{pde} = 10^{-3}$ and~$\epsilon_{s} = 10^{-3}$. For~\eqref{eq:PINN}, we use the weights~$\mu$ from~\citep{Daw23m}, namely, $\mu_{\text{D}} = 1$, $\mu_{\text{BC}} = 100$, and~$\mu_{\text{IC}} = 100$.

\paragraph{Model.}
We used MLPs with 4 hidden layers of 50 neurons.

\paragraph{Training.}
For \eqref{eq:scl_parametric}(I), we used Alg.~\ref{alg:MH_sampling}. For the invariance loss, we used~$\Sigma = \diag(0.5, 0.1)$ to sample from~$(x,t)$. All models were trained for~$200\,000$ epochs using Adam with the default parameters from~\citep{Kingma17a} and learning rate of $10^{-3}$ for~\eqref{eq:PINN}. For SCL, the dual learning rate was~$\eta_d = 10^{-4}$. We decayed the learning rates by a factor of~$0.9$ every 5000 epochs.

\paragraph{Testing.}
The solution was tested on a dense regular grid of~$256 \times 100$ points~$(x,t) \in \calD$ against its analytical solution.

\subsubsection{Supervised solution of BVPs~(Sec.~\ref{sec:pointwise_constrained_fno_experiments})}

For supervised experiments, we formulate an SCL without objective using only data constraints~(observational knowledge). Since we use FNOs, that can only make predictions on uniform grids, we replace~$\psi_0^{OB}$ in~\eqref{eq:scl_problem} with a uniform distribution over a fixed regular grid. The problem the FNOs tackle is that of predicting the solution~$u^\dagger$ of a BVP given its IC~$h(x,0)$. Hence, the training data is composed of pairs~$(u_j^\dagger,h_j)$ describing ICs and their corresponding solution. We therefore pose the SCL problem
\begin{prob*}
	\minimize_{\theta \in \Theta}& &&0
	\\
	\subjectto& &&\frac{1}{N} \sum_{n=1}^{N} \left( u_\theta(h_j)(x_n,t_n) - u_j^\dagger(x_n, t_n) \right)^2 \leq \epsilon_{\text{o}}, \quad j = 1, \dots, J.
\end{prob*}

\paragraph{Problem hyperparameters.}
For SCL, the tolerance was chosen as before, using a coarse hyperparameter search. The final values are reported in Table~\ref{tab:hyper_fno}.

\begin{table}[t]
	\centering
	\renewcommand{\arraystretch}{1.2}
	\caption{Problem hyperparameters for supervised solutions}
	\label{tab:hyper_fno}
    \resizebox{\textwidth}{!}{
	\begin{tabular}{l|ccccc}
		& $\epsilon_\text{o}$ & \makecell{\# training \\ samples} & \makecell{\# validation \\ samples} & \makecell{\# test \\ samples} & FNO architecture
		\\\hline
		Burgers'                       & $10^{-3}$          & 800 & 200 & 200 & 16 modes, 4 layers \\
		Diffusion-sorption             & $10^{-3}$          & 1000 & 500 & 500 & 8 modes,  5 layers \\
		Navier-Stokes: $\nu = 10^{-3}$ & $10^{-2}$          & 1000 & 500 & 500 & 8 modes,  8 layers \\
		Navier-Stokes: $\nu = 10^{-4}$ & $5 \times 10^{-2}$ & 1000 & 500 & 500 & 8 modes,  8 layers \\
		Navier-Stokes: $\nu = 10^{-5}$ & $10^{-2}$          & 800 & 200 & 200 & 8 modes,  8 layers \\
	\end{tabular}
    }
\end{table}

\paragraph{Model.}
We used the FNO architecture from~\citep{Li21f} with 64 hidden channels, 128 projection channels, and no lifting channels. The number of modes and layers are reported in Table~\ref{tab:hyper_fno}.

\paragraph{Training and Testing.}
The datasets from~\citep{Li21f} were used for Burgers' and Navier-Stokes equation, whereas the diffusion-sorption dataset was taken from~\citep{Takamoto22p}. All models were trained for~$500$ epochs using Adam with the default settings from~\citep{Kingma17a} with learning rate~$10^{-3}$ and batch size of~$20$. For SCL, the dual learning rate was~$\eta_d = 10^{-4}$. All learning rates were decreased by a factor of~$0.5$ every~$100$ epochs. All test errors are reported for the model that achieved the lowest validation error during training. The sizes of the training, validation and test sets are reported in Table~\ref{tab:hyper_fno}.

\newpage
\section{Additional experiments}
\label{app:additional_experiments}

\renewcommand{\arraystretch}{1.25}

\subsection{Solving parametric families of BVPs}

We begin by presenting additional experiments focused on solving parametric families of BVPs~(Sec.~\ref{sec:param_sol_experiments}) and show how the samples from MH can be used to gain insights into the PDE and the training process.

In what follows, we report the ``relative (computational) complexity'' of~\eqref{eq:scl_parametric} in terms of differential operator evaluations per epoch. Explicitly,
\begin{equation*}
	\text{Relative complexity} = \frac{
		\text{\# differential operator evaluations per epoch for \eqref{eq:scl_parametric}}
	}{
		\text{\# differential operator evaluations per epoch for \eqref{eq:PINN}}
	} \times 100 \%
\end{equation*}

Recall that in order to evaluate the PDE loss, \eqref{eq:PINN} uses 1000 collocation points per discretized coefficient~$\pi_j$ whereas~\eqref{eq:scl_parametric} takes $5000$ steps of Alg.~\ref{alg:MH_sampling}.

\paragraph{Convection equation.}
Table~\ref{tab:convection_param_sol_comparisson} considers simultaneously solving all BVPs corresponding to the convection equation with $\beta \in [1,30]$ and compares~\eqref{eq:scl_parametric}(M) with~\eqref{eq:PINN}. We see that~\eqref{eq:scl_parametric}(M) outperforms or matches~\eqref{eq:PINN} while being more efficient. In particular,~\eqref{eq:scl_parametric}(M) significantly outperforms~\eqref{eq:PINN} in terms of relative $L_2$ error for all but the finest discretization where they perform similarly. However, for that discretization,~\eqref{eq:scl_parametric}(M) is much more efficient that~\eqref{eq:PINN}. In that sense, it strikes a better compromise between error and computational cost. This is even clearer from Fig.~\ref{fig:error_vs_fp_convection}, particularly when we normalize the $x$-axis in terms of differential operator evaluations.

\begin{table}[tbh]
\centering
\caption{Relative $L_2$ error and computational efficiency for the parametric convection problem.}
	\label{tab:convection_param_sol_comparisson}
\begin{tabular}{|c|c|c|c|}
\hline
	\multirow{2}{*}{Discretization for~\eqref{eq:PINN}} &
	\multicolumn{2}{c|}{Average Relative $L_2$ Error} &
	\multirow{2}{*}{\makecell{Relative complexity\\\eqref{eq:scl_parametric} $\div$ \eqref{eq:PINN}}}
\\\cline{2-3}
& \eqref{eq:PINN} & \eqref{eq:scl_parametric}(M) &
\\\hline
$\{1.0, 10.0, 20.0, 30.0\}$                  & 0.365 & \multirow{3}{*}{$0.0110$} & 125\%
\\
$\{1.0, 5.0, 10.0, 15.0, 20.0, 25.0, 30.0\}$ & 0.220 &                         & 71\%
\\
$\{1.0, 2.0, 3.0, 4.0, 5.0, \ldots, 30\}$    & $0.0476$  &                        & 16\%
\\\hline
\end{tabular}
\end{table}

\begin{figure}[tbh]
    \centering
    \includegraphics[width=\linewidth]{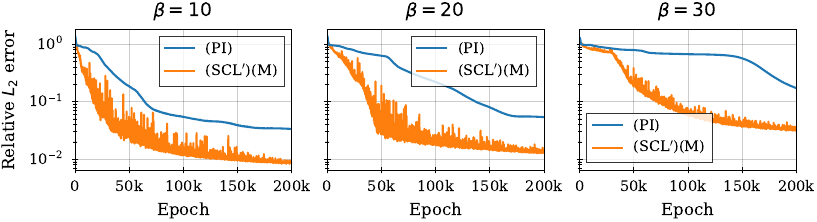}

    \includegraphics[width=\linewidth]{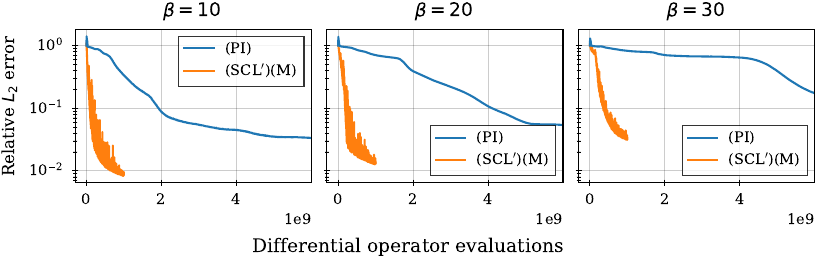}
    \caption{Relative $L_2$ error as a function of training epoch and differential operator evaluations for the parametric convection problem.}
    \label{fig:error_vs_fp_convection}
\end{figure}

\clearpage

\paragraph{Reaction-diffusion equation.}

Additional results for the parametric reaction-diffusion BVP are shown in Table~\ref{tab:rd_param_sol_comparisson}. Same as for the convection PDE, \eqref{eq:scl_parametric}(M) is more efficient than~\eqref{eq:PINN}, striking a better compromise between computational complexity and performance. Indeed, in order to achieve the same error as~\eqref{eq:scl_parametric}(M), \eqref{eq:PINN} requires between~5 and~7 times more evaluations of the PDE loss, i.e., of the differential operator~$D$, per epoch. This is once again clear when looking at the evolution of the error during training~(Fig.~\ref{fig:error_epoch_rd}), especially when the $x$-axis is displayed is terms of PDE evaluations. The distribution of errors across parameters is also more homogeneous for the SCL solution~(Fig.~\ref{fig:param_sol_reaction_diffusion}).

Finally, we can once again inspect the samples from~$\psi_0$ throughout training to understand where the advantage of SCL comes from~(Fig.~\ref{fig:samples_reaction_diffusion}). First, we do not note any interesting behavior over the $x$-marginal~(the samples are mostly uniform and the histogram is therefore omitted). Once again, we see that~\eqref{eq:scl_parametric}(M) starts by focusing more on earlier times~$t$, fitting the solution of the PDE ``causally.''
Additionally, since the diffusion term tends to make the solution more homogeneous for larger times, it is clear that these are regions that are easier to fit and therefore require less attention. Once again, this behavior is not manually encouraged, but arises naturally from Alg.~\ref{alg:primal_dual}. As for the~$(\nu,\rho)$, we see that the distributions shift during training, indicating the change in difficulty of fitting the solution of the reaction-diffusion PDE. In the end, the samples for~$\rho$ are quite uniform, while we notice that there remains a strong focus on smaller values of~$\nu$. Note that these distributions reflect the error patterns of the final solution~(Fig.~\ref{fig:param_sol_reaction_diffusion}).

\begin{table}[tbh]
\centering
\caption{Relative $L_2$ error and computational efficiency for the parametric reaction-diffusion problem.}
\label{tab:rd_param_sol_comparisson}
\begin{tabular}{|c|c|c|c|c|}
	\hline
	\multirow{2}{*}{\makecell[c]{Coefficients\\range}} &
	\multirow{2}{*}{\makecell{Discretization for~\eqref{eq:PINN}}} &
	\multicolumn{2}{c|}{Average relative $L_2$ error} &
	\multirow{2}{*}{\makecell{Relative complexity\\\eqref{eq:scl_parametric} $\div$ \eqref{eq:PINN}}}
	\\ \cline{3-4}
	&  & \eqref{eq:PINN} & \eqref{eq:scl_parametric}(M) &
	\\ \hline
	\multirow{4}{*}{\makecell{$\nu \in [0,5]$ \\ $\rho \in [0,5]$}} &
	$\{0.0, 2.5, 5.0\}^2$                  & 0.0793 & \multirow{4}{*}{0.0126} & 55.6\%
	\\
	& $\{0.0, 1.67, 3.33, 5.0\}^2$         & 0.0190 &                        & 31.3\%
	\\
	& $\{0.0, 1.25, 2.5, 3.75, 5.0\}^2$    & 0.0119 &                        & 20\%
	\\
	& $\{0.0, 1.0, 2.0, 3.0, 4.0, 5.0\}^2$ & 0.0105 &                        & 13.9\%
	\\\hline
	\multirow{3}{*}{\makecell{$\nu \in [0,10]$ \\ $\rho \in [0,10]$}} &
	$\{0.0, 5.0, 10.0\}^2$                  & 0.636  & \multirow{3}{*}{0.0133} & 55.6\%
	\\
	& $\{0.0, 2.5, 5.0, 7.5, 10.0\}^2$      & 0.0228 &                         & 20\%
	\\
	& $\{0.0, 2.0, 4.0, 6.0, 8.0, 10.0\}^2$ & 0.0131 &                         & 13.9\%
	\\\hline
	\multirow{2}{*}{\makecell{$\nu \in [1,20]$ \\ $\rho \in [1,20]$}} &
	$\{1.0, 10.0, 20.0\}^2$              & 0.841  & \multirow{2}{*}{0.0204} & 55.6\%
	\\
	& $\{1.0, 5.0, 10.0, 15.0, 20.0\}^2$ & 0.0128 &                         & 20\%
	\\\hline
\end{tabular}
\end{table}

\begin{figure}[tbh]
    \centering
    \includegraphics[width=\linewidth]{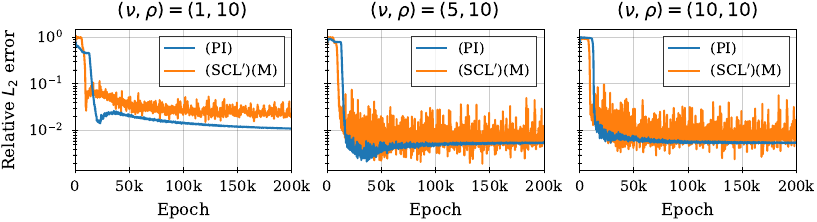}

    \includegraphics[width=\linewidth]{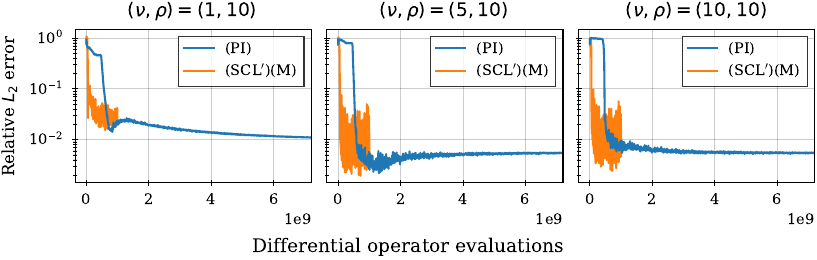}
    \caption{Relative $L_2$ error as a function of training epoch and differential operator evaluations for the parametric reaction-diffusion problem. \eqref{eq:PINN} uses the discretization $(\nu, \rho) \in \{0.0, 2.0, 4.0, 6.0, 8.0, 10.0\}^2$}
    \label{fig:error_epoch_rd}
\end{figure}

\begin{figure}[tbh]
    \centering
    \includegraphics[width=\linewidth]{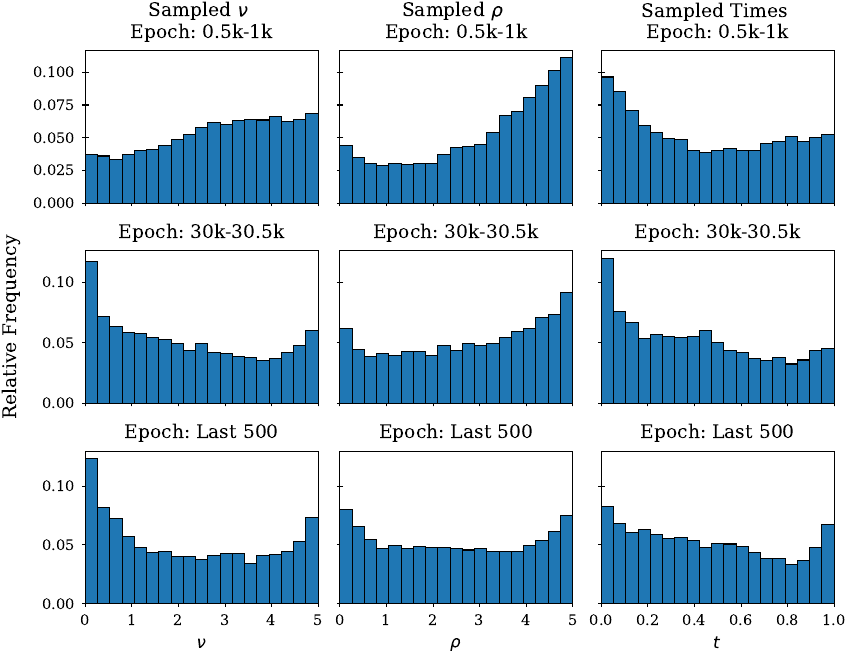}
    \caption{Histogram of (marginalized)~MH samples of~$\psi_0$ for the parametric reaction-diffusion equation.}
    \label{fig:samples_reaction_diffusion}
\end{figure}

\begin{figure}[tbh]
    \centering
    \includegraphics[width=\linewidth]{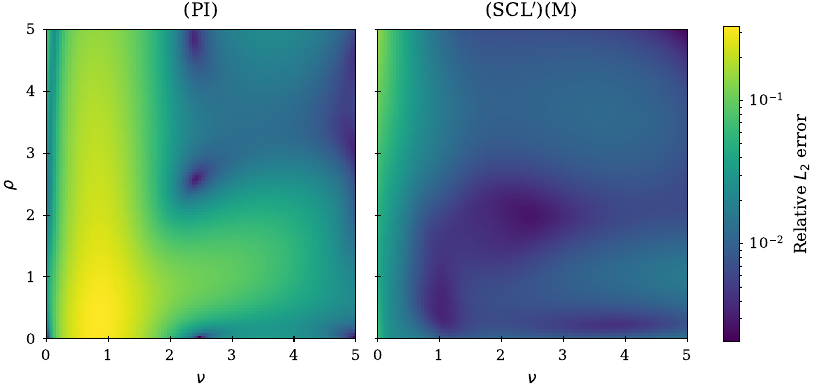}

    \hspace{-12mm}{\small (a)}

    \includegraphics[width=\linewidth]{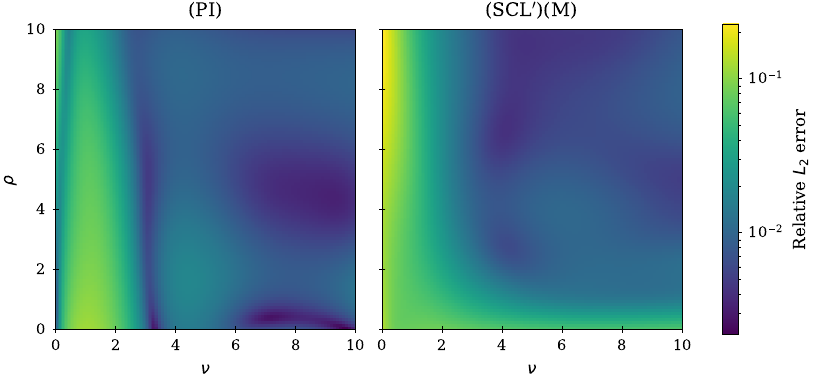}

    \hspace{-12mm}{\small (b)}

    \caption{Relative $L_2$ error for reaction-diffusion solutions trained using~\eqref{eq:scl_parametric} and~\eqref{eq:PINN} with discretization (a)~$(\nu, \rho) \in \{ 0.0, 2.5, 5.0 \}^2$ and (b)~$(\nu, \rho) \in \{0.0, 2.0, 4.0, 6.0, 8.0, 10.0\}^2$.}
    \label{fig:param_sol_reaction_diffusion}
\end{figure}

\clearpage

\paragraph{Helmholtz equation.}

Once again, we see from Table~\ref{tab:rd_param_sol_comparisson} that~\eqref{eq:scl_parametric}(M) makes more efficient use of computations than~\eqref{eq:PINN}. Indeed, in order to achieve the same error as~\eqref{eq:scl_parametric}(M), \eqref{eq:PINN} requires between~3 and~4 times more evaluations of the PDE loss~(i.e., of the differential operator~$D$) per epoch. This is clear by looking at the evolution of the error during training after normalizing the $x$-axis in terms of computational complexity~(Fig.~\ref{fig:error_epoch_rd}). Naturally, taking finer discretizations eventually leads to lower errors~(Fig.~\ref{fig:param_sol_helmholtz}), but the computational cost associated also rises considerably. On the other hand, we keep the computational cost of~\eqref{eq:scl_parametric} fixed throughout all experiments, showcasing its good performance across scenarios with little to no manipulation.

We can also inspect the samples from~$\psi_0$ throughout training to gain a better understanding of the difficulties perceived by the MLP to fit solutions of this problem~(Fig.~\ref{fig:samples_helmholtz}). We display only the~$x$ and~$a_1$ marginals, seen as they display the same behaviors as~$y$ and~$a_2$ respectively due to the symmetry of the Helmholtz equation. Here, we notice that the distribution of~$x$ has an alternating pattern initially. This makes sense seen as the solution of the Helmholtz equation is periodic. SCL clearly picks up on this pattern, focusing on the modes of the solution. As training continues, the sampling becomes more uniform, although with a focus on the boundaries of the domain where the MLP clearly has difficulties fitting the solution. With respect to the problem coefficients, we notice that~$\psi_0$ concentrates on larger values of~$a_1$, especially in the beginning of training. These are indeed coefficients for which the solution of the problem is harder to fit~(as evidenced by Fig.~\ref{fig:param_sol_helmholtz}).

\begin{table}[tbh]
\centering
\caption{Relative $L_2$ error and computational efficiency for the parametric reaction-diffusion problem.}
\begin{tabular}{|c|c|c|c|c|}
\hline
\multirow{2}{*}{\makecell[c]{Coefficients\\range}} &
\multirow{2}{*}{\makecell{Discretization for~\eqref{eq:PINN}}} &
\multicolumn{2}{c|}{Average relative $L_2$ error} &
\multirow{2}{*}{\makecell{Relative complexity\\\eqref{eq:scl_parametric} $\div$ \eqref{eq:PINN}}}
\\ \cline{3-4}
&  & \eqref{eq:PINN} & \eqref{eq:scl_parametric}(M) &
\\ \hline
\multirow{3}{*}{\makecell{$a_1 \in [1,2]$ \\ $a_2 \in [1,2]$}} &
$\{1.0, 1.5, 2.0\}^2$                  & 0.0307  & \multirow{3}{*}{0.0125} & 55.6\%
\\
& $\{1.0, 1.25, 1.5, 1.75, 2.0\}^2$    & 0.00593 &                         & 20\%
\\
& $\{1.0, 1.2, 1.4, 1.6, 1.8, 2.0\}^2$ & 0.00463 &                         & 13.9\%
\\\hline
\multirow{3}{*}{\makecell{$a_1 \in [1,3]$ \\ $a_2 \in [1,3]$}} &
$\{1.0, 2.0, 3.0\}^2$                  & 1.34 & \multirow{3}{*}{0.0549} & 55.6\%
\\
& $\{1.0, 1.5, 2.0, 2.5, 3.0\}^2$      & 0.00943 &                      & 20\%
\\
& $\{1.0, 1.4, 1.8, 2.2, 2.6, 3.0\}^2$ & 0.00953 &                      & 13.9\%
\\\hline
\end{tabular}
\label{tab:helmholtz_param_sol_comparisson}
\end{table}

\begin{figure}[tbh]
    \centering
    \includegraphics[width=\linewidth]{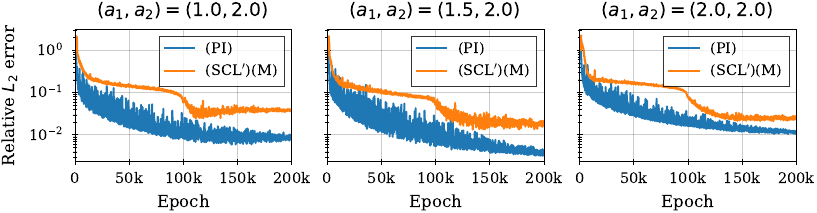}

    \includegraphics[width=\linewidth]{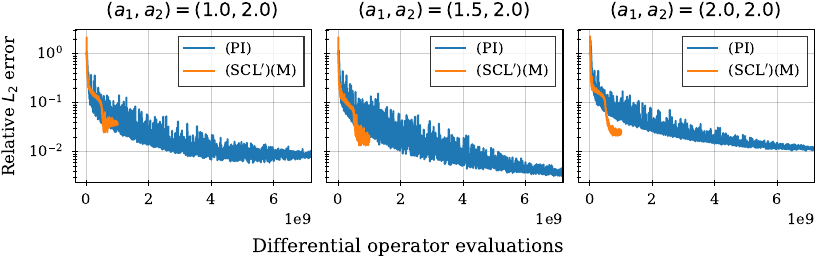}
    \caption{Relative $L_2$ error as a function of training epoch and differential operator evaluations for the parametric Helmholtz problem. \eqref{eq:PINN} uses the discretization $(a_1, a_2) \in \{1.0, 1.2, 1.4, 1.6, 1.8, 2.0 \}^2$}
    	\label{fig:error_epoch_helmholtz}
\end{figure}

\begin{figure}[tbh]
    \centering
    \includegraphics[width=\linewidth]{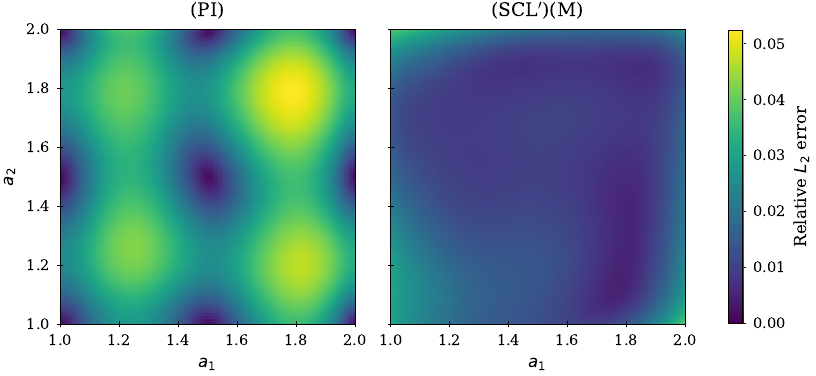}

    \hspace{-12mm}{\small (a)}

    \includegraphics[width=\linewidth]{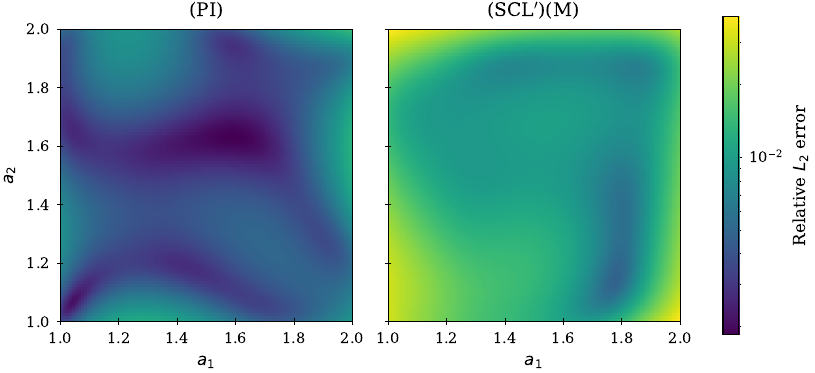}

    \hspace{-12mm}{\small (b)}

    \caption{Relative $L_2$ error for Helmholtz solutions trained using~\eqref{eq:scl_parametric} and~\eqref{eq:PINN} with discretization (a)~$(a_1, a_2) \in \{1.0, 1.5, 2.0\}^2$ and (b)~$(a_1, a_2) \in \{1.0, 1.2, 1.4, 1.6, 1.8, 2.0\}^2$.}
    	\label{fig:param_sol_helmholtz}
\end{figure}

\begin{figure}[tbh]
    \centering
    \begin{minipage}{0.48\linewidth}
        \centering
        \includegraphics[width=\linewidth]{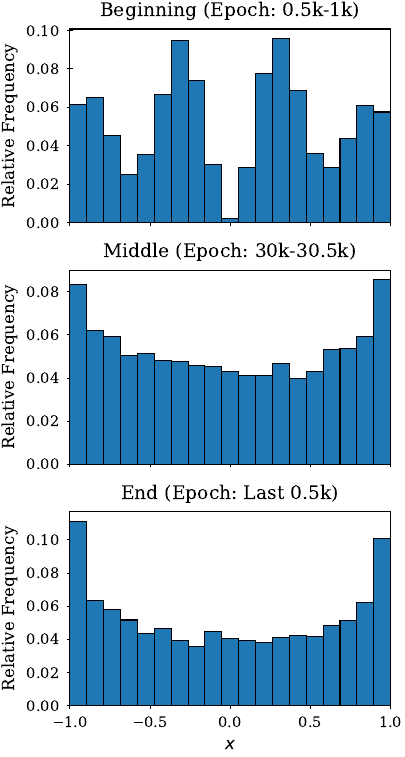}
    \end{minipage}
    \hfill
    \begin{minipage}{0.48\linewidth}
        \centering
        \includegraphics[width=\linewidth]{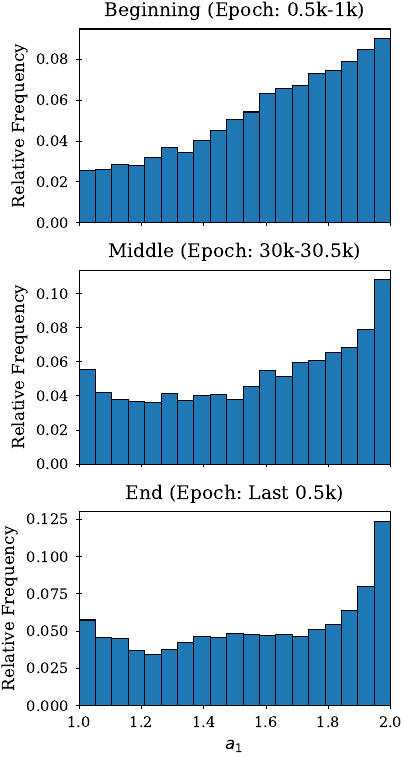}
    \end{minipage}
    \caption{Histogram of (marginalized)~MH samples of~$\psi_0$ for the parametric Helmholtz equation.}
    \label{fig:samples_helmholtz}
\end{figure}

\clearpage

\subsection{Supervised solution of BVPs}

\paragraph{Burgers' equation.}

Fig~\ref{fig:box_plot_burgers} shows the box plots for the relative~$L_2$ error across samples. They show that not only is the average error across samples smaller when using~\eqref{eq:scl_problem}(O), but in fact the whole error distribution is shifted down. This is due to the variety of weights given to different samples, weights that in fact vary during the training process~(Fig~\ref{fig:burgers_eq_dual_variable_statistics_histogram}). The few large dual variables are related to ICs that are harder to fit and can provide important information for data collection or architecture improvements. We have already shown which ICs are harder for FNOs to fit in Fig.~\ref{fig:dual_variables_burgers}.

\begin{figure}[tbh]
	\centering
	\centering
	\includegraphics[width=\linewidth]{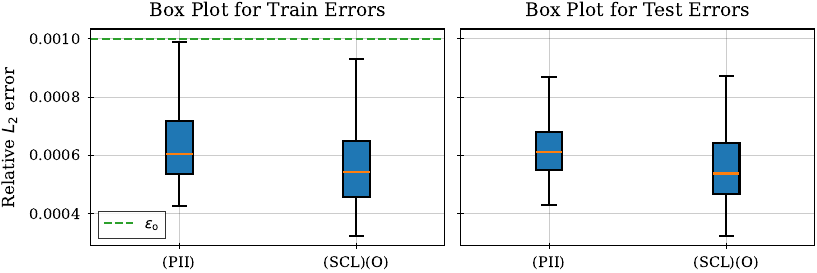}
	\caption{Distribution of train and test errors (across data points) for the Burgers' equation~(orange line indicates the median).}
	\label{fig:box_plot_burgers}
\end{figure}

\begin{figure}[tbh]
	\centering
	\begin{minipage}{0.48\linewidth}
		\centering
		\includegraphics[width=\linewidth]{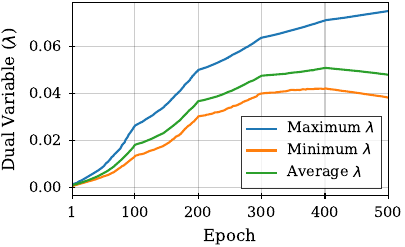}

		\hspace{8mm}{\small (a)}
	\end{minipage}
	\hfill
	\begin{minipage}{0.48\linewidth}
		\centering
		\includegraphics[width=\linewidth]{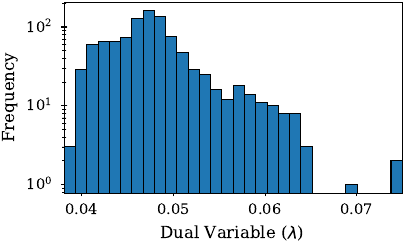}

		\hspace{8mm}{\small (b)}
	\end{minipage}
	\caption{Dual variables obtained by Alg.~\ref{alg:primal_dual} for the Burgers' equation: (a)~evolution during training and (b)~distribution of the dual variables after training.}
	\label{fig:burgers_eq_dual_variable_statistics_histogram}
\end{figure}

\clearpage

\paragraph{Diffusion-sorption equation.}

We once again display the distribution of the relative $L^2$ errors across training and test data points for models trained using~\eqref{eq:NO} and~\eqref{eq:scl_problem}(O)~(Fig.~\ref{fig:box_plots_diffsorp}). Here, we clearly see that by bounding the maximum error rather than minimizing its average leads to a more homogeneous fit across samples. This is once again due to the different weight assigned to each data sample, weights that also evolve throughout training~(Fig~\ref{fig:dual_variables_diffsorp}a). By inspecting the ICs with large and small values of~$\lambda$, we notice the pattern showcased in Fig~\ref{fig:dual_variables_diffsorp}b, where IC with either large or small magnitude are more challenging to fit than those with moderate ones.

\begin{figure}[tbh]
	\centering
	\includegraphics[width=\linewidth]{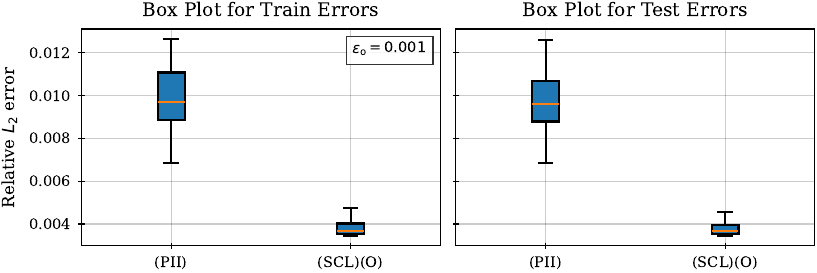}
	\caption{Distribution of train and test errors (across data points) for the diffusion-sorption equation~(orange line indicates the median).}
	\label{fig:box_plots_diffsorp}
\end{figure}

\begin{figure}[tbh]
	\centering
	\begin{minipage}{0.48\linewidth}
		\centering
		\includegraphics[width=\linewidth]{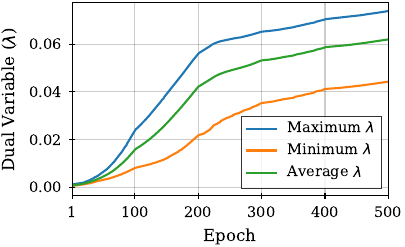}

		\hspace{8mm}{\small (a)}
	\end{minipage}
	\hfill
	\begin{minipage}{0.48\linewidth}
		\centering
		\includegraphics[width=\linewidth]{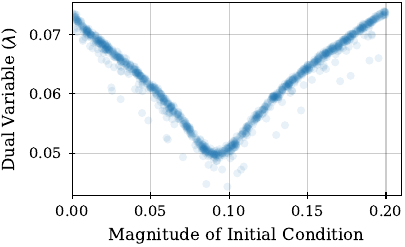}

		\hspace{8mm}{\small (b)}
	\end{minipage}
	\caption{Dual variables obtained by Alg.~\ref{alg:primal_dual} for the diffusion-sorption equation: (a)~evolution during training and (b)~value as a function of the IC magnitude.}
	\label{fig:dual_variables_diffsorp}
\end{figure}

\clearpage

\paragraph{Navier-Stokes equation.}

We start by displaying an extended version of Table~\ref{tab:results_NO} including standard deviations of results over 10 runs to show the consistency of our results across random seeds~(Table~\ref{tab:results_NO_std}). We then turn to Fig~\ref{fig:box_plot_navier_stokes} which shows that~\eqref{eq:scl_problem}(O) not only improves the average relative $L_2$ error, but its entire distribution across train and test data points. For the Navier-Stokes equations, however, it is harder to find a relation between IC properties and the difficulty of fitting the solution. That is because, as we show in Fig~\ref{fig:dual_variables_navier_stokes}, the dual variables do not have such extremely different values. This is certainly due to the fact that the tolerance~$\epsilon_\text{o}$ is set very loose~($0.01$) and that in these situations, all ICs are similarly difficult to fit. Still, some ICs have outlier values of~$\lambda$~(Fig~\ref{fig:dual_variables_navier_stokes}b), which does point to the fact that the FNO does struggle more to fit certain conditions. That being said, it not easy to identify what in those conditions make them hard~(Fig~\ref{fig:navier_stokes_ic_dual_variables}). Nevertheless, this is not an issue as we need not know beforehand which ICs are challenging: suffices it to run Alg.~\ref{alg:primal_dual} to solve~\eqref{eq:scl_problem}(O).

\begin{table}
	\centering
	\captionof{table}{Relative $L_2$ error on test set (mean $\pm$ standard deviation).}
	\label{tab:results_NO_std}
	\renewcommand{\arraystretch}{1.2}
	\begin{tabular}{lccc}
		& $\nu$ & \eqref{eq:NO} & \textbf{\eqref{eq:scl_problem}(O)}
		\\\hline
		\textbf{Burgers'}                       & $10^{-3}$ &$0.0540 \pm 0.0027\,\%$ & $0.0444 \pm 0.0020\,\%$
		\\\hline
		\multirow{3}{*}{\textbf{Navier-Stokes}} & $10^{-3}$ & $4.29 \pm 0.40 \,\%$  & $3.31 \pm 0.16 \,\%$  \\
												& $10^{-4}$ & $32.2 \pm 0.87 \,\%$  & $29.9 \pm 0.54 \,\%$  \\
												& $10^{-5}$ & $27.6 \pm 0.63\,\%$  & $26.0 \pm 0.33 \,\%$
		\\\hline
		\multicolumn{2}{l}{\textbf{Diffusion-Sorption}} &$0.274 \pm 0.049 \,\%$ & $0.218 \pm 0.036\,\%$
		\\\hline
	\end{tabular}
\end{table}

\begin{figure}[tbh]
	\centering
	\includegraphics[width=\linewidth]{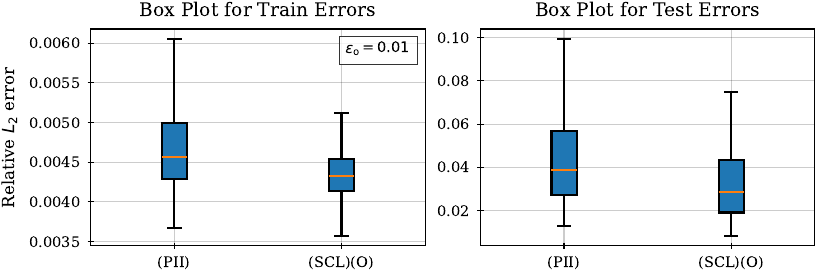}
	\caption{Distribution of train and test errors (across data points) for the Navier-Stokes equation with~$\nu = 10^{-3}$~(orange line indicates the median).}
	\label{fig:box_plot_navier_stokes}
\end{figure}

\begin{figure}[tbh]
	\centering
	\begin{minipage}{0.48\linewidth}
		\centering
		\includegraphics[width=\linewidth]{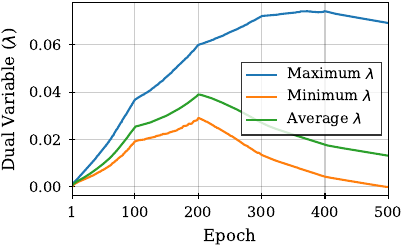}

		\hspace{8mm}{\small (a)}
	\end{minipage}
	\hfill
	\begin{minipage}{0.48\linewidth}
		\centering
		\includegraphics[width=\linewidth]{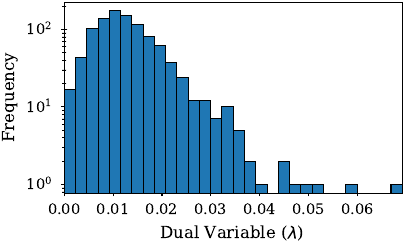}

		\hspace{8mm}{\small (b)}
	\end{minipage}
	\caption{Dual variables obtained by Alg.~\ref{alg:primal_dual} for the Navier-Stokes equation with $\nu = 10^{-3}$: (a)~evolution during training and (b)~distribution of the dual variables after training.}
	\label{fig:dual_variables_navier_stokes}
\end{figure}

\begin{figure}[tbh]
	\centering
	\includegraphics[width=\linewidth]{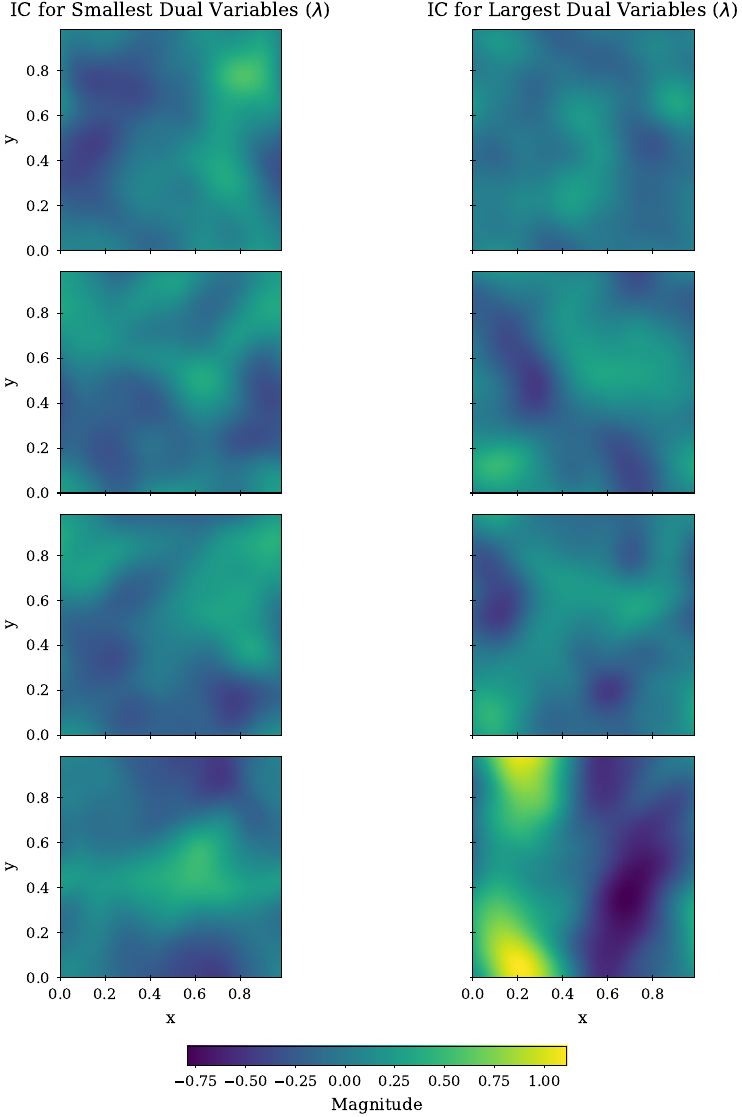}
	\caption{Initial conditions corresponding to the smallest and largest final dual variables for Navier-Stokes equation with $\nu = 10^{-3}$.}
	\label{fig:navier_stokes_ic_dual_variables}
\end{figure}

\clearpage
\section{Additional related work}
\label{app:related_work}

\subsection{Physics-informed neural networks}

Fitting MLPs to the solution of BVPs goes back to~\citep{Psichogios92a, Dissanayake94n, Lagaris98a}. The advent of differentiable programming and automatic differentiation, however, lead to an increased interest in this approach, which has since been used to tackle both forward and inverse problems involving a variety of PDEs~(see, e.g., \citep{Raissi19p, Wight21s, Chen20p, Lu21p, Basir22p, Yu22g, Xu23t}). This led to new architectures tailored for PDEs~\citep{Raissi19p, Fathony21m, Gao21p, Wang21u, Kang23p, Moseley23f, Cho24p, Chalapathi24s}, leveraging positional embedding~\citep{Wang21o} and adaptive activation functions~\citep{Jagtap20a}.

\paragraph{Training PINNs.} PINNs tend to be very sensitive to training hyperparameters, particularly the choice of collocation points and loss weights. Many works have theoretically and empirically investigated the origins of these issues~\citep{Krishnapriyan21c, Markidis21t, Wight21s, Wang21u, Wang22w, Wang22i, Grossmann24c}. Based on these observations, adaptive heuristics have been proposed to select the collocation points based on importance sampling~\citep{Nabian21e, Wu23a}, adversarial training~\cite{Wang22i}, rejection sampling~\citep{Daw23m}, and causality-inspired rules~\citep{Penwarden23a, Wang24r}. Similarly, empirical rules for determining the loss weights~[$\mu$ in~\eqref{eq:PINN}] have been developed using the magnitude of the gradients~\citep{Wang21u}, eigenvalues of the neural tangent kernel~\citep{Wang22w}, inverse-Dirichlet weighting~\citep{Maddu22i}, soft attention mechanisms~\citep{McClenny23s}, and (augmented)~Lagrangian formulations~\citep{Lu21p, Basir22p}. Other works have addressed these challenges by changing the problem formulation inspired by traditional numerical methods~\citep{Kharazmi21h, Chiu22c, Patel22t}, proposing different objective functions~\citep{Yu22g, Son21s}, and using sequential~\citep{Wight21s, Krishnapriyan21c} and transfer learning~\citep{Goswami20t, Chakraborty21t, Desai22o} techniques.

In contrast to these approaches, we address these issues jointly by using worst-case losses and constrained learning to obviate these hyperparameters. In fact, we prove that the constrained learning problems we pose yield (weak)~solutions of BVPs. Hence, it is not enough to use either adversarial training to estimate the worst-case loss as in~\citep{Wang22i} or constrained learning to manipulate the loss weights as in~\citep{Lu21p, Basir22p}. Both are required simultaneously. Leveraging findings from adversarially robust learning, we also replace the gradient methods used in~\citep{Wang22i}, i.e., the technique from~\citep{Madry18t}, by the sampling-based approach in~\citep{Robey21a}.

\subsection{Neural Operators}

In contrast to the MLPs and convolutional NNs~(CNNs) typically used in PINNs, NOs are NNs capable of handling infinite-dimensional inputs and outputs. They can therefore be trained to find BVP solutions for different IC or forcing functions. Many different architectures have been proposed, such as DeepONets~\citep{Lu21l}, FNOs~\citep{Li21f}, and NO based on U-Nets~\citep{Gupta23t}. FNOs in particular have become quite popular and garnered many efforts towards addressing their limitations, such as improving memory efficiency~\citep{Rahman23u}, designing equivariant FNOs~\citep{Helwig23g}, extending FNOs to general geometries~\citep{Li23f}, factorizing the Fourier transform~\citep{Tran23f}, and leveraging multiwavelets~\citep{Gupta21m}.

\paragraph{Training NOs.} Regardless of these improvements, the vast majority of NOs are trained in a supervised manner by minimizing their average error across samples as in~\eqref{eq:NO}~(see, e.g., \citep{Lu21l, Li20m, Kovachki23n}). Unless substantial domain knowledge has been used during data collection, challenging cases may be underrepresented in the dataset, which could hinder the accuracy of the NO. Although semi-supervised techniques involving PDE losses have also been used~\citep{Li24p}, computing the space-time derivatives needed to evaluate~$D_\pi[u]$ is challenging for NOs.

In this paper, we do not develop new NO architectures, but focus on the problem of training them. Explicitly, rather than targeting the average error, we target the maximum error across samples. This is much better suited to handle the heterogeneous difficulty in fitting the data. We also incorporate structure in the solution during training, without the need to design new architectures.

\subsection{Constrained and adversarially robust learning}

The main tool used in the development of SCL is constrained learning, or more specifically, robustness-constrained learning. Constrained learning is a technique to train ML systems under requirements, such as fairness~\citep{Kearns18p, Cotter19o, Chamon20p, Chamon23c} and robustness~\citep{Chamon20p, Robey21a, Hounie23a, Chamon23c}, or to handle applications in which we want to attain good performance with respect to more than one metric. As in unconstrained learning, it is formulated as statistical risk minimization problem, albeit with constraints. Despite its non-convexity in virtually every modern ML task, certain duality properties hold when using sufficiently expressive parametrizations, leading to a practical learning rule with generalization guarantees~\citep{Chamon20p, Chamon23c}. These duality results have also been exploited to automatically adapt each constraint specification to their underlying difficulty, striking better compromises between objective and requirements~\cite{Hounie23r}.

Aside from dealing with the nominal accuracy vs.\ robustness trade-off typical in ML systems, constrained learning has itself been used to optimize robust losses. Typically, this is done using some combination of gradient ascent and random initialization, restart, and pruning heuristics~\citep{Goodfellow15e, Madry18t, Dhillon18s, Wu20a, Cheng22c}. Using semi-infinite optimization techniques, however, these deterministic methods can be replaced by a sampling approach that has been successful in a variety of domains~\citep{Robey21a}. Different MCMC methods~\citep{Robert04m} have been used in this context, including Langevin Monte Carlo~(LMC)~\citep{Robey21a} and Metropolis-Hastings~(MH)~\citep{Hounie23a}.

\end{document}